\documentclass{article}

\usepackage{microtype}
\usepackage{graphicx}
\usepackage{caption}
\usepackage{subcaption}
\usepackage{booktabs}
\usepackage{adjustbox}
\usepackage{multirow}
\usepackage{comment}
\usepackage{placeins}
\usepackage{booktabs} 
\usepackage{algorithm2e}
\usepackage{hyperref}

\usepackage[accepted]{icmlarxiv}

\usepackage{amsmath}
\usepackage{amssymb}
\usepackage{mathtools}
\usepackage{amsthm}

\usepackage[capitalize,noabbrev]{cleveref}

\theoremstyle{plain}
\newtheorem{theorem}{Theorem}[section]
\newtheorem{proposition}[theorem]{Proposition}
\newtheorem{example}[theorem]{Example}
\newtheorem{lemma}[theorem]{Lemma}

\theoremstyle{definition}
\newtheorem{definition}[theorem]{Definition}

\theoremstyle{remark}

\DeclareMathOperator*{\argmin}{arg\,min}

\DeclareMathOperator{\tr}{Tr}
\DeclareMathOperator{\E}{\mathbb{E}}
\let\P\relax
\DeclareMathOperator{\P}{\mathbb{P}}

\usepackage[textsize=tiny]{todonotes}

\icmltitlerunning{Statistical Guarantees for Link Prediction using Graph Neural Networks}

\begin{document}
\twocolumn[
\icmltitle{Statistical Guarantees for Link Prediction using Graph Neural Networks}



\icmlsetsymbol{equal}{*}

\begin{icmlauthorlist}
\icmlauthor{Alan Chung }{harvard}
\icmlauthor{Amin Saberi}{stanford}
\icmlauthor{Morgane Austern}{harvard}
\end{icmlauthorlist}

\icmlaffiliation{harvard}{Department of Statistics, Harvard University}
\icmlaffiliation{stanford}{Department of Management Science \& Engineering, Stanford University}

\icmlcorrespondingauthor{Alan Chung}{alanchung@g.harvard.edu}



\icmlkeywords{Machine Learning, graph neural network, statistical guarantees, link prediction}

\vskip 0.3in
]


\printAffiliationsAndNotice{}


\begin{abstract}
This paper derives statistical guarantees for the performance of Graph Neural Networks (GNNs) in link prediction tasks on graphs generated by a graphon. We propose a linear GNN architecture (LG-GNN) that produces consistent estimators for the underlying edge probabilities. We establish a bound on the mean squared error and give guarantees on the ability of LG-GNN to detect high-probability edges. Our guarantees hold for both sparse and dense graphs. Finally, we demonstrate some of the shortcomings of the classical GCN architecture, as well as verify our results on real and synthetic datasets. 


\end{abstract}


\section{Introduction}

Graph Neural Networks (GNNs) have emerged as a powerful tool for link prediction \cite{zhang2018link, GNNBook-ch10-zhang}. A significant advantage of GNNs lies in their adaptability to different graph types. Traditional link prediction heuristics tend to presuppose network characteristics. For example, the common neighbors heuristic presumes that nodes that share many common neighbors are more likely to be connected, which is not necessarily true in  biological networks \cite{ppi-kovasc}. In contrast, GNNs inherently learn predictive features through the training process, presenting a more flexible and adaptable method for link prediction.

This paper provides statistical guarantees for link prediction using GNNs, in graphs generated by the graphon model. A graphon is specified by  a symmetric measurable kernel function $W:\Omega^2\rightarrow [0,1]$. A graph $G_n = (V_n, E_n)$ with the vertex set $V_n = \{1, 2, \cdots, n\}$ is sampled from $W$ as follows: (i) each vertex $i \in V_n$ draws latent feature  $(\omega_i)\overset{i.i.d}{\sim} \mu$ for some probability distribution $\mu$ on $\Omega \subset \mathbb{R}^q$; (ii) the edges of $G_n$ are generated independently and with probability $W_{n, i, j} := \rho_n \cdot W(\omega_i, \omega_j),$ where  $\rho_n\in (0,1]$ is a constant called the sparsifying factor\footnote{Throughout the paper, we will assume $(\omega_i)\sim \text{Unif}[0,1]$. This is without loss of generality. For any graphon $\tilde{W}$ with features in some arbitrary $\Omega \subset \mathbb{R}^q$ sampled from $\mu$ on $\Omega$, there exists some graphon $W$ with latent features drawn from $\text{Unif}[0,1]$ so that the graphs generated from these two graphons are  equivalent in law. See Remark 4 in \cite{morgane_paper} for more details.}.

 The graphon model includes various widely researched graph types, such as Erdos-Renyi, inhomogeneous random graphs, stochastic block models, degree-corrected block models, random exponential graphs, and geometric random graphs as special cases; see \cite{lovasz2012large} for a more detailed discussion.

This paper's key contribution is the analysis of a linear Graph Neural Network model (LG-GNN) which can provably estimate the edge probabilities  in a graphon model. Specifically, we present a GNN-based algorithm that yields estimators, denoted as  $\hat{p}_{i,j}$, that converge to true edge probabilities $\rho_n W(\omega_i, \omega_j)$. Crucially, these estimators have mean squared error converging to 0 at the rate $o_{n \to \infty} (\rho_n^2)$. To our knowledge, this work is the first to rigorously characterize the ability of GNNs to estimate the underlying edge probabilities in general graphon models.





The estimators $\hat{p}_{i,j}$ are constructed in two main steps. We first employ LG-GNN (\cref{algo1}), to embed the vertices of $G_n$. Concretely, for each vertex $i \in [n]$, LG-GNN computes a set $\Lambda_i =\{ \lambda_i^0, \lambda_i^1, \dots, \lambda_i^L\},$ where $\lambda_i^k \in \mathbb{R}^{d_n}$ and  $d_n$ is the embedding dimension. Then, $\Lambda_i, \Lambda_j$ are used to construct estimators $\hat{q}_{i,j}^{(k)}$ for the \textit{moments} of $W$. We refer the reader to  \cref{sec:main-results} for the formal definition of the moments. Intuitively, the $k$th moment $W_{n, i, j}^{(k)}$ represents the probability that there is a path of length $k$ between vertices with latent features $\omega_i$ and $\omega_j$ in $G_n$.

The next  step is to show that when the number of distinct nonzero eigenvalues of $W$, denoted $m_W$, is finite, 
then the edge probabilities  $W_{n, i, j}$ can be written as a linear function of the moments $W_{n, i, j}^{(2: m_W+1)}$. This naturally motivates \cref{alg:compute_regression_coefs}, which learn the $W_{n, i, j}$'s from the moment estimators $\hat{q}_{i,j}^{(2: m_W+1)}$, using a constrained regression. The regression coefficients $\hat{\beta}^{n, m_W}$  are then used to produce estimators $\hat{p}_{i,j}$ for $W_{n, i, j}$. 

The main result of the paper (stated in  \cref{prop:convergence_of_c}) presents the convergence rate of the mean square error of $\hat{\beta}^{n, m_W}$. It shows that if $L$, the number of message-passing layers in LG-GNN, is at least $m_W-1$, then the mean square error converges to 0. That  implies that our estimators for the edge probabilities $W_{n, i, j}$ are consistent. For $L < m_W-1$, the theorem provides the rate at which the mean square error decreases when $L$ increases. 
The second main result, stated in \cref{prop:preserve-rank},  gives statistical guarantees on how well LG-GNN can detect in-community edges in a symmetric stochastic block model. A notable feature of this theorem is that the implied convergence rate is much faster than that of \cref{prop:convergence_of_c}, which demonstrates mathematically that \textit{ranking} high and low probability edges is easier than estimating the underlying probabilities of edges.

Finally, we would like to highlight two key aspects of the results of the paper. Firstly, our statistical guarantees for edge prediction are proven for scenarios when  node features are absent and the initial node embeddings \(\lambda_i^0\) are chosen at random. This underscores that effective link prediction can be achieved solely through the appropriate selection of GNN architecture, even in the absence of additional node data. The second point relates to graph sparsity: although graphons typically produce dense graphs, introducing the sparsity factor \(\rho_n\) results in vertex degrees of \(O(\rho_n \cdot n)\), facilitating the exploration of sparse graphs. Our findings are pertinent for \( \log(n)/n \ll \rho_n \leq 1\). Note that a sparsity of $\log(n)/n$ is the necessary threshold for connectivity \cite{spencer2001strange}, highlighting the generality of our results. 


While the primary focus of this paper is theoretical, we complement our theoretical analysis with experimental evaluations on real-world datasets (specifically, the Cora dataset) and graphs derived from random graph models. Our empirical observations reveal that in scenarios where node features are absent, LG-GNN exhibits performance comparable to the traditional Graph Convolutional Network (GCN) on simple random graphs, and  surpasses GCN in more  complex graphs sampled from graphons. Additionally, LG-GNN presents two further benefits: LG-GNN does not involve any parameter tuning (e.g., through the minimization of a loss function), resulting in significantly faster operation, and it avoids the common oversmoothing issues associated with the use of numerous message-passing layers.

\subsection{Organization of the Paper}

Section \ref{sec:related-works} discusses related works and introduces the motivation for our paper. \cref{sec:notation-preliminaries} introduces our notation and presents an outline for our exposition. \cref{sec:main-results} presents our main results and \cref{sec:negative-gcn} states a negative result for naive GNN architectures with random embedding initialization. Lastly, \cref{sec:identifiability} discusses the issues of identifiability, and \cref{sec:mainbodyexperiments} presents our experimental results.

\section{Related Works}
\label{sec:related-works}

Link prediction on graphs have a wide range of applications in domains ranging from social network analysis to drug discovery \cite{hasan2011survey, abbas2021application}. A survey of techniques and applications can be found in \cite{kumar2020link, martinez2016survey, comprehensive-survey}. 

Much of the existing theory on GNNs is regarding their expressive power. For example, \cite{xu2018powerful, morris2021weisfeiler} show that GNNs with deterministic node initializations have expressive power bounded by that of the 1-dimensional Weisfeiler-Lehman (WL) graph isomorphism test. Generalizations such as $k$-GNN \cite{morris2021weisfeiler} have been proposed to boost the expressive power higher in the WL-hierarchy. The Structural Message Passing GNN (SGNN) \cite{vignac2020building} was also proposed and was shown to be universal on graphs with bounded degrees, and converges to continuous "c-SGNNs" \cite{keriven2021universality}, which were also shown to be universal on many popular random graph models. Lastly, \cite{abboud2020surprising} showed that GNNs that use random node initializations are universal, in that they can approximate any function defined on graphs with fixed order. 

A recent wave of works focus on deriving statistical guarantees for graph representation algorithms. A common data-generating model for the graph is a graphon model \cite{lovasz2006limits,borgs2008convergent,borgs2012convergent}. A large literature has been devoted to establishing guarantees for community detection on graphons such as the stochastic block model; see \cite{abbe2018community} for an overview. For this task, spectral embedding methods have long been proposed (see \cite{deng2021strong,ma2021determining} for some recent examples). Lately, statistical guarantees for modern random walk-based graph representation learning algorithms have also been obtained. Notably \cite{davison2021asymptotics, barot2021community,qiu2018network,zhang2021consistency} characterize the asymptotic properties of the embedding vectors obtained by deepwalk, node2vec, and their successors and obtain statistical guarantees for downstream tasks such as edge prediction. Recently, some works also aim at obtaining learning guarantees for GNNs. Stability and transferability of certain untrained GNNs have been established in \cite{ruiz2021graphon,maskey2023transferability,ruiz2023transferability,keriven2020convergence}. For example \cite{keriven2020convergence} shows that for relatively sparse graphons, the embedding produced by an untrained GNN will converge in $L^2$ to a limiting embedding that depends on the underlying graphon. They use this to study the stability of the obtained embeddings to small changes in the training distribution. Other works established generalization guarantees for GNNs. Those depend respectively on the number of parameters in the GNN \cite{maskey2022generalization}, or on the VC dimension and Radamecher complexity \cite{esser2021learning} of the GNN. 

Differently from those two lines of work, our paper studies when link prediction is possible using GNNs, establishes statistical guarantees for link prediction, and studies how the architecture of the GNN influences its performance. More similar to our paper is \cite{kawamoto2018mean}, which exploits heuristic mean-field approximations to predict when community detection is possible using an untrained GNN. Note, however, that contrary to us, their results are not rigorous and the accuracy of their approximation is instead numerically evaluated.  \cite{lu2021learning} formally established guarantees for in-sample community detection for two community SBMs with a GNN trained via coordinate descent. However, our work establishes learning guarantees for general graphons beyond two-community SBMs, both in the in-sample and out-sample settings. Moreover, the link prediction task we consider, while related to community detection, is still significantly different. \cite{baranwal2021graph} studies node classification for contextual SBMs and shows that an oracle GNN can significantly boost the performance of linear classifiers. Another related work \cite{magner2020power} studies the capacity of GNN to distinguish different graphons when the number of layers grows at least as $L=\Omega(\log(n))$. Interestingly they find that GNN struggles in differentiating graphons whose expected degree sequence is not sufficiently heterogeneous, which unfortunately occurs for many graphon models, including the symmetric SBM. It is interesting to note that in \cref{prop:negative-result} we will show that this is also the regime where the classical GCN fails to provide reliable edge probability prediction. Finally, some learning guarantees have also been derived for other graph models. Notably \cite{alimohammadi2023local} studied the convergence of GraphSAGE and related GNN architectures under local graph convergence.
\section{Notation and Preliminaries} 
\label{sec:notation-preliminaries} 

In this section, we present our assumptions, some background regarding GNNs, and the link prediction goals that we focus on.

\subsection{Assumptions}

As mentioned in the introduction,  the random graph $G_n = (V_n, E_n)$ with the vertex set $V_n = \{1, 2, \cdots, n\}$ is sampled from a graphon $W:[0,1]^2\rightarrow [0,1]$, where each vertex $i \in V_n$ draws latent feature  $(\omega_i)\overset{i.i.d}{\sim}\text{Unif}[0,1]$ and the edges are generated independently and with probability $W_{n, i, j} := \rho_n \cdot W(\omega_i, \omega_j).$ We let $A = (a_{ij})$ denote the adjacency matrix. When the graph and context are clear, we let $W_n := \rho_n W$, and let $W_{n, i,j} := \rho_n W(\omega_i, \omega_j)$. We make the following three assumptions: 
\begin{gather}
 \label{asp1}\tag{$H_1$}  \log(n)/n \ll \rho_n \le 1 \\
\label{asp2}\tag{$H_2$} \exists \text{ } \delta_W 
> 0 \text{ s.t. } \delta_W\le W(\cdot, \cdot)\le 1-\delta_W \\
\label{asp3}\tag{$H_3$} \text{$W$ is a H\"{o}lder-by-parts function}
\end{gather}
We refer the reader to \cref{graphon-appendix} for a more detailed discussion.

\subsection{Graph Neural Networks}
An $L$-layer GNN, comprised of $L$ processing layers, transforms graph data into numerical representations, or embeddings, of each each vertex. Concretely, a GNN associates each vertex $i \in [n]$ to some $\lambda_i^L \in \mathbb{R}^{d_n},$ where we call $d_n$ the embedding dimension. The learned embeddings are then used for downstream tasks such as node prediction, graph classification or link prediction, as investigated in this paper.  

A GNN computes the embeddings iteratively through message passing. We let $\lambda_i^k$ denote the embedding produced for vertex $i$ after $k$ GNN iterations. As such, $\lambda_i^0$ denotes the initialization of the embedding for vertex $u$. The message passing layer can be expressed generally as 
$$
\lambda_i^{k+1} = \phi \left( \lambda_i^k,  \bigoplus_{j \in N(i)} \psi(\lambda_i^k, \lambda_j^k, e_{ij}) \right),
$$
where $N(i)$ is the set of neighbors of vertex $i$, $\phi, \psi$ are continuous functions, $e_{ij}$ is the feature of the edge $(u,v)$, and $\bigoplus$ is some permutation-invariant aggregation operator, for example, the sum \cite{GNNBook2022}. 

One classical architecture is the Graph Convolutional Network (GCN) \cite{kipf2017semisupervised}, whose update equation is given by \begin{equation}\label{GCN-equation}
    \lambda_i^{k} = \sigma \left(M_{k, 0} \lambda_i^{k-1}  + M_{k, 1} \sum_{j \in N(i)} \frac{\lambda_j^{k-1}}{\sqrt{|N(i)|\cdot|N(j)|}} \right),
\end{equation}
where $\sigma(\cdot)$ is a non-linear function and  $M_{k,0},M_{k,1}\in \mathbb{M}_{d_n\times d_n}(\mathbb{R})$ are matrices. These matrices are chosen by minimizing some empirical risk during a training process, typically through gradient descent. 

In some settings, additional node features for each vertex are given, and the initialization $\lambda_i^0$ is chosen to incorporate this information. In this paper, we focus on the setting when no node features are present, and a natural way to initialize our embeddings $\lambda_i^0$ is at random. One of our key messages is that even without additional node information, link prediction is provably possible with a correct choice of GNN architecture.

\subsection{Link Prediction}

Given a graph $G_n=([n],E_n)$ generated from a graphon, potential link prediction tasks are (a) to determine which of the non-edges are most likely to occur, or (b) to estimate the underlying probability of a particular edge $(i,j)$ according to the graphon. 
Here, we make the careful distinction between two different link prediction evaluation tasks. One task is regarding the \textbf{ranking} of a set of test edges. Suppose a set of test edges $e_1, e_2, \dots, e_k$ has underlying probabilities $p_{e_1} \ge p_{e_2} \ge \dots \ge p_{e_k}$ according to $W$. The prediction algorithm assigns a predicted probability $\hat{p}_{e_i}$ for each edge $e_i$ and is evaluated on how well it can extract the true ordering (e.g., the AUC-ROC metric). 

Another link prediction task is to estimate the underlying probabilities of edges in a random graph model. For example, in a stochastic block model, a practitioner might wish to determine the underlying connection probabilities, as opposed to simply determine the ranking. We will refer to this task as \textbf{graphon estimation}. It is important to note that the latter task is generally more difficult.


We also distinguish between two link prediction settings, i.e., the \textbf{in-sample} and \textbf{out-sample} settings. In in-sample prediction, the aim is to discover potentially missing edges between two vertices $i,j\in [n]$ already present at training time. On the contrary, in out-of-sample prediction, the objective is to predict edges among vertices that were not present at training. If $\tilde{V}$ are the set of vertices not present at training, the goal is to use the trained GNN to predict edges $(i,j)$ for $i, j \in \tilde{V}$, or $(i,j)$ for $i \in V_{train}$ and $j \in \tilde{V}.$

\section{Main Results}
\label{sec:main-results}
We introduce the \textit{Linear Graphon Graph Neural Network}, or LG-GNN in \cref{algo1}. The algorithm starts by assigning each node $i$ a random feature $Z_i \sim \frac{1}{\sqrt{d_n}} \mathcal{N}(0, I_{d_n}),$ where $d_n  = \Omega(1/\rho_n)$ is the embedding dimension. The first message passing layer computes $\lambda_i^0$ by summing $Z_j$ for all $j \in N(i)$, scaled by $\frac{1}{\sqrt{n}}.$ The subsequent layers normalize the $\lambda_j^k$'s by $1/n$ before adding them to $\lambda_i^k$. We show in \cref{prop:form-of-embedding} and \cref{lemma:expectation-dotproducts} that this procedure essentially counts the number of paths between pairs of vertices. Specifically, $\E[ \langle \lambda_i^{k_1}, \lambda_j^{k_2} \rangle | A, (\omega_\ell)]$ is a linear combination of the "empirical moments" of $W$ (\cref{defn:empirical-moment}). The second stage of \cref{algo1} then recovers these empirical moments by decoupling the aforementioned linear equations. We refer the reader to \cref{appendix:lg-gnn-construction} for more details and intuition behind LG-GNN.


We note that the scaling of $1/\sqrt{n}$ in the first message passing layer is crucial in allowing the embedding vectors $(\lambda_\ell^k)$ to learn information about the latent features $(\omega_\ell)$ asymptotically. We show in \cref{prop:negative-result} that without this construction, the classical GCN is unable to produce meaningful emebddings with random feature initializations.

\subsection{Statistical Guarantees for Moment Estimation}

Define the $k$th moment of a sparsified graphon $W_n$ as $W_n^{(k)}(x,y):=$ \begin{align}
   \int_{[0,1]^{k-1}} W_n(x, t_1)W_n(t_1, t_2)  \dots W_n(t_{k-1}, y) \text{d}t_{1:k-1}, \nonumber
\end{align} which is the probability that there is a path of a length $k$ between two vertices with latent features $x,y$, averaging over the latent features of the vertices in the path. As with the graphon itself, we denote $W_{n, i,j}^{(k)} := W_n^{(k)}(\omega_i, \omega_j).$ The following proposition shows that the estimators $\hat{q}_{i,j }^{(k)}$ are consistent estimators for these moments $W_{n, i,j}^{(k)}.$

\begin{proposition}
\label{prop:graph_concentration_sparse}
Suppose that the graph $G_n=([n],E_n)$ is generated according to a graphon $W_n=\rho_nW$. Suppose that assumptions (\ref{asp2}) and (\ref{asp3}) hold. Then, with probability at least $1 - 5/n - n \cdot \rm{exp}(-\delta_W \rho_n(n-1)/3)$,  for all $2 \leq k \leq L+2,$ 
\begin{gather}
    \left| \hat{q}_{i,j}^{(k)} - {W}_{n, i, j}^{(k)} \right| 
    \leq \frac{\rho_n^{k-1}}{\sqrt{n-1}} \log(n)^k \left[ 3 a_k \sqrt{\rho_n} + \frac{96 a_{k-1}}{\sqrt{d_n}} \right],
\end{gather}
where $a_k =C  (8(k+2))^k k^{k+1}\sqrt{k!}$ and $C$ is some absolute constant.
\end{proposition}

\begin{algorithm}[h]
\caption{LG-GNN architecture}
\label{algo1}
\textbf{Input:} a Graph $G_n=([n],E_n)$; $L \ge 0$\\
\textbf{Output:} estimators $\hat{q}_{i,j}^{(k)}$ for the $k$th moments $W_{n,i,j}^{(k)}.$ \\

Sample $(Z_i)_{i=1}^n \stackrel{iid}{\sim} \frac{1}{\sqrt{d_n}} \mathcal{N}(0, I_{d_n}).$

\textbf{GNN Iteration:}

\For{$i \in [n]$}{
   $\lambda_i^0 \gets \frac{1}{\sqrt{n-1}} \sum_{\ell=1}^n a_{i \ell} Z_\ell$
}
\For{$k \in [L]$}{
    \For{$i \in [n]$} {
    $\lambda_i^k \gets \lambda_i^{k-1} + \frac{1}{n-1} \sum_{\ell \leq n} a_{i\ell} \lambda_\ell^{k-1}$
    }
} 
\text{} \\
\textbf{Computing Estimators for $W_{n,i,j}^{(k)}$:}

\For{$i \neq j$}{
    $\hat{q}_{i,j}^{(2)} := \langle \lambda_i^0, \lambda_j^0 \rangle.$
}
\For{$k \in \{3, 4, \dots, L+2\}$}{
    $\hat{q}_{i,j}^{(k)}:= \langle \lambda_i^{k-2}, \lambda_j^0 \rangle -\sum_{r=0}^{k-3} \binom{k-2}{r} \hat{q}_{i,j}^{(r+2)}$
}

\textbf{Return: $ \big \{ (\hat{q}_{ij}^{(2)}, \hat{q}_{ij}^{(3)}, \dots, \hat{q}_{ij}^{(L+2)})_{i \neq j} \big \}$} 
\end{algorithm}

\subsection{Edge Prediction Using the Moments of the Graphon}
\label{sec:linear-relationship}
\cref{prop:graph_concentration_sparse} relates the embeddings produced by LG-GNN to the underlying graph moments. 
We show in \cref{prop:convergence_of_c} that the  $\hat{q}_{i,j}^{(k)}$s can be used to derive consistent estimators for the underlying edge probability $W_{n, i,j}$ between vertices $i$ and $j$. 

The key observation is that for  any H\"{o}lder-by-parts graphon $W$, there exists some $m\in \mathbb{N}\cup\{\infty\}$ such that $$W(x,y) =  \sum_{i=1}^m \mu_i \phi_i(x) \phi_i(y)\qquad\forall x,y\in[0,1]$$ for some sequence of eigenvalues $(\mu_i)$ with $|\mu_i| \leq 1$ and eigenfunctions $(\phi_i)$ orthonormal in $L^2([0,1])$. This, coupled with the Cayley-Hamilton theorem \cite{hamilton1853lectures}, implies that $W$ can be re-expressed as a linear combination of its moments. We will refer to the number of distinct nonzero eigenvalues of $W$ as $m_W$, which we call the \textit{distinct rank}.

\begin{proposition}
\label{prop:linear-relationship}
    Suppose that $W:[0,1]^2\rightarrow[0,1]$ is a H\"{o}lder-by-parts graphon. Then, there exists a vector $\beta^{*, m_W} = \left(\beta_1^{*, m_W}, \beta_2^{*, m_W} \dots, \beta_{m_W}^{*, m_W} \right)$ such that for all $(x,y) \in [0,1]^2$,
    \begin{equation}
    \label{eq:linear-fit}
        W(x,y) = \sum_{i=1}^{m_W} \beta_i^{*, m_W} W^{(i+1)}(x,y).
    \end{equation}
\end{proposition}

The above suggests the following algorithm for edge prediction using the embedding produced by LG-GNN.
\begin{algorithm}[h]
\caption{LG-GNN edge prediction algorithm}
\label{alg:compute_regression_coefs}
\textbf{Input:}  Graph $G_n=([n],E_n)$, search space $\mathcal{F}$, threshold $\beta$; $L \ge 0$. \\
\textbf{Output:} Set of predicted edges 


Using \cref{algo1},  compute $q_{i,j}^{(2:L+2)} := (\hat q_{i,j}^{(2)},\dots, q_{i,j}^{(L+2)})$ for every vertex $i,j$

\textbf{Compute:} $$\hat{\beta}^{n, L+1} = \argmin_{\beta \in \mathcal{F}} \sum_{i\ne j} \Big(\left \langle \beta,\hat q_{i,j}^{(2:L+2)}\right \rangle-a_{i,j}\Big)^2$$

\textbf{Compute:} $\hat p_{i,j}:=\left \langle \hat \beta^{n, L+1},\hat q_{i,j}^{(2:L+2)} \right \rangle $ for all $i,j$

\textbf{Return:} $\{(i,j) |~\hat p_{i,j}\ge \gamma\}$ the set of predicted edges.
\end{algorithm}

\cref{alg:compute_regression_coefs} estimates the  edge probabilities by regressing the moment estimators $\hat{q}_{i,j}^{(2: L+2)}$ onto the  $a_{ij}$'s. The coefficients of the regression are chosen through constrained optimization. This is necessary due to high multi-collinearity among the observations $\hat{q}_{i,j}^{(2: L+2)}$. Other methods to control the multi-collinearity include using Partial Least Squares (PLS) regression in \cref{alg:compute_regression_coefs}. This leads to an alternative algorithm presented in \cref{alg:pls}, called PLSG-GNN, that is also evaluated in the experiments section.

Before stating our main theorem, we define a few quantities. 
\begin{definition}[MSE error]
For any vector $\beta \in \mathbb{R}^{k}$, define the mean squared error
\begin{equation*}
R_T(\beta)=\E \left[ \left( \left \langle \beta, \hat{q}_{n+1, n+2}^{(2, \rm{len}(\beta)+1)} \right \rangle - W_n(\omega_{n+1},\omega_{n+2}) \right) ^2 \right],
\end{equation*}
\end{definition}
where the expectation is taken with respect to the randomness in $\omega_{n+1}, \omega_{n+2}$. We interpret $n+1, n+2$ as being two new vertices that were not present at training time.  We also define the following quantity, used in the statement of \cref{prop:convergence_of_c}:
$$R(\beta) = \E \left[ \left( \left \langle \beta,  {W}^{(2, \rm{len}(\beta)+1)}(x,y)   \right \rangle - W(x,y) \right) ^2 \right].$$

For some set $\mathcal{F} \subset \mathbb{R}^{k}$, we can interpret $\argmin_{\beta \in \mathcal{F}} R(\beta)$ as the ``$L^2$ projection" of $W(x,y)$ onto the subspace spanned by $\langle \beta, W^{(2: k+1)}(x,y) \rangle$. In particular, if $\mathcal{F}$ contains a vector $\beta^{*, k}$ satisfying \cref{eq:linear-fit}, then $\argmin_{\beta \in \mathcal{F}} R(\beta) = 0.$ In the context of \cref{alg:compute_regression_coefs}, this suggests that we can obtain a consistent estimator for $W_{n, i,j}.$ Hence, since \cref{prop:linear-relationship} guarantees that such a $\beta^{*, k}$ exists when $m_W < \infty$, the intuition is that if both the search space and number of layers $\mathcal{F}, L$ are sufficiently large, then \cref{alg:compute_regression_coefs} should produce estimators $\hat{p}_{i,j}$ that are consistent.


The following theorem shows that this intuition is indeed true. It states that $\hat p_{i,j}$ is a consistent estimator for the edge probability $W_{n, i,j}$ if the number of LG-GNN layers is large enough, and characterizes its convergence rate. We will show this when the search space $\mathcal{F}$ is a rectangle of the form $\mathcal{F}:=\prod_{i=1}^{L+1} [- \frac{b_i}{\rho_n^i} , \frac{b_i}{\rho_n^i}] \subset \mathbb{R}^{L+1} $ for some $b_i>0$. We discuss the implications of this result after its statement.


\begin{theorem}[Main Theorem]
\label{prop:convergence_of_c}
Let $G_n = ([n], E_n)$ be sampled from some graphon $\rho_n W$, where $W$ satisfies (\ref{asp2}) and (\ref{asp3}).  Take $\hat{\beta}^{n, L+1}$ to be the  estimators given by  \cref{alg:compute_regression_coefs}. Define $\beta^{*, L+1} \in \argmin_{\beta \in \mathcal{F}} R(\beta)$ to be the population minimizer. Then, with probability at least $1-5/n - n \cdot \exp(-\delta_W \rho_n (n-1)/3),$ the MSE converges at rate 
\begin{gather*}
R_T(\hat{\beta}^{n, L+1}) \leq R(\beta^{*,L+1}) + \tilde{O}\left( \frac{\kappa_1^2 \rho_n^2}{\sqrt{n}} \right)\\ + \kappa_1\kappa_2 \frac{\rho_n \cdot\log(n)^{L+1}}{\sqrt{n}} \left[\sqrt{\rho_n} + \frac{1}{\sqrt{d_n}} \right],
\end{gather*}
where $\kappa_1=O((1-\delta_W)\sum_{i=1}^{L+1}|b_i|(1-\delta_W)^i)$ and  $\kappa_2=O(\sum_{i=1}^{L+1}|b_i|)$.

\end{theorem}




We remark that when $d_n$ increases quickly enough, the inequality in \cref{prop:convergence_of_c} implies that $$R_T(\hat{\beta}^{n, L+1}) \leq R(\beta^{*,L+1}) + O\left( \frac{\log(n)^{L+1} \cdot \rho_n^{3/2}}{\sqrt{n}} \right).$$ In particular, when $\rho_n \gg \log(n)^{2L+2}/n$, and $d_n$ increases fast enough, then $R_T(\hat{\beta}^{n, L+1})$  $\leq R(\beta^{*,L+1}) + o(\rho_n^2),$ i.e. the MSE decreases faster than the sparsity of the graph.  

As mentioned in the discussion preceeding \cref{prop:convergence_of_c}, if the search space $S$ is large enough to contain some vector $\beta^{*, L+1}$ such that $W(x,y) = \sum_{i=1}^{L+1} \beta^{*, i} W^{(i+1)}(x,y)$ for all $x,y$, then $R(\beta^{*, L+1}) = 0$, and the MSE converges to 0. Notably, \cref{prop:linear-relationship} implies that this search space exists for $L \ge m_W-1$. In this sense, $m_W$ captures the ``complexity" of $W$, and each layer of LG-GNN extracts an additional order of complexity.

When $L = m_W-1$, in order for the search space $S$ to contain the $\beta^{*, m_W}$ defined in \cref{prop:linear-relationship}, we require that $b_i > \beta^{*, m_W}_i,$ where $b_i$ is defined in \cref{prop:convergence_of_c}. Considering the proof of \cref{prop:linear-relationship}, if $m_W < \infty$, then $b_{m_W}$ is on the order of $\frac{1}{|\mu_1 \mu_2 \dots \mu_{m_W}|}$, and hence the constant $\kappa_2$ in \cref{prop:linear-relationship} is on this order as well. This dependence on the inverse of small eigenvalues is a statistical bottleneck; it turns out, however, that if we are concerned only with edge ranking, instead of graphon estimation, this dependence can be greatly reduced. This is outlined in \cref{prop:preserve-rank}.

If \cref{alg:compute_regression_coefs} is used for predicting all the edges that have a probability of more than $\gamma>0$ of existing, then the $0$-$1$ loss will also go to zero.  Indeed for almost every $\gamma>0$ we have $$
\frac{ 1}{n^2}\sum_{i,j\le n}\mathbb{I}(\hat p_{i,j}\ge \gamma)-\mathbb{I}(\rho_n W(
\omega_i,\omega_j)\ge \gamma)\xrightarrow{p}0.$$
Furthermore, if $L< m_W-1$ is smaller than the number of distinct eigenvalues of $W$, then we have $${\small R(\beta^{*,L+1})\le \sqrt{ \sum_{s=1}^{m_W} \left[ \sum_{r=L+1}^{m_W} \beta_r^{*, m_W} \left( \mu_s^{r+1} - \mu_s^{L+1}  \right)  \right]^2}},$$
implying that the the $R(\beta^{*, L+1})$ in \cref{prop:convergence_of_c} decreases as $L$ increases. This latter bound is proved in \cref{bound:gen-error}.

\subsection{Preserving Ranking in Link Prediction}

\cref{prop:convergence_of_c} states that under general conditions, LG-GNN yields a consistent estimator for the underlying edge probability $W_n(\omega_i, \omega_j) = \rho_n W_{i,j}.$ However, estimating edge probabilities is strictly harder than discovering a set of high-probability edges. In practical applications, one often cares about \textit{ranking} the underlying edges, i.e., whether an algorithm can assign higher probabilities to positive test edges than to negative ones. Metrics such as the AUC-ROC and Hits@k capture this notion. The following proposition characterizes the performance of LG-GNN in ranking edges in a $k$-community symmetric SBM. See \cref{sec:sbm} for more details about SBMs. 

Before stating the proposition, we define the following notation for a $k$-community symmetric SBM. Let $S_{in} = \{(i, j) | \text{vertices $i,j$ belong to the same community} \}$, $S_{out} = \{ (k, \ell) |\text{vertices $k, \ell$ are in different communities} \},$ and define $$E_{rank} :=  \left \{ \min_{(i,j) \in S_{in}} \hat{p}_{i,j} > \max_{(k,\ell) \in S_{out}} \hat{p}_{k, \ell} \right\}$$ to be the event that the predicted probabilities for all of the in-community edges are greater than all of the predicted probabilities for the across-community edges, i.e., LG-GNN achieves perfect ranking of the graph edges. We will prove that this event happens with high probability. See \cref{prop:preserve-rank-formal} for the full proposition.

\begin{proposition}[Informal]
\label{prop:preserve-rank}
 Consider a $k$-community symmetric stochastic block model with parameters $p > q$ and sparsity factor $\rho_n.$ Let $\mu_1 = \frac{p+(k-1)q}{k} > \frac{p-q}{k} = \mu_2$ be the eigenvalues of the associated graphon. Suppose that the search space $\mathcal{F}$ is such that  $ \{ \beta \in \mathbb{R}^{L+1} | \text{ } ||\beta||_{L^1} \leq (\mu_1 \rho_n)^{-1} \}\subseteq \mathcal{F}$.

Produce probability estimators $\hat{p}_{i,j}$ for the probability of an edge between vertices $i$ and $j$ using \cref{algo1} and \cref{alg:compute_regression_coefs} with parameters $L, \mathcal{F}$ where $L\ge 1$. Then, there exists a constant $A>0$ such that when 
$$
\frac{\log(n)^{L+1}}{ \rho_n \sqrt{n}} \Big[\sqrt{\rho_n}+\frac{1}{\sqrt{d_n}}\Big]\leq A{\mu_2^3}
$$
holds, then with high probability, $E_{rank}$ occurs, i.e., LG-GNN correctly predicts higher probability for all of the in-community edges than for cross-community edges.
\end{proposition}

\cref{prop:preserve-rank} gives conditions under which LG-GNN achieves perfect ranking on a $k$-community symmetric SBM. One subtle but important point is the implied convergence rate. In \cref{prop:preserve-rank}, the size of the search space is required only to be on the order of $(\mu_1 \rho_n)^{-1}$. In the notation of \cref{prop:convergence_of_c}, this means that the constant $\kappa_2$ is upper bounded by $1/\mu_1$, which indicates a much faster rate of converge than the rate that is required by \cref{prop:convergence_of_c} to define consistent estimators. This confirms the intuition that ranking is easier than graphon estimation, and in particular, should be less sensitive to small eigenvalues. \cref{prop:preserve-rank} demonstrates the extent to which ranking is easier than graphon estimation.

\section{ Performance of the Classical GCN Architecture }
\label{sec:negative-gcn}


As mentioned in \cref{sec:main-results}, in the context of random node initializations, a naive choice of GNN architecture can cause learning to fail. In the following proposition, we demonstrate that for a large class of graphons, the Classical GCN architecture with random initializations results in embeddings that cannot be informative in out-of-sample graphon estimation. To make this formal, we assume that at training, only $n-m$ vertices are observable. We denote by $G_n|_{V_{n-m}}$ the induced subgraph with vertex set $V_{n-m}=\{1,\dots,n-m\}$. And we consider graphons that are such that \begin{align}
    \label{asp4}\tag{$H_4$} \text{The function }W:x\rightarrow \int_0^1W(x,y)dy \text{ is constant.}
\end{align}Note that many graphons satisfy this assumption, including symmetric SBMs. 

\begin{proposition}
\label{prop:negative-result} Suppose that the graph $G_n=([n],E_n)$ is generated according to a graphon $W_n=\rho_nW$. Moreover assume that Assumptions (\ref{asp1}), (\ref{asp2}), (\ref{asp4}) hold.

Suppose that the initial embeddings $(\lambda_i^0)\overset{i.i.d}{\sim}\mu$ are so that each coordinate is generated i.i.d. from a $\frac{s^2}{\sqrt{d_n}}$ sub-Gaussian distribution. Assume that the subsequent embeddings $(\lambda_i^\ell)$ are computed iteratively according to \cref{GCN-equation}, where $\sigma(\cdot)$ is taken to be $1-$Lipschitz and where the weight matrices $(M_{k,0}, M_{k,1})$ are trained on $G_n|_{V_{n-m}}$ and satisfy $\|M_{k,0}\|_{\rm{op}},\|M_{k,1}\|_{\rm{op}}\overset{a.s}{\le} M$.

Then, there exist random variables $\mu_{n}^\ell,$ $\ell \in [L]$, that are independent of $\omega_{n-m+1},\dots,\omega_{n}$ such that for a certain $\kappa>0$ with probability at least $1-\frac{2}{n}-2ne^{-\frac{12\log(n)}{\rho_n}}-2ne^{-\kappa d_n}$,
\begin{equation}
\sup_{{\ell \leq L}}\| \lambda_{n}^L - \mu_{n}^\ell \| \leq \frac{K}{\sqrt{\rho_n(n-1)}},
\end{equation} where $K>0$ is an absolute constant.


\end{proposition}

 We show that this leads to suboptimal risk for graphon estimation. For simplicity we show this for dense graphons, e.g when $\rho_n=1$.
 \begin{proposition}
 \label{negative-lipschitz}
     Suppose that the conditions of \cref{prop:negative-result}
hold. Moreover assume that the graphon $W(\cdot,\cdot)$ is not constant and that $\rho_n=1$ for all $n\in \mathbb{N}$. Then, there exists some constant $K>0$ such that for any Lipchitz prediction rule $f(\cdot,\cdot)$, for all vertices $i\in [n],$ we have 
\begin{align}&
   \mathbb{E}\Big(\big[W(\omega_i,\omega_n)-f(\lambda_i^L,\lambda_{n}^L)\big]^2\Big)\ge 
   K+o_n(1). 
\end{align}
\end{proposition}

\cref{prop:negative-result} and \cref{negative-lipschitz} imply that in the out-of-sample setting, the embeddings produced by \cref{GCN-equation} with random node feature initializations will lead to sub-optimal estimators for the edge probability $W(\omega_i,\omega_i)$. A key feature of the proof of \cref{prop:negative-result} is that $\sum_{u \in N(v)} \frac{\lambda_u^{k-1}}{\sqrt{|N(u)\|N(v)|}}$ concentrates to 0 very quickly with random node initializations. This demonstrates the importance of the (subtle) construction of the first round of message passing $\lambda_u^0$ in \cref{algo1}. We also note that \cref{negative-lipschitz} doesn't necessarily imply that predicted probabilities $\hat{p}_{e_i}$ will be ineffective at ranking test edges, though we do see in the experiments that the performance is decreased for the out-of-sample case.

\section{Identifiability and Relevance to Common Random Graph Models}\label{sec:identifiability}

We remark that a key feature of LG-GNN is that it uses the embedding vector $\lambda_i^k$ produced at each layer. Indeed $\hat{p}_{i,j}$ depends on all of the terms $\{ \langle \lambda_i^0, \lambda_j^0 \rangle, \langle \lambda_i^0, \lambda_j^1 \rangle, \dots, \langle \lambda_i^0, \lambda_j^{L} \rangle \}$. This is in contrast to many classical ways of using GNNs for link prediction that depend only on $\langle \lambda_i^L, \lambda_j^L \rangle.$ The following proposition shows that this construction is necessary to obtain consistent estimators.




\begin{proposition}
\label{prop:idenfiability}
For any $L \ge 0$, there exists a 2-community stochastic block model,  such that for every continuous function $f:\mathbb{R}\rightarrow\mathbb{R}$ we have $$f(\langle \lambda_i^L, \lambda_j^L \rangle) \stackrel{p}{\not\to} W(\omega_i, \omega_j).$$This notably implies that $$\liminf_{n,d_n\rightarrow\infty}\inf_{f\in c^0(\mathbb{R})}\mathbb{E}\Big(\big(f(\langle \lambda_i^L, \lambda_j^L \rangle)-W(\omega_i,\omega_j)\big)^2\Big)>0$$
\end{proposition}
To illustrate this, consider the following example for $L=0.$
\begin{example}
\label{example:low-rank}
    Consider an 2 community symmetric SBM with edge connection probability matrix $\begin{pmatrix} 1/2 & 1/4 \\ 1/4 & 3/4\end{pmatrix}.$ The matrix of second moments, that is, the matrix of probabilities of paths of length two between members of the two communities is $\begin{pmatrix} 5/32 & 5/32 \\ 5/32 & 5/16 \end{pmatrix}.$ Hence for any continuous function $f$ and every vertex $i,j,k$ belonging respectively to communities $1$ for $i,j$ and $2$ for $k$ then we have that $f(\langle\omega_i^0,\omega_j^0\rangle)\xrightarrow{p} f(5/32)$ and  $f(\langle\omega_i^0,\omega_k^0\rangle)\xrightarrow{p} f(5/32)$ have the same limit. Hence this implies that no consistent estimator of $W(\cdot,\cdot)$ can be built by using only $(\langle\lambda_i^0,\lambda_j^0\rangle).$
\end{example} 


While this result is for the specific case of \cref{algo1}, which in particular contains no non-linearities, we anticipate that this general procedure of learning a function that maps a set of dot products $\{ \langle \lambda_i^{k_1}, \lambda_j^{k_2} \rangle \}_{k_1, k_2}$ to a predicted probability, instead of just $\langle \lambda_i^L, \lambda_j^L \rangle$, can lead to better performance for practioners on various types of GNN architectures.

\section{Experimental Results}
\label{sec:mainbodyexperiments}
We compare experimentally a GCN, LG-GNN, and PLSG-GNN. We perform experiments on the Cora dataset \cite{mccallum2000automating} in the in-sample setting. We also show results for various random graph models. The results for random graphs below are in the out-sample setting, more results are in \cref{sec:experiments}. We report the AUC-ROC and Hits@k metric, and also a custom metric called the Probability Ratio@k, which is more suited to the random graph setting. We refer the reader to \cref{sec:experiments} for a more complete discussion.

LG-GNN and PLSG-GNN perform similarly to the classical GCN in settings with no node features and can outperform it on more complex graphons. One major advantage of LG-GNN/PLSG-GNN is that they do not require extensive tuning of hyperparameters (e.g., through minimizing a loss function) and hence run much faster and are easier to fit. For example, training the 4-layer GCN resulted in convergence issues, even on a wide set of learning rates.



\subsection{Real Data: Cora Dataset}

The following results are in the in-sample setting. We consider when (a) the GCN has access to node features (b) the GCN does not. 

\begin{table}[h]
\caption{GCN has no access to node features}
\label{tab:simulation_results}
\begin{tabular}{llll}
\toprule
Params & Model & Hits@50 & Hits@100 \\
\midrule
\multirow{3}{*}{layers=2} 
 & GCN & 0.496 $\pm$ 0.025 & 0.633 $\pm$ 0.023 \\
 & LG-GNN & 0.565 $\pm$ 0.012 & 0.637 $\pm$ 0.006 \\
 & PLSG-GNN & \textbf{0.591} $\pm$ 0.014 & \textbf{0.646} $\pm$ 0.013 \\
\cline{1-4}
\multirow{3}{*}{layers=4}
 & GCN & 0.539 $\pm$ 0.008 & \textbf{0.665} $\pm$ 0.007 \\
 & LG-GNN & 0.564 $\pm$ 0.005 & 0.620 $\pm$ 0.008 \\
 & PLSG-GNN & \textbf{0.578} $\pm$ 0.014 & 0.637 $\pm$ 0.013 \\
\bottomrule
\end{tabular}
\end{table}
\FloatBarrier

\begin{table}[h]
\caption{GCN has access to node features}
\label{tab:simulation_results}
\begin{tabular}{llll}
\toprule
Params & Model & Hits@50 & Hits@100 \\
\midrule
\multirow{3}{*}{layers=2} 
 & GCN & \textbf{0.753} $\pm$ 0.019 & \textbf{0.898} $\pm$ 0.021 \\
 & LG-GNN & 0.555 $\pm$ 0.027 & 0.603 $\pm$ 0.034 \\
 & PLSG-GNN & 0.577 $\pm$ 0.033 & 0.626 $\pm$ 0.042 \\
\cline{1-4}
\multirow{3}{*}{layers=4} 
 & GCN & \textbf{0.609} $\pm$ 0.072 & \textbf{0.776} $\pm$ 0.069 \\
 & LG-GNN & 0.560 $\pm$ 0.013 & 0.601 $\pm$ 0.012 \\
 & PLSG-GNN & 0.574 $\pm$ 0.025 & 0.625 $\pm$ 0.024 \\
\bottomrule
\end{tabular}
\end{table}
\FloatBarrier

\subsection{Synthetic Dataset: Random Graph Models}

\subsubsection{10-Community Symmetric SBM}

The following are results for a 10-community stochastic block model with parameter matrix $P$ that has randomly generated entries. The diagonal entries $P_{i,i}$ are generated as $\text{Unif}(0.5, 1)$, and $P_{i,j}$ is generated as $\text{Unif}(0, \min(P_{i,i}, P_{j,j}))$. The specific connection matrix that was used is in \cref{sec:experiments}.

\begin{table}[h]
\caption{$\rho_n=1$}
\label{tab:simulation_results}
\begin{tabular}{llll}
\toprule
Params & Model & P-Ratio@100 & AUC-ROC \\
\midrule
\multirow{3}{*}{layers=2} 
 & GCN & 0.709 $\pm$ 0.125 & 0.716 $\pm$ 0.019 \\
 & LG-GNN & 0.883 $\pm$ 0.016 & 0.734 $\pm$ 0.005 \\
 & PLSG-GNN & \textbf{0.886} $\pm$ 0.016 & \textbf{0.735} $\pm$ 0.005 \\
\cline{1-4}
\multirow{3}{*}{layers=4} 
 & GCN & 0.645 $\pm$ 0.025 & 0.578 $\pm$ 0.109 \\
 & LG-GNN & 0.879 $\pm$ 0.011 & \textbf{0.786} $\pm$ 0.002 \\
 & PLSG-GNN & \textbf{0.883} $\pm$ 0.013 & 0.732 $\pm$ 0.001 \\
\cline{1-4}
\bottomrule
\end{tabular}
\end{table}
\FloatBarrier

\begin{table}[h]
\caption{$\rho_n=1/\sqrt{n}$}
\label{tab:simulation_results}
\begin{tabular}{llll}
\toprule
Params & Model & P-Ratio@100 & AUC-ROC \\
\midrule
\multirow{3}{*}{layers=2} 
 & GCN & 0.344 $\pm$ 0.021 & 0.493 $\pm$ 0.004 \\
 & LG-GNN & 0.580 $\pm$ 0.020 & 0.497 $\pm$ 0.009 \\
 & PLSG-GNN & \textbf{0.586} $\pm$ 0.035 & \textbf{0.521} $\pm$ 0.008 \\
\cline{1-4}
\multirow{3}{*}{layers=4} 
 & GCN & 0.285 $\pm$ 0.016 & 0.486 $\pm$ 0.006 \\
 & LG-GNN & \textbf{0.589} $\pm$ 0.016 & \textbf{0.532} $\pm$ 0.003 \\
 & PLSG-GNN & 0.578 $\pm$ 0.013 & 0.508 $\pm$ 0.011 \\
\cline{1-4}
\bottomrule
\end{tabular}
\end{table}
\FloatBarrier



\subsubsection{Geometric Graph}

Each vertex $i$ has latent feature $X_i$ generated uniformly at random on $\mathbb{S}^{d-1}$, $d=11.$ Two vertices $i$ and $j$ are connected if $\langle X_i, X_j \rangle \ge t = 0.2,$ corresponding to a connection probability $\approx 0.26.$ Higher sparsity is achieved by adjusting the threshold $t$. 

\begin{table}[h]
\caption{$\rho_n=1$}
\label{tab:simulation_results}
\begin{tabular}{llll}
\toprule
Params & Model & P-Ratio@100 & AUC-ROC \\
\midrule
\multirow{3}{*}{layers=2} 
 & GCN & \textbf{1.000} $\pm$ 0.000 & 0.873 $\pm$ 0.020 \\
 & LG-GNN & \textbf{1.000} $\pm$ 0.000 & 0.915 $\pm$ 0.007 \\
 & PLSG-GNN & 0.997 $\pm$ 0.005 & \textbf{0.917} $\pm$ 0.010 \\
\cline{1-4}
\multirow{3}{*}{layers=4} 
 & GCN & 0.813 $\pm$ 0.021 & 0.591 $\pm$ 0.016 \\
 & LG-GNN & \textbf{1.000} $\pm$ 0.000 & 0.956 $\pm$ 0.001 \\
 & PLSG-GNN & \textbf{1.000} $\pm$ 0.000 & \textbf{0.958} $\pm$ 0.001 \\
\cline{1-4}
\bottomrule
\end{tabular}
\end{table}

\begin{table}[!h]
\caption{$\rho_n=1/\sqrt{n}$}
\label{tab:simulation_results}
\begin{tabular}{llll}
\toprule
Params & Model & P-Ratio@100 & AUC-ROC \\
\midrule
\multirow{3}{*}{layers=2} 
 & GCN & 0.333 $\pm$ 0.017 & 0.840 $\pm$ 0.008 \\
 & LG-GNN & \textbf{0.523} $\pm$ 0.037 & 0.818 $\pm$ 0.022 \\
 & PLSG-GNN & 0.423 $\pm$ 0.054 & \textbf{0.842} $\pm$ 0.017 \\
\cline{1-4}
\multirow{3}{*}{layers=4} 
 & GCN & 0.313 $\pm$ 0.021 & \textbf{0.848} $\pm$ 0.021 \\
 & LG-GNN & \textbf{0.570} $\pm$ 0.016 & 0.823 $\pm$ 0.010 \\
 & PLSG-GNN & 0.510 $\pm$ 0.014 & 0.843 $\pm$ 0.013 \\
\cline{1-4}
\bottomrule
\end{tabular}
\end{table}
\FloatBarrier

\section{Impact Statement}

This paper presents work whose goal is to advance the field of Machine Learning. There are many potential societal consequences of our work, none which we feel must be specifically highlighted here.

\section{Acknowledgements}

The first author would like to thank Qian Huang for helpful discussions regarding the experiments. The authors would also like to thank the Simons Institute for the Theory of Computing, specifically the program on Graph Limits and Processes on Networks. Part of this work was done while Austern and Saberi were at the Simon's Institute for the Theory of Computing. 

This research is supported in part by the AFOSR under Grant No. FA9550-23-1-0251, and by an ONR award N00014-21-1-2664.


\bibliography{refs}
\bibliographystyle{icml2023}

\newpage
\appendix
\onecolumn



\section{Notation and Preliminaries}

We let $\|\cdot\|_{p}$ be the vector Euclidean norm. The $L^p$ norm over the probability space will be denoted $\|X\|_{L^p} = \left( \E[|X|^p] \right)^{1/p}$.

\subsection{Graph Notation}

We let $A = (a_{ij})_{i,j=1}^n$ be the adjacency matrix of the graph. Let $(\omega_i)_{i=1}^n$ be the latent features of the vertices, generated from $\text{Unif}(0,1).$ Let $W$ be the graphon, and let $\rho_n$ be the sparsifying factor. We denote $W_n := \rho_n W$ (we will typically be concerned only with $W_n$, since that is the graphon from which the graph is generated). Let $N(i)$ be the set of neighbors of a vertex $i$, and hence $|N(i)|$ is the degree of $i$.

We define the $k$th moment of a graphon $W_n$ to be the function from $[0,1]^2 \to [0,1]$ given by
\begin{equation}
\label{defn:moment}
    W^{(k)}_n(x,y) :=  \int_{[0,1]^{k-1}} W_n(x, t_1) W_n(t_1, t_2) \dots W_n(t_{k-1}, y) \text{d}t_1 \dots \text{d}t_{k-1}.
\end{equation}
Heuristically, if one fixes two vertices $v_x, v_y$ with latent features $x,y$, then this is the probability of a particular path of length $k$ from $v_x, v_y$, when averaging over the possible latent features of the vertices in the path. Correspondingly, we define the empirical $k$th moment between two vertices $i$ and $j$ to be 
\begin{equation}
\label{defn:empirical-moment}
    \hat{W}_{n,i,j}^{(k)} = \frac{1}{(n-1)^{k-1}} \sum_{r_1, \dots, r_{k-1} \leq n} a_{i r_1} a_{r_1 r_2} \dots a_{r_{k-1} j}.
\end{equation}

\subsection{GNN Notation}

We let $\lambda_i^k$ be the embedding for the $i$th vertex produced by the GNN after the $k$th layer. The linear GNN architecture is given by \cref{gnn-iteration-general-weights}:
\begin{equation}
\lambda_i^k = M_{k, 0} \lambda_i^{k-1} + M_{k, 1} \frac{1}{n-1} \sum_{\ell \leq n} a_{i\ell} \lambda_\ell^{k-1},\end{equation}
where $M_{k,0}$ and $M_{k,1}$ denote the weight matrices of the GNN at the $k$th layer. We remark that LG-GNN corresponds to $M_{k,0}=M_{k,1}=\mathbb{I}d_{d_n}$ being the identity matrix. Let 
\begin{equation}
\label{defn:N-sumproduct-matrices}
N_s^k := \sum_{\stackrel{r_1, \dots, r_k \in \{0,1\}}{\sum_{i=1}^k r_i = s} } M_{k, r_1} M_{k-1, r_2} \dots M_{1, r_k},
\end{equation}
which is a quantity that shows up naturally in the GNN iteration. The classical GCN architecture we consider is also given in \cref{GCN-equation}.

\subsection{Stochastic Block Model}
\label{sec:sbm}

We define the stochastic block model, as it is a running model in this paper.

A stochastic block model $\rm{SBM}(n, P)$ is parameterized by the number of vertices $n$ in the graph and a connection matrix $P \in \mathbb{R}^{k \times k}$. Each vertex belongs in a particular \textit{community}, labeled $\{ 1, 2, \dots, k\}$. We assign each vertex to belong to community $j$ with probability $p_j$. In this paper, we choose $p_j = 1/k$ for all $j \in [k]$. Let $c_i$ denote the community of the $i$th vertex. The graph is generated as follows. For each pair of vertices $i \neq j$, we connect them with an edge with probability $P_{c_i, c_j}$. We also denote the \textit{symmetric} stochastic block model by $\rm{SSBM}(n, p, q)$. The SSBM is a stochastic block model with only two parameters: the parameter matrix $P$ is so that $P_{ii} = p$, $P_{ij} = q$ if $i \neq j$.


The following lemma details how to represent a SBM using a graphon.

\begin{lemma}
\label{lemma:sbm}
   Consider a stochastic block model $\rm{SBM}(n, P)$. Suppose that $P \in \mathbb{R}^{k \times k}$ is a symmetric matrix and that $P$ has spectral decomposition $P = \sum_{i=1}^k \lambda_i v_i v_i^T,$ where $\|v_i\|_2 = 1.$ Let $W: [0,1]^2 \to [0,1]$ be the corresponding graphon. 
   
   Then $\rm{SBM}(n, P)$ can be represented by a graphon as follows. $W(x,y) = P_{ij}$ if $x \in [(i-1)/k, i/k]$ and $y \in [(j-1)/k, j/k].$ The eigenvalues of $W$ are given by $\mu_i := \lambda_i/k$ with corresponding eigenfunctions $\phi_i(x)$, where $\phi_i(x) = \sqrt{k} (v_i)_j$ if $x \in [(j-1)/k, j/k].$

    This further implies that the eigenfunctions $\phi_i(x)$ are bounded above pointwise by $\sqrt{k}$, in that $|\phi_i(x)| \leq \sqrt{k}$ for all $i \in [k]$ and $x \in [0,1],$ since $\|v_i\|_2 = 1$, implying that each entry of $v_i$ has norm bounded above by 1.
\end{lemma}

This proof of this lemma is a simple verification of the properties. We note that the eigenfunctions are scaled by $\sqrt{k}$ because they integrate to $1$ in $L^2([0,1]).$

\subsection{PLSG-GNN Algorithm}

We state the PLSG-GNN Algorithm, which is an analog of \cref{alg:compute_regression_coefs} that uses Partial Least Squares Regression (PLS). Let $g:\mathbb{N}^2 \to \mathbb{N}$ be an enumeration of the pairs $(i, j)$, $i \neq j.$ In the algorithm below, let PLS denote the Partial Least Squares algorithm as introduced in \cite{PLS-reference}.

\begin{algorithm}[h]
\caption{LG-GNN edge prediction algorithm}
\label{alg:pls}
\textbf{Input:}  Graph $G=(V,E)$, set $S$, threshold $\beta$; $L$. \\
\textbf{Output:} Set of predicted edges 


Using \cref{algo1}, compute $q_{i,j}^{(2,L+2)} := (\hat q_{i,j}^{(2)},\dots, q_{i,j}^{(L+2)})$ for every vertex $i,j$. Define the matrix $\hat{Q}$ and vector $\vec{a}$, as $\hat{Q}_{g(i, j)} := \left( q_{i,j}^{(2:L+2)} \right)$, $i < j$, and $\vec{a}_{g(i, j)} = a_{i,j},$ $i < j$.

\textbf{Compute:} $$\hat{\beta}^{n, L+1} = \rm{PLS} (\hat{Q}, \vec{a}).$$

\textbf{Compute:} $\hat p_{i,j}:=\left \langle  \hat{\beta}^{n, L+1},\hat q_{i,j}^{(2:L+2)} \right \rangle $ for all $i,j$

\textbf{Return:} $\{(i,j) |~\hat p_{i,j}\ge \gamma\}$ the set of predicted edges.
\end{algorithm}

\section{Properties of Holder-by-Parts Graphons}\label{graphon-appendix}

Here, we discuss properties of symmetric, piecewise-Holder graphons. This section is largely from \cite{morgane_paper}, Appendix H. Refer to that text for a more complete exposition; we just present the details most relevant to our needs.

Let $\mu$ be the Lebesgue measure. We define a partition $\mathcal{Q}$ of $[0,1]$ to be a finite collection of pairwise disjoint, connected sets whose union is $[0,1]$, such that , for all $Q \in \mathcal{Q}$, $\mu(\text{int}(Q)) > 0$ and $\mu(\text{cl}(Q) \backslash \text{int}(Q)) = 0.$ This induces a partition $\mathcal{Q}^{\otimes 2} = \mathcal{Q} \otimes \mathcal{Q}$ of $[0,1]^2.$ We say that a graphon $W$ lies in the Holder class $\text{Holder}([0,1]^2, \beta, M, \mathcal{Q}^{\otimes 2})$ if $W$ is $(\beta, M)$ Holder continuous in each $Q_i \otimes Q_j \in \mathcal{Q}^{\otimes 2}.$ All graphons in question in this paper are assumed to belong to this class.

A graphon $W$ can be viewed as an operator between $L^p$ spaces. In this paper, we focus on the case of $p=2.$ In particular, for a fixed Graphon $w$, one can define the Hilbert-Schmidt operator $$T_W[f](x) := \int_0^1 W(x,y) f(y) dy.$$ Since $W$ is symmetric, $T$ is self-adjoint. Furthermore, because $W(\cdot, \cdot) \leq 1,$ $T_W$ is a compact operator, as in \cite{stein2009real} page 190. Hence, the spectral theorem (for example, in \cite{fabian2013functional}, Theorem 7.46) states that there exists a sequence of eigenvalues $\mu_i \to 0$ and eigenvectors $\phi_i$ (that form an orthonormal basis of $L^2([0,1])$, such that 
\begin{equation}
    T_W[f] = \sum_{n=1}^\infty \mu_n \langle f, \phi_n\rangle \phi_n, \quad W(x,y) =\sum_{n=1}^\infty \mu_n \phi_n(x) \phi_n(y),
\end{equation}
and $\sum_{n=1}^\infty \mu_n^2 < \infty.$ We note also that $|\mu_i| \leq 1.$ This is because if $\mu_i$ is an eigenvalue, then $$\int_0^1 W(x,y) \phi_i(y) \text{d}y = \mu_i \phi_i(x) \Rightarrow \mu_i^2 \phi_i(x)^2 = \left( \int_0^1 W(x,y) \phi_i(y) \text{d}y \right)^2 \leq 1,$$
where the last inequality is true because $W$ is bounded by $1$ and $\phi_i(y)^2$ integrates to 1. Then, since $\phi_i(x)^2$ also integrates to 1, this shows the result.

\subsection{Linear Relationship Between Moments and $W$ (Proof of \cref{prop:linear-relationship})}

\begin{proof}[Proof of \cref{prop:linear-relationship}]
Suppose that $m_W < \infty$ is the number of distinct nonzero eigenvalues of $W$, and label them by $|\mu_1| \ge |\mu_2| \ge \dots \ge |\mu_{m_W}|.$ Recall that
\begin{equation*}
    W(x,y) = \sum_{i=1}^{m_W} \mu_i \phi_i(x) \phi_i(y).
\end{equation*}
We first prove via induction that 
$$
W^{(k)}(x,y) = \sum_{i=1}^{m_W} \mu_i^k \phi_i(x) \phi_i(y).
$$
Assume that this is true for $k \in \{1, 2, \dots, K\}$. Now we show that $W^{(K+1)}(x,y) = \sum_{i=1}^{m_W} \mu_i^{K+1} \phi_i(x) \phi_i(y)$. Because the $\phi_i$ are orthonormal in $L^2([0,1]),$ we can compute 
\begin{align*}
    W^{(K+1)}(x,y) &= \int_0^1 W^{(K)}(x, t) W(t, y) \text{d} t \\
    &= \int_0^1 \left( \sum_{i=1}^{m_W} \mu_i^K \phi_i(x) \phi_i(t) \right) \cdot \left( \sum_{i=1}^{m_W} \mu_i \phi_i(y) \phi_i(t) \right) \text{d} t \\
    &= \int_0^1 \sum_{i, j}^{m_W} \mu_i^K \mu_j \phi_i(x) \phi_j(y) \phi_i(t) \phi_j(t) \text{d} t \\
    &= \sum_{i=1}^{m_W} \mu_i^{K+1} \phi_i(x) \phi_i(y),
\end{align*}
where the last equality is due to the orthonormality of the $\phi_i$ in $L^2([0,1]).$ This completes the induction. We now argue that there is a linear relationship between $W(x,y)$ and $(W^{(2)}(x,y), \dots, W^{(m_W+1)}(x,y)),$ i.e., there exists some $\beta^{*, m_W}$ such that 
$$W(x,y) = \sum_{i=1}^{m_W} \beta^{*, m_W}_i W^{(i+1)}(x,y)$$ for all $x,y.$ In light of the above discussion, we can observe that the vector $\beta^{*, m_W} = \left(\beta^{*, m_W}_1, \beta^{*, m_W}_2 \dots, \beta^{*, m_W}_{m_W} \right)$ is simply the solution (if it exists) to the system of equations
$$
\label{eq:linear-system}
    \begin{pmatrix}
        \mu_1^2 & \mu_1^3 & \dots & \mu_1^{m_W+1} \\
        \mu_2^2 & \mu_2^3 & \dots & \mu_2^{m_W+1} \\ 
        \vdots & \vdots & \ddots & \vdots \\ 
        \mu_{m_W}^2 & \mu_{m_W}^3 & \dots & \mu_{m_W}^{m_W+1}
    \end{pmatrix} \begin{pmatrix} \beta_1 \\ \beta_2 \\ \vdots \\ \beta_{m_W} \end{pmatrix}
    = \begin{pmatrix} \mu_1 \\ \mu_2 \\ \vdots \\ \mu_{m_W} \end{pmatrix}
$$  
To observe that a solution indeed exists, it suffices to observe that the matrix on the LHS is of full rank, i.e., has nonzero determinant. To see this, we note that the $i$th row is a multiple of $v_i := (1, \mu_i, \dots, \mu_i^{m_W-1}).$ We note that the matrix whose $i$th row is $v_i$ is a Vandermonde matrix, which has nonzero determinant if all of the variables are distinct. Then, since multiplying each row by a constant changes the determinant only by a multiplicative factor, this suffices for the proof. 
\end{proof}

\section{Proof of \cref{prop:negative-result} and \cref{negative-lipschitz}}
\label{appendix:negative}



We let $A$ denote the adjacency matrix. In the proof, for random variables $X$ and Borel sets $B$, we might write quantities of the form $\P(X \in B | A)$. This denotes a conditional probability, where we condition on the realization of the graph. A notation we also use is $X | A \sim \text{Dist}$, which denotes the conditional distribution of a random variable $X$, conditioned on the realization of the graph.
We first state the following Lemma, used in the proof of \cref{prop:negative-result}. 

\begin{lemma}
\label{lemma:negative-result}
Suppose that $G_n=([n],E_n)$ is generated from the graphon $W_n(\cdot,\cdot)=\rho_n W(\cdot,\cdot)$. Write $W(\omega_i,\cdot):=\int_0^1W(\omega_i,x)dx$. Then we have that
    \begin{equation}
 \P \left( \sup_{i\le n}\frac{1}{n-1}\Big||N(i)|-\rho_n W(\omega_i,\cdot) \Big|\ge \rho_n t  \right)\le 2n\big(e^{-\frac{(n-1)\rho_nt^2}{3}}+e^{-{2(n-1)t^2}}\big)
    \end{equation}
\end{lemma}

\begin{proof}[Proof of Lemma \ref{lemma:negative-result}]
We first show the above result for a fixed vertex $i$ (without loss of generality, let $i=n$), and then conclude the proof through a union bound. We first state

\begin{lemma}[\cite{chernoff-lecture}, Theorem 4]
\label{bernoulli-hoeffding}
    Let $X = \sum_{i=1}^n X_i$, where $X_i \sim \text{Bern}(p_i),$ and all the $X_i$ are independent. Let $\mu = \E[X] = \sum_{i=1}^n d_n.$ Then $$\P( |X - \E[X]| \ge \delta \mu) \leq 2\exp \left( -\mu \delta^2/3 \right)$$ for all $\delta > 0.$
\end{lemma}

Now suppose that the latent feature $\omega_n$ is fixed. For any vertex $j \neq n$, we have
\begin{align}
    \P(a_{jn} = 1|\omega_n) 
    &= \int_0^1 W_n(\omega_n, x) dx \\
    &=\rho_n W(\omega_n,\cdot).
\end{align}
Recall that $|N(n)| = \sum_{j \neq n} a_{jn}$ hence $\mathbb{E}(|N(n)|\big|\omega_n)=\rho_n W(\omega_n,\cdot)$. We show that $|N(n)|$ concentrates around $\rho_n W(\omega_n,\cdot)$. In this goal, remark that $|N(n)|\big|(\omega_i)$ is distributed as a sum of independent Bernoulli random variables with probabilities $\rho_n W(\omega_i,\omega_n)$. Therefore, according to \cref{bernoulli-hoeffding}, for all $t\in (0,1)$ we have 
\begin{align}
    \P\Big(\frac{1}{n-1}\Big||N(n)|-\mathbb{E}\big(|N(n)|\big|(\omega_i)\big)\Big|\ge \rho_n t\Big)\le 2 e^{-\frac{(n-1)\rho_n t^2}{3}}.
\end{align}Moreover, we remark that conditionally on $\omega_n$, the random variables $(W(\omega_n,\omega_j))_{j\ne n}$ are i.i.d. Therefore, according to Hoeffding's inequality, we remark that for all $t>0$, we have 
    \begin{equation}
        \P \left( \frac{1}{n-1}\Big|\mathbb{E}\big(|N(n)|\big|(\omega_i)\big)-\mathbb{E}(|N(n)|\big|\omega_n)\Big|\ge \rho_n t  \right) \le 2 \exp \left( -{2(n-1) t^2}\right) .
    \end{equation}
Using the union bound this directly implies that for all $t\in (0,1)$ we have 

\begin{align}
     \P \left( \sup_{i\le n}\frac{1}{n-1}\Big||N(i)|-\rho_n W(\omega_i,\cdot) \Big|\ge \rho_n t  \right)\le 2n\big(e^{-\frac{(n-1)\rho_nt^2}{3}}+e^{-{2(n-1)t^2}}\big)
\end{align}
\end{proof}

\begin{proof}[Proof of \cref{prop:negative-result}]The proof proceeds through induction. Let $A>0$ be a constant such that $\sqrt{\frac{A\log(n)}{\rho_n(n-1)}}\le  \delta_W/2$. We denote the event $$E:=\Big\{ \sup_{i \in [n]} \Big||N(i)| -(n-1)\rho_nW(\omega_i,\cdot)\Big|\ge  \sqrt{\rho_n  (n-1)}\sqrt{A\log(n)}\Big\}.$$ We remark that according to \cref{lemma:negative-result}  we have  $\P(E^c)\le 2n\big(e^{-\frac{A\log(n)}{3}}+e^{-\frac{2A\log(n)}{\rho_n}}\big)$.
For the remainder of the proof we will work under the event $E$. 
Note that when $E$ holds this also implies that 
\begin{equation}
    \inf_{i \in [n]} |N(i)| \ge \frac{1}{2} \rho_n \delta_W (n-1)
\end{equation}

For ease of notation we define $\xi:=\frac{2\sqrt{2}(s\wedge 1)M}{\sqrt{\delta_W}}$ and  write $$\epsilon(n,k):=\frac{\xi}{\sqrt{\rho_n(n-1})}\Big\{1+\Big((2M)^{k-1} -1\Big)\Big(1+\frac{2M^L(s\wedge 1)}{4{\delta_W}}\sqrt{1+\frac{\sqrt{A\log(n)}}{\sqrt{\rho_n(n-1)}}}{\sqrt{A\log(n)}}\Big(\frac{1+\sqrt{\frac{A\log(n)}{{(n-1)\rho_n}}}}{\overline{W}}\Big)^{L-2}\Big) \Big\}$$
We will then show that there is a constant $\kappa>0$ so that, conditional on $E$ holding, with a probability of at least $1-2ne^{-\kappa d_n}$ there exists embedding vectors $(\mu_i^k)$ that are independent from $\omega_{n-m+1:n}$ such that for every $k\le L$ we have $$\sup_{i\le n}\|\lambda_i^k-\mu_i^k\|_{2}\le\epsilon(n,k).$$ To do so we proceed by induction. 
Firstly, since $\sigma(\cdot)$ is Lipschitz, we observe that for all $i\le n$ that we have
\begin{equation}\label{lipschitz_small}
\Big\|\lambda_{i}^1 - \sigma\left( M_{1,0} \lambda_{i}^0 \right) \Big\|_{2} \leq\Big\| M_{1, 1}  \sum_{\ell \leq n} \frac{a_{i\ell} \lambda_\ell^{k-1}}{\sqrt{|N(i)\|N(\ell)|}}  \Big\|_{2},
\end{equation}
which we will show is bounded by $ O\left( \sqrt{\sum_{\ell \in N(i)} \frac{1}{|N(i)\|N(\ell)|} } \right)$ with high probability. Using the hypothesis that $\|M_{1,1}\|_{\rm{op}}\overset{a.s}{\le} M$, we note that 
\begin{align}
\Big\|M_{1,1} \sum_{\ell \leq n} \frac{a_{i\ell} \lambda_\ell^{k-1}}{\sqrt{|N(i)\|N(\ell)|}} \Big\|_{2} &\leq \|M_{1,1}\|_{op}\Big\| \sum_{\ell \leq n } \frac{a_{i\ell} \lambda_\ell^{k-1}}{\sqrt{|N(i)\|N(\ell)|}} \Big\|_{2} \\
&\leq M\Big\| \sum_{\ell \leq n} \frac{a_{i\ell} \lambda_\ell^{k-1}}{\sqrt{|N(i)\|N(\ell)|}}\Big \|_{2}.
\end{align}


To bound this last quantity, we note that conditioned on $G_n$, we have that $\sum_{\ell \leq n} \frac{a_{i\ell} \lambda_{\ell}^{0}}{\sqrt{|N(\ell)|}} $ is a$\sqrt{\sum_{\ell\in N(i)}\frac{s^2}{d_n|N(i)\|N(\ell)|}}$-sub-Gaussian vector with i.i.d entries.  We will therefore use the following lemma



\begin{lemma}
\label{lemma:concentration-gaussianvec}
Suppose that $X\in \mathbb{R}^{d_n}$ is a $\eta/\sqrt{d_n}$ sub-Gaussian vector with i.i.d coordinates. There exists some universal constant $\kappa>0$ such that 
\begin{equation}
\P \left(\big|\|X\|_{2}-\mathbb{E}(\|X\|_{2})\ge t\right)\big|\leq 2 \exp \left( - \frac{\kappa d_n t^2}{\eta^2} \right)\end{equation}
\end{lemma}
\begin{proof}[Proof of \cref{lemma:concentration-gaussianvec}] 
This is a direct consequence of Theorem 3.1.1 from \cite{vershynin2018high}.

\end{proof}
We remark that $$\mathbb{E}\Big(\Big\|\sum_{\ell \leq n} \frac{a_{i \ell} \lambda_\ell^0}{\sqrt{|N(i)| |N(\ell)|}}\Big\|_{2}\Big|G_n\Big)\le \sqrt{\sum_{\ell\in N(i)}\frac{s^2}{|N(i)||N(\ell)|}}.$$
Therefore we obtain that  there exists a universal constant $\kappa>0$ such that 
\begin{equation}
\P \left(\Big\| \sum_{\ell \leq n} \frac{a_{i \ell} \lambda_\ell^0}{\sqrt{|N(i)| |N(\ell)|}} \Big\|_{2} - \sqrt{\sum_{\ell\in N(i)}\frac{s^2}{|N(i)|N(\ell)|}}  \ge t \Bigg| G_n\right) \leq 2 \exp \left( - \kappa t^2 d_n \left( \sum_{\ell \in N(i)} \frac{s^2}{|N(i)\|N(\ell)|}\right)^{-1} \right),
\end{equation}




and from this, setting $t =\sqrt{\sum_{\ell \in N(i)} \frac{s^2}{|N(i)\|N(\ell)|}},$ we can deduce that with probability at least $1-2 \exp \left( - \kappa d_n \right),$ 
\begin{align}
\|\lambda_i^1 - \sigma\left( M_{1,0} \lambda_i^0 \right) \|_{2} &\leq M\bigg| \bigg| \sum_{\ell \leq n} \frac{a_{i \ell} \lambda_\ell^0}{\sqrt{|N(i)| |N(\ell)|}} \Big\|_{2}\\& \leq 2(s\wedge 1 ) M \sqrt{\sum_{\ell \in N(i)} \frac{1}{|N(i)| |N(\ell)|}}
\\&\overset{(a)}{\le }\frac{2\sqrt{2}(s\wedge 1 ) M }{\sqrt{(n-1)\delta_W\rho_n}}\le\frac{\xi}{\sqrt{(n-1)\rho_n}}.
\end{align}where to get (a) we used the fact that under $E$ we have that $\inf_{l\ne n}|N(l)|\ge \frac{\rho_n\delta_W(n-1)}{2}$.

We denote $$\mu_i^1 := \sigma(M_{k,0} \lambda_i^0)$$  and remark that the random variables $(\mu_i^1)$ are independent from $(\omega_j)_{j=n-m+1}^n,$ since  $((\lambda_i^0)_i, M_{1,0}, M_{1,1}))$ are assumed to be independent from $(\omega_j)_{j=n-m+1:n}$. We remark in addition that for all $i$ we have $\|\mu_i^1\|_{2}\le M\|\lambda_i^0\|_{2}.$ We know that $\lambda_i^0$ is a $s/\sqrt{d_n}$ sub-Gaussian vector with i.i.d coordinates. Therefore by using \cref{lemma:concentration-gaussianvec} again, we obtain that there exists $\tilde k>0$ such that  with probability of at least $1-2ne^{-2\tilde \kappa d_n}$ we have $$\sup_{i\le n}\|\mu_i^1\|_{2}\le 2M(s\wedge 1).$$
Denote the event $$\tilde E_1:=\Big\{\sup_{i \in [n]} \|\lambda_i^1 - \mu_i^1\|_{2} \le \epsilon(n,1)~\&~\sup_{i\le n}\|\mu_i^1\|_{2}\le 2M(s\wedge 1) \Big\}.$$ Taking a union bound over all vertices, we know that $\tilde E_1$ holds, conditionally on $E$ holding, with a probability of at least $1-2n\rm{exp}(-\kappa d_n)-2n\rm{exp}(-\tilde \kappa d_n)$.
 We now suppose that both $\tilde E_1$ and $\tilde E$ hold. Suppose that for some $1 < k < L$ the following event is true: for all $r \leq k$, there exists some set of vectors $(\mu_i^r)_{i \in [n]}$ independent of the latent features $(\omega_i)_{i=n-m+1}^n$ such that
\begin{equation}
    \sup_{i \in [n]}\| \lambda_i^r - \mu_i^r \|_{2} \leq  \epsilon(n,r),\qquad \sup_{i\le n}~\|\mu_i^r\|_{2}\le 2(s\wedge 1)M^r\left(\frac{1+\sqrt{\frac{A\log(n)}{{(n-1)\rho_n}}}}{\overline{W}}\right)^{r-1} 
\end{equation}  We will show that the same statement holds for $k+1$. In this goal, we denote by $\tilde E_k$ the event $$\tilde E_k:=\left\{\sup_{i \in [n]}\| \lambda_i^r - \mu_i^r \|_{2} \leq\epsilon(n,r),~\&~\sup_{i\le n}~\|\mu_i^r\|_{2}\le 2(s\wedge 1)M^r\left(\frac{1+\sqrt{\frac{A\log(n)}{{(n-1)\rho_n}}}}{\overline{W}}\right)^{r-1}  \qquad \forall r\le k \right\}.$$
 For ease of notation, for each $i$, write $v_i^k = \lambda_i^k - \mu_i^k$. We write $\lambda_i^k = \mu_i^k + v_i^k,$ where the norm of $v_i^k$ is bounded, under the event $\tilde E^k$,  by $\epsilon(n,k)$. Furthermore, we note that 
\begin{align}
\lambda_i^{k+1} &= \sigma \left( M_{k+1, 0} \lambda_i^k + M_{k+1, 1}  \sum_{\ell \leq n} \frac{a_{i\ell} \lambda_\ell^k}{\sqrt{|N(i)\|N(\ell)|}} \right) \\
&= \sigma \Biggl( M_{k+1, 0} \mu_i^k + M_{k+1, 1}  \sum_{\ell \leq n} \frac{a_{i \ell} \mu_\ell^k}{\sqrt{|N(i)\|N(\ell)|}} +
M_{k+1, 0} v_i^k \\&\qquad+ M_{k+1, 1} \sum_{\ell \leq n} \frac{a_{i \ell} v_\ell^k}{\sqrt{|N(i)\|N(\ell)|}}  \Biggl)
\end{align}
Under the event $\tilde E^k$ we have  $$\Big\|M_{k+1, 0} v_i^k + M_{k+1, 1} \sum_{\ell \leq n} \frac{a_{i \ell} v_\ell^k}{\sqrt{|N(i)\|N(\ell)|}}\Big\|_{2}\le 2 M \epsilon(n,k).$$
 As $\sigma(\cdot)$ is Lipschitz, this implies that\begin{align}&
    \sup_{i\le n}\Big\| \lambda_i^{k+1} -\sigma \Biggl( M_{k+1, 0} \mu_i^k + M_{k+1, 1}  \sum_{\ell \leq n} \frac{a_{i \ell} \mu_\ell^k}{\sqrt{|N(i)\|N(\ell)|}} \Big)\Big\|_{2}
    \\&\le  2 M \epsilon(n,k).
\end{align} Moreover we also remark that as $E$ and $\tilde E^k$ holds we have
\begin{align}
  & \sup_{i\le n}\Big\| \lambda_i^{k+1} -\sigma \Big( M_{k+1, 0} \mu_i^k + M_{k+1, 1}  \sum_{\ell \leq n} \frac{a_{i \ell} \mu_\ell^k}{(n-1)\rho_n\sqrt{W(\omega_i,\cdot)W(\omega_{\ell},\cdot)}}\Big) \Big\|_{2}  
  \\&\le M\sup_{i\le n}\|\mu_i^k\|_{2}  \sum_{\ell \leq n}a_{i \ell} \Big|\frac{1}{{\sqrt{|N(i)\|N(\ell)|}}}- \frac{1}{(n-1)\rho_n\sqrt{W(\omega_i,\cdot)W(\omega_{\ell},\cdot)}}\Big|
 \\&\le\Big(\frac{1+\sqrt{\frac{A\log(n)}{{(n-1)\rho_n}}}}{\overline{W}}\Big)^{k-1} \frac{\sqrt{2} (s\wedge 1)M^{k+1}}{\sqrt{\delta_W}\delta_W}\sqrt{1+\frac{\sqrt{A\log(n)}}{\sqrt{\rho_n(n-1)}}}\frac{\sqrt{A\log(n)}}{\sqrt{\rho_n(n-1)}}
 \\&\le\Big(\frac{1+\sqrt{\frac{A\log(n)}{{(n-1)\rho_n}}}}{\overline{W}}\Big)^{L-2} \frac{\sqrt{2} (s\wedge 1)M^{L}}{\sqrt{\delta_W}\delta_W}\sqrt{1+\frac{\sqrt{A\log(n)}}{\sqrt{\rho_n(n-1)}}}\frac{\sqrt{A\log(n)}}{\sqrt{\rho_n(n-1)}} .
\end{align}Note however that we have assumed that $W(x,\cdot)=\overline{W}$ is a constant function. This therefore implies that $\sigma \Big( M_{k+1, 0} \mu_i^k + M_{k+1, 1}  \sum_{\ell \leq n} \frac{a_{i \ell} \mu_\ell^k}{(n-1)\rho_n\overline{W}}\Big) $ is independent from $\omega_{n-m+1,n}$. 
Defining
\begin{equation}
    \mu_i^{k+1} :=\sigma \Big( M_{k+1, 0} \mu_i^k + M_{k+1, 1}  \sum_{\ell \leq n} \frac{a_{i \ell} \mu_\ell^k}{(n-1)\rho_n\overline{W}}\Big) ,
\end{equation}
we have that 
\begin{equation}
\sup_{i \in [n]} \|\lambda_i^{k+1} - \mu_i^{k+1} \|_{2} \leq \epsilon(n,k+1).
\end{equation}Moreover we note that \begin{align}
    \|\mu_i^{k+1}\|_{2}&\le 2M\sup_{i\le n}
\|\mu_i^k\|_{2}\Big(1+ \sum_{\ell \leq n} \frac{a_{i \ell} \mu_\ell^k}{(n-1)\rho_n\overline{W}}\Big)
\\&\le  2M\sup_{i\le n}
\|\mu_i^k\|_{2}\Big(1+ \frac{|N(i)|}{(n-1)\rho_n\overline{W}}\Big)\\&\overset{(a)}{\le} 2M\sup_{i\le n}
\|\mu_i^k\|_{2}\frac{1+\sqrt{A\log(n)/({(n-1)\rho_n)}}}{2\overline{W}}\end{align}where to get (a) we used the fact that we assumed that $\tilde E$ holds. Hence we obtain that $$\sup_{i\le n} \|\mu_i^{k+1}\|_{2}\le2M(s\wedge 1)\Big(\frac{2M+2M\sqrt{A\log(n)/({(n-1)\rho_n)}}}{2\overline{W}}\Big)^{k} .$$
Hence if $\tilde E^1$ and $\tilde E$ hold this implies that $\tilde E^{k+1}$ and $\tilde E$ hold which completes the induction. We hence have that $$P\big(\sup_{i\le n}\sup_{i \in [n]}\| \lambda_i^r - \mu_i^r \|_{2}\le \epsilon(n,r),~\forall r\le L\big)\ge 1-2n\big(e^{-\frac{A\log(n)}{3}}+e^{-\frac{2A\log(n)}{\rho_n}}+e^{-\kappa d_n}+e^{-\tilde \kappa d_n}\big)$$Choosing $A=6$ yields the desired result.
\end{proof}
We then prove \cref{negative-lipschitz}
\begin{proof}
    Suppose that $f:\mathbb{R}^2\rightarrow\mathbb{R}$ is Lipchtiz with respect to the Euclidean distance in $\mathbb{R}^2$. Using \cref{prop:negative-result}, we know that there exists $\kappa>0$ and embeddings $(\mu_j^L)$ that are independent from $\omega_{n-m+1:n}$ such that with a probability of at least $1-\frac{2}{n}-\frac{2}{n^{11}}-2ne^{-\kappa d_n}$ we have,
\begin{equation}
\sup_{{\ell \leq L}}\| \lambda_{n}^L - \mu_{n}^\ell \|_2 \leq \frac{K}{\sqrt{n}},
\end{equation} where $K>0$ is an absolute constant. As $f$ is assumed to be a Lipchitz function we obtain that $$\big|f(\lambda_n^L,\lambda_i^L)-f(\mu_n^L,\mu_i^L)\big|\le \frac{2K}{\sqrt{n}}.$$ Now denote the event $E_n:=\{f(\lambda_i^L,\lambda_{n}^L)\ge 2\}$ and define $\tilde E_n:=\{\big|f(\lambda_n^L,\lambda_i^L)-f(\mu_n^L,\mu_i^L)\big|\le \frac{2K}{\sqrt{n}}\Big\}$. 
We will obtain two different bounds respectively when\begin{itemize} \item$P(E_n)\ge\frac{1}{3}\mathbb{E}\Big(\big[W(\omega_i,\omega_n)-W(\omega_i,\cdot)\big]^2\Big)$ \item $P(E_n)<\frac{1}{3}\mathbb{E}\Big(\big[W(\omega_i,\omega_n)-W(\omega_i,\cdot)\big]^2\Big)$.\end{itemize} Firstly, if $P(E_n)\ge\frac{1}{3}\mathbb{E}\Big(\big[W(\omega_i,\omega_n)-W(\omega_i,\cdot)\big]^2\Big)$ we remark that 
\begin{align}&\mathbb{E}\Big(\big[W(\omega_i,\omega_n)-f(\lambda_i^L,\lambda_{n}^L)\big]^2\Big)\ge \mathbb{E}\Big(\big[W(\omega_i,\omega_n)-f(\lambda_i^L,\lambda_{n}^L)\big]^2\mathbb{I}(E_n)\Big)\\&\ge P(E_n)>\frac{1}{3}\mathbb{E}\Big(\big[W(\omega_i,\omega_n)-W(\omega_i,\cdot)\big]^2\Big).\end{align} Now assume instead $P(E_n)<\frac{1}{3}\mathbb{E}\Big(\big[W(\omega_i,\omega_n)-W(\omega_i,\cdot)\big]^2\Big)$. This implies that 
\begin{align}&
     \mathbb{E}\Big(\big[W(\omega_i,\omega_n)-f(\lambda_i^L,\lambda_{n}^L)\big]^2\Big)
   - \mathbb{E}\Big(\big[W(\omega_i,\omega_n)-f(\mu_i^L,\mu_{n}^L)\big]^2\mathbb{I}(E^c_n\cap \tilde E^c_n)\Big)
   \\&=  \mathbb{E}\Big(\big[W(\omega_i,\omega_n)-f(\lambda_i^L,\lambda_{n}^L)\big]^2\mathbb{I}(E_n)\Big)+\mathbb{E}\Big(\big[W(\omega_i,\omega_n)-f(\lambda_i^L,\lambda_{n}^L)\big]^2\mathbb{I}(E^c_n\cap \tilde E_n)\Big)\\&\quad+\mathbb{E}\Big(\big[W(\omega_i,\omega_n)-f(\lambda_i^L,\lambda_{n}^L)\big]^2\mathbb{I}(E^c_n\cap \tilde E^c_n)\Big)
   - \mathbb{E}\Big(\big[W(\omega_i,\omega_n)-f(\mu_i^L,\mu_{n}^L)\big]^2\mathbb{I}(E^c_n\cap \tilde E^c_n)\Big)
   \\&\ge -\Big|\mathbb{E}\Big(\big[W(\omega_i,\omega_n)-f(\lambda_i^L,\lambda_{n}^L)\big]\Big[f(\mu_i^L,\mu_{n}^L)-f(\lambda_i^L,\lambda_{n}^L)\Big]\mathbb{I}(E^c_n\cap \tilde E^c_n)\Big)\Big|\\&\qquad -\Big|\mathbb{E}\Big(\big[W(\omega_i,\omega_n)
   -f(\mu_i^L,\mu_{n}^L)\big]\Big[f(\mu_i^L,\mu_{n}^L)-f(\lambda_i^L,\lambda_{n}^L)\Big](E^c_n\cap \tilde E^c_n)\Big)\Big|
     \\&\overset{(a)}{\ge} -\frac{2K}{\sqrt{n}}(2+2+\frac{2K}{\sqrt{n}})P(E_n^C\cap\tilde E_n^c)
\end{align}where to get (a) we used the fact that under $E^c_n\cap \tilde E^c_n$ we have $$|W(\omega_i,\omega_n)-f(\mu_i^L,\mu_{n}^L)|\le 2+\frac{2K}{\sqrt{n}}$$ and $$|W(\omega_i,\omega_n)-f(\lambda^L_i,\lambda_{n}^L)|\le 2.$$Hence we obtain that $$\mathbb{E}\Big(\big[W(\omega_i,\omega_n)-f(\lambda_i^L,\lambda_{n}^L)\big]^2\Big)
   \ge  \mathbb{E}\Big(\big[W(\omega_i,\omega_n)-f(\mu_i^L,\mu_{n}^L)\big]^2\mathbb{I}(E^c_n\cap \tilde E^c_n)\Big)+o_n(1).$$However as $(\mu_j^L)$ are independent from $\omega_n$ we have that $f(\mu_i^L,\mu_n^L)$ is independent from $\omega_n$. Hence if we write $W(x,\cdot)=\int_0^1 W(x,y)dy$ we obtain that \begin{align}&
       \mathbb{E}\Big(\big[W(\omega_i,\omega_n)-f(\mu_i^L,\mu_{n}^L)\big]^2\mathbb{I}(E^c_n\cap \tilde E^c_n)\Big)\\&= \mathbb{E}\Big(\big[W(\omega_i,\omega_n)-W(\omega_i,\cdot)\big]^2\mathbb{I}(E^c_n\cap \tilde E^c_n)\Big)+\mathbb{E}\Big(\big[f(\mu_i^L,\mu_{n}^L)-W(\omega_i,\cdot)\big]^2\mathbb{I}(E^c_n\cap \tilde E^c_n)\Big)
       \\&\ge\mathbb{E}\Big(\big[W(\omega_i,\omega_n)-W(\omega_i,\cdot)\big]^2\mathbb{I}(E^c_n\cap \tilde E^c_n)\Big)
       \\&\ge-P(E_n)-P(\tilde E_n)+ \mathbb{E}\Big(\big[W(\omega_i,\omega_n)-W(\omega_i,\cdot)\big]^2\Big).
   \end{align}Now we have assumed that $P(E)\rightarrow 0$ and we know that $P(\tilde E)\rightarrow0$. Hence we obtain that $$\mathbb{E}\Big(\big[W(\omega_i,\omega_n)-f(\lambda_i^L,\lambda_{n}^L)\big]^2\Big)
   \ge  \frac{2}{3}\mathbb{E}\Big(\big[W(\omega_i,\omega_n)-W(\omega_i,\cdot)\big]^2\Big)+o_n(1).$$Now we have assumed that $W(\cdot,\cdot)$ is not the constant graphon but \ref{asp3}
assumes that $x\rightarrow W(x,\cdot)$ is a constant function. Hence  by choosing $$K
:=\frac{1}{3}\mathbb{E}\Big(\big[W(\omega_i,\omega_n)-W(\omega_i,\cdot)\big]^2\Big)>0$$we obtain that$$\mathbb{E}\Big(\big[W(\omega_i,\omega_n)-f(\lambda_i^L,\lambda_{n}^L)\big]^2\Big)
   \ge  K+o_n(1).$$\end{proof}

\section{Proof of \cref{prop:graph_concentration_sparse}}

We proceed in two main steps. The first step is to establish a high-probability bound for $\big|\hat{W}_{n, i, j}^{(k)} - W_{n, i, j}^{(k)}\big|.$ This bound is then used to establish a bound on $|\hat{q}_{i, j}^{(k)} - W_{n, i, j}^{(k)}|$. The main goal is to prove \cref{prop:estimators-to-empiricalmoments}, which is a restatement of \cref{prop:graph_concentration_sparse}.

We will do these steps separately in the below subsections.

\subsection{Proof of \cref{prop:graph_concentration_sparse}, Part 1}

The goal of this subsection is to prove the following lemma.

\begin{lemma}
\label{lemma:concentration-of-graph}
With probability at least $1-3/n$, we have that for all $2 \leq k \leq L+2$,
\begin{align*}
    \max_{i \neq j} \big|\hat{W}_{n, i, j}^{(k)} - {W}_{n, i, j}^{(k)} \big| \leq 3a_k \frac{\rho_n^{k-1/2}}{\sqrt{n-1}} \log(n)^k,
\end{align*}    
where $a_k = C \sqrt{2} (8(k+2))^k k^{k+1}\sqrt{k!}/\sqrt{B},$ where $B, C$ are some absolute positive constants.
\end{lemma}

We proceed in three steps. We first establish a high probability bound for $\big|\hat{W}_{n, i, j}^{(k)} - \E[\hat{W}_{n, i, j}^{(k)} | (\omega_\ell)]\big|.$ Then, we establish a high probability bound for $\big|\E[\hat{W}_{n, i, j}^{(k)} | (\omega_\ell)] - \E[\hat{W}_{n, i, j}^{(k)} | \omega_i, \omega_j]\big|.$ We then bound $\big|\E[\hat{W}_{n, i, j}^{(k)} | \omega_i, \omega_j] -W_{n,i, j}^{(k)}\big|.$

\subsubsection{Bounding $\big|\hat{W}_{n, i, j}^{(k)} - \E[\hat{W}_{n, i, j}^{(k)} | (\omega_\ell)]\big|$}
We use the following
\begin{lemma}[\cite{kimconcentration}]\label{luna2}
    Let $(\xi_i)$ be a sequence of independent Bernouilli random variables. Let $N$ be an integer and $P:\mathbb{R}^N\rightarrow\mathbb{R}$ be a polynomial of degree $k$. For a subset $A\subset [|N|]^k$ we write by $\partial_AP$ the partial derivative of $P$ with respect to the indexes $A$. Define $\mu_1=\max_{|A|\ge 1}\mathbb{E}[\partial_AP((\xi_i)_{i\le N})]$ and let $\mu_0=\max_{|A|\ge 0}\mathbb{E}[\partial_AP((\xi_i)_{i\le N})]$. Then, 
    \begin{equation}
        \P\left( |P((\xi_i)_{i\le N}) - \E[P((\xi_i)_{i\le N})| (\omega_l)]| > a_k \sqrt{ \mu_0 \mu_1} \lambda^k \right) \leq  G \cdot  \rm{exp}(-\lambda + (k-1) \log(N)),
    \end{equation}
    where $a_k = 8^k \sqrt{k!},$ and $G$ is an absolute constant.
\end{lemma}

We will apply \cref{luna2} to obtain the desired result. We first fix $i \neq j$. In this goal we set $N:=n(n-1)/2$ and define $P$ to be the following polynomial:
$$P((a_{k,l})_{k\ne l\le n}):= \hat{W}_{n, i, j}^{(k)} =  \frac{1}{(n-1)^{k-1}}\sum_{r_1,\dots,r_k}a_{i,r_i}\dots a_{r_{k-1},j}.$$
We remark that conditionally on the features $(\omega_l)$, the random variables $(a_{k,l})$ are independent Bernouili random variables. We note that our goal is to give a high probability bound on the difference between $P((a_{k,l})_{k\ne l\le n})$ and its expectation $\mathbb{E}[P((a_{k,l})_{k\ne l\le n})|(\omega_l)]$. We first bound $\mathbb{E}[\partial_AP((a_{k,l})_{k\ne l\le n})]$. We note that this is maximized when $A$ contains only one element. This is because when differentiating by $a_{i,j}$, all of the terms that do not include this edge vanish, hence differentiating by more $a_{i,j}$ will cause more edges to vanish. 

Furthermore, $\mathbb{E}[(\partial/\partial a_{s,t}) P((a_{k,l})_{k\ne l\le n})]$ is maximized by choosing $(s,t)$ to be an edge that appears most often, such that as many terms as possible are preserved. Because the endpoints are fixed as $i,j$, it suffices to bound the desired quantity for $(s,t) = (i,1)$ (without loss of generality, assume $i \neq 1$; note the choice of $1$ was arbitrary). For each string $a_{i, r_1} a_{r_1, r_2} \dots a_{r_{k-1}, j}$, if it contains $a_{i,1}$, then the number of terms in the string will be lowered by 1 upon differentiation (otherwise it equals 0 identically), hence after differentiating, the maximum number of terms in the string is $k-1.$ 

We now upper-bound the number of strings $a_{i, r_1} a_{r_1, r_2} \dots a_{r_{k-1}, j}$ that contain $a_{i1}$ and also have exactly $t$ distinct edges. 
\begin{enumerate}
    \item \textbf{Case 1:} $r_1= 1$. Then, there are $k-2$ free indices remaining. However, since there are $t$ distinct edges, that means $k-t$ edges are repeated (appear at least more than once). Note that each repeated edge removes one free index. Hence, the remaining number of degrees of freedom is $t-2 \vee 0.$ 
    \item \textbf{Case 2:} $r_1 \neq 1.$ Then, since the edge $(i,1)$ needs to appear in the sequence, there are at most $k$ locations for it to appear, and then $2$ ways to orient it (it can either be $(i,1)$ or $(1,i)$). So, there are $2k$ ways to choose the edge $(i,1)$, and then there remain $t-3 \vee 0$ ways degrees of freedom remaining.
\end{enumerate}

Combining the two cases, there are at most $(n-1)^{t-2\vee 0} + 2k(n-1)^{t-3 \vee 0} \leq 2 (n-1)^{t-2\vee 0}$ ways to choose the set of indices $\{r_1, r_2, \dots, r_{k-1}\}.$ Then, there are at most $(k-1)^{k-1}$ ways to choose the values of $r_1, r_2, \dots, r_{k-1}$ among this set, which is upper bounded by $k^k$. Hence, the number of configurations with exactly $t$ distinct edges is upper bounded by $2k^k (n-1)^{t-2 \vee 0}$. Hence, we can bound
\begin{align*}
    \mathbb{E}[(\partial/\partial a_{s,t}) P((a_{k,l})_{k\ne l\le n})|(\omega_l)] &\leq \frac{2 k^k}{(n-1)^{k-1}} \sum_{t=1}^{k} (n-1)^{t-2 \vee 0}\rho_n^{t-1} \\
    &\leq 2k^k \sum_{t=1}^k \frac{\rho_n^{t-1}}{(n-1)^{k-t+1}} \\
    &\leq 2 k^{k+1} \frac{\rho_n^{k-1}}{n-1},
\end{align*}
where the last inequality follows if $\rho_n > \frac{1}{n-1}.$ We now bound $\mathbb{E}[P((a_{k,l})_{k\ne l\le n})|(\omega_l)].$ We first upper-bound the number of paths from $i$ to $j$ of length $k$ with exactly $\ell$ distinct edges. For convenience, denote $r_0 = i$ and $r_{k} = j$. Firstly, we note that if there are exactly $\ell$ distinct edges, then $| \{ r_0,  r_1, r_2, \dots, r_{k-1}, r_k \} | \leq \ell+1.$ Since $r_0 = i, r_k = j,$ there are at most $\binom{n-1}{\ell-1}$ ways to choose a superset in which $\{ r_1, r_2, \dots, r_{k-1}\}$ lies. Then, there are at most $(\ell-1)^{k-1} \leq k^{k}$ ways to choose the indices $r_1, r_2, \dots, r_{k-1}$ among this set. Hence, there are most $(n-1)^{\ell-1} k^k$ paths of length $k$ with exactly $\ell$ distinct edges from $i$ to $j$. Hence, 
\begin{align*}
    \mathbb{E}[P((a_{k,l})_{k\ne l\le n})|(\omega_l)] &\leq \frac{1}{(n-1)^{k-1}}\sum_{\ell=1}^k (n-1)^{\ell-1} k^k \cdot \rho_n^\ell \\
    &\leq k^k \sum_{\ell=1}^k \frac{\rho_n^\ell}{(n-1)^{k-\ell}} \\
    &\leq k^{k+1} \rho_n^k.
\end{align*}
Now we apply \cref{luna2} to obtain 
\begin{align*}
    \P\left( |P((\xi_i)_{i\le N}) - \E[P((\xi_i)_{i\le N})| (\omega_l)]| > b_k \frac{\rho_n^{k-1/2}}{\sqrt{n-1}} \lambda^k \right)  &= G \cdot  \rm{exp}(-\lambda + (k-1) \log(N)) \\
    &\leq  G \cdot \rm{exp}(-\lambda + (k-1) \log(n) )
\end{align*}
for some absolute constant $G$, where $b_k = \sqrt{2} 8^k k^{k+1}\sqrt{k!}.$ Choosing $\lambda = \log(G) +  (k+2) \log(n)$, and union bounding over all $i \neq j$ and $2 \leq k \leq L+2$, we have that with probability at least $1-1/n$, for all $2 \leq k\leq L+2,$
\begin{align*}
    \max_{i \neq j} \big|\hat{W}_{n, i, j}^{(k)} - \E[\hat{W}_{n, i, j}^{(k)}| (\omega_l)] \big| \leq a_k \frac{\rho_n^{k-1/2}}{\sqrt{n-1}} \log(n)^k.
\end{align*}
where $a_k = C \sqrt{2} (8(k+2))^k k^{k+1}\sqrt{k!}/\sqrt{B},$ where $B$ is from the constant in the Big O factor, and $C$ is some constant.

\subsubsection{Step 2: bounding $\big|\E[\hat{W}_{n, i, j}^{(k)} | (\omega_\ell)] - \E[\hat{W}_{n, i, j}^{(k)} | \omega_i, \omega_j]\big|$}

We now bound $|\E[\hat{W}_{n, i,j}^{(k)} | (\omega_\ell)] - \E[\hat{W}_{n, i,j}^{(k)} | \omega_i, \omega_j]|$ using McDiarmid's Inequality. For ease of notation, we assume WLOG that $i=1$ and $j=2$, and denote $r_0 = i$ and $r_k = j$. To use McDiarmid's inequality; we first bound the maximum deviation in altering one of the coordinates. WLOG we alter the $n$th coordinate $\omega_n$ and bound $$\left|\E[\hat{W}_{n, 1,2}^{(k)} | (\omega_\ell)_{\ell \neq n}, \omega_n] - \E[\hat{W}_{n, 1,2}^{(k)} | (\omega_\ell)_{\ell \neq n}, \omega_n'] \right|.$$ Recalling the definition $$\hat{W}_{n, 1,2}^{(k)} = \frac{1}{(n-1)^{k-1}} \sum_{r_1, r_2, \dots, r_{k-1}} a_{1, r_1} a_{r_1, r_2} \dots a_{r_{k-1}, 2},$$ denote $B_{(r_s)} = \E[a_{1, r_1} a_{r_1, r_2} \dots a_{r_{k-1}, 2} | (\omega_\ell)_{\ell \neq n}, \omega_n] - \E[a_{1, r_1} a_{r_1, r_2} \dots a_{r_{k-1}, 2} | (\omega_\ell)_{\ell \neq n}, \omega_n'].$ We first bound each $B_{(r_s)}$ individually over different choices of the indices $(r_s)$. We note that if none of the $r_s = n$, then $B_{(r_s)} = 0$. Hence, we need consider only the terms in the summation for which at least one of the $r_s$ equals $n$. 

If $(r_s)$ corresponds to a path with exactly $k-t$ distinct edges, then $|B_{(r_s)}| \leq \rho_n^{k-t}.$ We upper bound the number of paths of length $k$ that have exactly $k-t$ distinct edges. Note that $t \leq k-2$, since our we are considering terms such that there exist $r_s$ that equal $1, 2, n$, so there cannot be less than  two distinct edges. We note that if there are exactly $k-t$ distinct edges, then the number of distinct numbers among the set $\{ r_0, r_1, \dots, r_k \}$ is at most $k+1-t$. Because $r_0 = 1$ and $r_{k}=2$, and $n$ must be one of the $r_s$ must equal $n$, there are at most $\binom{n-1}{k-2-t}$ ways to choose the remaining vertices. Then, the number of ways to choose the values of $r_s$ among these $k+1-t$ options is bounded by $(k+1-t)^{k-1}.$ Hence, the total number of options is upper bounded by $\binom{n-1}{k-t-2}(k+1-t)^{k-1} \leq (n-1)^{k-t-2}(k+1)^{k-1}.$ Lastly, we note that $t \in \{ 0, 1, \dots, k-1\}.$ Hence, the constant in the exponential bound of McDiarmid's Inequality is given by 

\begin{align}
    \frac{1}{(n-1)^{k-1}} \sum_{t=0}^{k-2} \rho_n^{2(k-t)} (n-1)^{k-t-2} (k+1)^{k-1} &= \frac{\rho_n^{2k} (k+1)^{k-1}}{n-1} \sum_{t=0}^{k-2} \left( \frac{1}{n \rho_n^2} \right)^{t} \nonumber\\
    &= \frac{\rho_n^{2k} (k+1)^{k-1}}{n-1} \frac{1 - (1/n \rho_n^2)^{k-1}}{1 - (1/n \rho_n^2)}\nonumber\\
    &\leq \frac{\rho_n^{2k}(k+1)^{k-1}}{n-1} \frac{1 - (1/n \rho_n^2)^{k-1}}{1 - (1/n \rho_n^2)}\nonumber\\
    &\leq 4 \frac{\rho_n^{2k}k^k}{n}
\end{align}
if $n \rho_n^2 \ge \frac{1}{10}$, since then $\frac{1 - (1/n \rho_n^2)^{k-1}}{1 - (1/n \rho_n^2)} \leq \frac{10}{9},$ and $(k+1)^{k-1} \leq 2 k^k$ for all $k \ge 2.$ Then, the McDiarmid Inequality states that 
\begin{equation}
   \P\left(  |\E[\hat{W}_{n, i,j}^{(k)} | (\omega_\ell)] - \E[\hat{W}_{n, i,j}^{(k)} | \omega_i, \omega_j]| \ge t \right) \leq 2 \exp\left( -2 t^2 \frac{n}{\rho_n^{2k} k^k} \right)
\end{equation}
Hence, choosing $t = \frac{\sqrt{k^k \rho_n^{2k}}}{\sqrt{n}}\sqrt{2 \log(n)}$ and union bounding over $i \neq j$, $2 \leq k \leq L+2$, we have that with probability at least $1 - 2/n,$ for all $2 \leq k\leq L+2$,
\begin{equation}
\label{bound1}
   \max_{i \neq j} |\E[\hat{W}_{n, i,j}^{(k)} | (\omega_\ell)] - \E[\hat{W}_{n, i,j}^{(k)} | \omega_i, \omega_j]| \ge \sqrt{\frac{k^k \rho_n^{2k}}{n}}\sqrt{2 \log(n)} 
\end{equation}

\subsubsection{Step 3: bounding $\big|\E[\hat{W}_{n, i, j}^{(k)} | \omega_i, \omega_j] -W_{n,i, j}^{(k)}\big|$}

Recall $$\hat{W}_{n,i,j}^{(k)} = \frac{1}{(n-1)^{k-1}} \sum_{r_1, \dots, r_{k-1}} a_{i r_1} a_{r_1 r_2} \dots a_{r_{k-1} j}$$

We see that 
$$
\E[\hat{W}_{n,i,j}^{(k)} | \omega_i,\omega_j] = \frac{1}{(n-1)^{k}} \sum_{\ell=1}^{k} W_{n, i, j}^{(\ell)} \cdot  (\text{$\#$ paths with $\ell$ distinct edges}) 
$$

Firstly, we claim that the number of paths of length $k$ starting from vertex $i$ to $j$ that have no repeated edges is lower bounded by $(n-2)(n-3) \dots (n-k)$. This is simply because if no vertex is passed through twice along the path, then there cannot exist repeated edges. There are $n-2$ choices for $r_1$, then $n-3$ choices for $r_2,$ etc., which shows this assertion. This implies that the number of paths with $k$ distinct edges is $(n-1)^{k-1} + P_{k}$, where $|P_k| = O(k^2n^{k-2}).$ Note that this also implies that the number of paths of length $k$ is of order $O(k^2 n^{k-2}).$ Hence, we can write 
\begin{gather*}
\E[\hat{W}_{n,i,j}^{(k)} | \omega_i,\omega_j] = W_{n, i, j}^{(k)} + \frac{1}{(n-1)^{k-1}} P_k \cdot W_{n,i,j}^{(k)} + \frac{1}{(n-1)^{k-1}} \sum_{\ell=2}^{k-1} W_{n, i, j}^{(\ell)}  (\text{$\#$ paths with $\ell$ distinct edges}) \\
\Rightarrow |\E[\hat{W}_{n,i,j}^{(k)} | \omega_i,\omega_j] - W_{n, i, j}^{(k)}| \leq \frac{1}{(n-1)^{k-1}} |P_k| \cdot W_{n,i,j}^{(k)} + \left| \frac{1}{(n-1)^{k-1}} \sum_{\ell=2}^{k-1} W_{n, i, j}^{(\ell)}  (\text{$\#$ paths with $\ell$ distinct edges}) \right|
\end{gather*}

To proceed with the triangle inequality, we first upper-bound the number of paths from $i$ to $j$ of length $k$ with exactly $\ell$ distinct edges. For convenience, denote $r_0 = i$ and $r_{k} = j$. Firstly, we note that if there are exactly $\ell$ distinct edges, then $| \{ r_0,  r_1, r_2, \dots, r_{k-1}, r_k \} | \leq \ell+1.$ Since $r_0 = i, r_k = j,$ there are at most $\binom{n-1}{\ell-1}$ ways to choose a superset in which $\{ r_1, r_2, \dots, r_{k-1}\}$ lies. Then, there are at most $(\ell-1)^{k-1} \leq k^{k}$ ways to choose the indices $r_1, r_2, \dots, r_{k-1}$ among this set. Hence, there are most $(n-1)^{\ell-1} k^k$ paths of length $k$ with exactly $\ell$ distinct edges from $i$ to $j$. Hence, 
\begin{align*}
    |\E[\hat{W}_{n,i,j}^{(k)} | \omega_i,\omega_j] - W_{n, i, j}^{(k)}| &\leq O\left( \frac{k^2}{n} \right) \rho_n^k + \sum_{\ell=2}^{k-1} \frac{1}{(n-1)^{k-\ell}} \rho_n^{\ell} k^k \\
    &= k^k O\left( \frac{\rho_n^{k-1}}{n} + \frac{\rho_n^{k-2}}{n^2} + \dots + \frac{\rho_n}{n^{k-1}} \right) \\
    &= O \left( k^{k+1} \frac{\rho_n^{k-1}}{n} \right),
\end{align*}
where this last line is true because $\rho_n > \frac{1}{n}.$ 

\subsubsection{Step 4: Combining the Bounds}

Combining the three steps and using the triangle inequality, we have that with probability at least $1-3/n$, we have that for all $2 \leq k \leq L+2$,

\begin{align*}
    \max_{i \neq j} \big|\hat{W}_{n, i, j}^{(k)} - {W}_{n, i, j}^{(k)} \big| &\leq  \sqrt{\frac{k^k \rho_n^{2k}}{n}}\sqrt{2 \log(n)}  + a_k \frac{\rho_n^{k-1/2}}{\sqrt{n-1}} \log(n)^k + O \left( k^{k+1} \frac{\rho_n^{k-1}}{n} \right) \\
    &\leq 3a_k \frac{\rho_n^{k-1/2}}{\sqrt{n-1}} \log(n)^k,
\end{align*}
for sufficiently large $n$, as we note that the second term is the dominating one when $\rho_n > 1/n$. This suffices for the proof of \cref{lemma:concentration-of-graph}.

\subsection{Proof of \cref{prop:graph_concentration_sparse}, Part 2}
\label{appendix:lg-gnn-construction}

The main goal of this subsection is to prove \cref{prop:estimators-to-empiricalmoments}. Before proving that, we first present general properties of the GNN embedding vectors our proposed algorithm produces (where we consider a more general version of our proposed GNN in which the weight matrices are not the identity). Uninterested readers can skip directly to \cref{prop:estimators-to-empiricalmoments} to see the main result, and those who interested in more details can continue to read the exposition below.

In this appendix, we consider a version of our proposed GNN architecture with general weight matrices, given by
\begin{equation}
\label{gnn-iteration-general-weights}
\lambda_i^k = M_{k, 0} \lambda_i^{k-1} + M_{k, 1} \frac{1}{n-1} \sum_{\ell \leq n} a_{i\ell} \lambda_\ell^{k-1},\end{equation}
where $M_{k,0}, M_{k,1}$ are matrices that can be freely chosen. Note also that $a_{ii} = 0$, and hence the normalization by $n-1$. As proposed in Algorithm \ref{algo1}, we initialize the embeddings by first sampling $(Z_i) \stackrel{iid}{\sim} \frac{1}{\sqrt{d_n}} \mathcal{N}(0, I_{d_n})$, and then computing the first layer through  
\begin{equation}
\lambda_i^0 = \frac{1}{\sqrt{n-1}} \sum_{\ell=1}^n a_{i \ell} Z_\ell.
\end{equation}
We compute a total of $L$ GNN iterations and for all vertices $i$, produce the sequence $\lambda_i^0, \lambda_i^1, \dots, \lambda_i^{L}.$

In this appendix, we prove \cref{prop:estimators-to-empiricalmoments} in a series of steps:
\begin{enumerate}
    \item We first give a general formula for $\lambda_i^k$, and then demonstrate that $\E[\langle \lambda_i^{k_1}, \lambda_j^{k_2} \rangle]$ is a linear combination of the empirical moments of the graphon $\hat{W}_{n, i,j}^{(k)}$. This is done in \cref{lemma:expectation-dotproducts}.
    \item We then show in \cref{lemma:form-of-q} that $\hat{q}_{i,j}^{(k)}$ can be written in the simpler form $$\hat{q}_{i,j}^{(k)} = \left \langle \frac{1}{\sqrt{n-1}} \sum_{\ell \leq n} a_{j, \ell} Z_{\ell}, \frac{1}{\sqrt{n-1}} \sum_{\ell \leq n} \hat{W}_{n, i, \ell}^{(k-1)} Z_\ell  \right \rangle.$$ 
    \item We then use the above observation to establish a concentration result for $\hat{q}_{i,j}^{(k)}$ in \cref{prop:estimators-to-empiricalmoments}.
\end{enumerate}

\subsection{Formula for the Embedding Vectors}

Recall the definition from \cref{defn:N-sumproduct-matrices}
\begin{equation}
N_s^k := \sum_{\stackrel{r_1, \dots, r_k \in \{0,1\}}{\sum_{i=1}^k r_i = s} } M_{k, r_1} M_{k-1, r_2} \dots M_{1, r_k}.
\end{equation}
For example, $N_0^3 = M_{3,0} M_{2,0} M_{1,0}$ and $N_1^3 = M_{3,0} M_{2,1} M_{1,0} + M_{3,0} M_{2,0} M_{1,1} + M_{3,1} M_{2,0}M_{1,0}.$Then,  

\begin{proposition}
\label{prop:form-of-embedding}
Consider the GNN Architecture defined in \cref{algo1}, and recall the definition of the empirical moment between vertices $i$ and $j$, $$\hat{W}_{n, i,j}^{(k)} = \frac{1}{(n-1)^{k-1}} \sum_{r_1, \dots, r_{k-1} \leq n} a_{i r_1} a_{r_1 r_2} \dots a_{r_{k-1} j}$$ as in \cref{defn:empirical-moment}. Then for $k \ge 0,$ we have
\begin{equation}
\lambda_i^k = \frac{1}{\sqrt{n-1}} \sum_{\ell \leq n}  \left( \sum_{q=0}^k  N_q^k \cdot \hat{W}_{n, i, \ell}^{q+1}  \right)Z_\ell.
\end{equation}
\end{proposition}


\begin{proof}[Proof of \cref{prop:form-of-embedding}]
We proceed through induction. The induction base case of $k=0$ is satisfied by definition of $\lambda_i^0$. Now, suppose for induction that, for $k=K,$ $$\lambda_i^K =  \frac{1}{\sqrt{n-1}} \sum_{\ell \leq n} \left( \sum_{q=0}^K N_q^K \cdot \hat{W}_{n, i, \ell}^{q+1}  \right)Z_\ell.$$ We use the definition of our GNN iteration to compute $\lambda_i^{K+1}.$ In particular, we observe that $\lambda_i^{K+1}$ will be a linear combination of the $Z_\ell,$ where the coefficient of $Z_\ell$ is given by
\begin{align}
&M_{K+1,0}\frac{1}{\sqrt{n-1}} \left(  \sum_{q=0}^K N_q^K \cdot \hat{W}_{n, i, \ell}^{q+1} \right) + M_{K+1, 1}\frac{1}{n-1 } \sum_{r \leq n} a_{ir}\left( \frac{1}{\sqrt{n-1}} \sum_{q=0}^K N_q^K\cdot  \hat{W}_{n, r, \ell}^{q+1} \right) \nonumber \\
=& \frac{1}{ \sqrt{n-1}} M_{K+1,0} N_0^K\cdot \hat{W}_{n, i, \ell}^1 \nonumber \\
+&\frac{1}{\sqrt{n-1}} \left( \sum_{q=1}^K \left( M_{K+1,0}  N_q^K\cdot \hat{W}_{n, i, \ell}^{q+1} + \frac{1}{n-1} M_{K+1,1} N_{q-1}^K\cdot \hat{W}_{n, r, \ell}^{q} \sum_{r \leq n} a_{ir} \right) \right) \nonumber\\
+& \frac{1}{\sqrt{n-1}} \left( \frac{1}{n-1 } M_{K+1, 1} N_q^K\cdot \hat{W}_{n, r, \ell}^{K+1} \sum_{r \leq n} a_{ir} \right). 
\end{align}

\noindent To arrive at the desired result, we first make a few observations. Firstly, we note that 
\begin{align}
    \frac{1}{n-1} \hat{W}_{n, r, \ell}^q \sum_{r \leq n} a_{ir} &= \frac{1}{(n-1)^{q}} \left( \sum_{r_1, \dots, r_{q-1} \leq n}a_{r r_1} a_{r_1 r_2} \dots a_{r_{q-1} \ell} \right)\sum_{r \leq n} a_{i r} \nonumber\\
    &= \frac{1}{(n-1)^{q}} \sum_{r_1, r_2, \dots, r_{q}} a_{i r_1} a_{r_1 r_2} \dots a_{r_q \ell}\nonumber \\
    &= \hat{W}_{n, i,\ell}^{q+1},
\end{align} which allows us to simplify the analogous quantities in the last two terms. To simplify the second term, we use the definition of $N_q^K$ and note that $$M_{K+1, 0} N_{q}^K + M_{K+1, 1} N_{q-1}^K = N_{q}^{K+1}.$$ To see this, we note that 
\begin{align*}
    M_{K+1, 0} N_{q}^K + M_{K+1, 1} N_{q-1}^K  &= M_{K+1, 0} \sum_{\stackrel{r_1, \dots, r_K \in \{0,1\}}{\sum_{i=1}^K r_i = q} } M_{K, r_1} M_{K-1, r_2} \dots M_{1, r_K} \\
    &+ M_{K+1, 1} \sum_{\stackrel{r_1, \dots, r_K \in \{0,1\}}{\sum_{i=1}^K r_i = q-1} } M_{K, r_1} M_{K-1, r_2} \dots M_{1, r_K} \\
    &= \sum_{\stackrel{r_1, \dots, r_K, r_{K+1} \in \{0,1\}}{\sum_{i=1}^K r_i = q} } M_{K+1, r_1} M_{K, r_2} M_{K-1, r_2} \dots M_{1, r_{K+1}} \\
    &= N_{q}^{K+1},
\end{align*}
which allows us to simplify the second term. Finally, we see that the coefficient of $Z_\ell$ is given by
$$
     \frac{1}{ \sqrt{n-1}} N_{0}^{K+1} \cdot \hat{W}_{n, i, \ell}^1 +\frac{1}{\sqrt{n-1}} \sum_{q=1}^K \left( N_q^{K+1} \cdot \hat{W}_{n, i, \ell}^{q+1} \right) + \frac{1}{\sqrt{n-1}} N_q^{K+1} \cdot \hat{W}_{n, i,l}^{K+2} $$
     $$=\frac{1}{\sqrt{n-1}} \sum_{\ell \leq n} \sum_{q=0}^{K+1} N_q^K \cdot \hat{W}_{n, i, \ell}^{q+1}.$$
Hence, we obtain that 
\begin{equation}
    \lambda_i^{K+1} = \left( \frac{1}{\sqrt{n-1}} \sum_{\ell \leq n} \sum_{q=0}^{K+1} N_q^K \cdot \hat{W}_{n, i, \ell}^{q+1} \right)Z_\ell,
\end{equation}
as desired.
\end{proof}

\subsection{Expectation of Dot Products and their Concentration}

The following lemma shows that the expectation of the dot products of the embedding vectors, conditional on the graph, is a linear combination of the empirical moments $\hat{W}_{i,j}^k$.
\begin{lemma}
\label{lemma:expectation-dotproducts}
Suppose that $\lambda_i^k$ are produced through \cref{algo1}. Then, conditional on the latent features $(\omega_i)_{i=1}^n$ and the adjacency matrix $A$, we have 
\begin{equation}
    \E\left[ \langle \lambda_i^{k_1}, \lambda_j^{k_2} \rangle | A, (\omega_i)_{i=1}^n \right] = \frac{1}{d_n} \sum_{q_1=0}^{k_1} \sum_{q_2=0}^{k_2} \tr \left( \left( N_{q_1}^{k_1} \right)^T N_{q_2}^{k_2}\right)   \hat{W}_{n, i, j}^{q_1+q_2+2}.
\end{equation}
\end{lemma}

\begin{proof}[Proof of Lemma \ref{lemma:expectation-dotproducts}]

Firstly, we note that if $W \sim \mathcal{N}(0, I_k),$ then $\E[W^T A W] = \tr(A)$. Then, we can compute
\begin{align}
    \E \left[ \langle \lambda_i^{k_1}, \lambda_j^{k_2} \rangle | A, (\omega_i)_{i=1}^n \right] &= \frac{1}{d_n} \sum_{\ell \leq n} \tr \left( \left( \frac{1}{\sqrt{n-1}} \sum_{q=0}^{k_1}  N_q^{k_1}\cdot \hat{W}_{n, i, \ell}^{q+1}(A)  \right)^T \left( \frac{1}{\sqrt{n-1}} \sum_{q=0}^{k_2} N_q^{k_2}\cdot \hat{W}_{n, j, \ell}^{q+1}(A)  \right) \right) \\
    &= \frac{1}{d_n} \sum_{q_1=0}^{k_1} \sum_{q_2=0}^{k_2} \tr \left(  \left( N_{q_1}^{k_1} \right)^T N_{q_2}^{k_2} \frac{1}{n-1} \sum_{\ell \leq n} \hat{W}_{n, i, \ell}^{q_1+1}(A) \hat{W}_{n, j, \ell}^{q_2+1}(A) \right) \\
    &= \frac{1}{d_n} \sum_{q_1=0}^{k_1} \sum_{q_2=0}^{k_2} \tr \left( \left( N_{q_1}^{k_1} \right)^T N_{q_2}^{k_2}\right)   \hat{W}_{n, i,j}^{q_1+q_2+2}. \label{eq_linearcombo}
\end{align}
\end{proof}

Now that these properties of the embedding vectors have been shown, we now return to the setting of our algorithm, where the weight matrices $M_{k,i}$ are chosen to be the identity. We now prove that the algorithm to produce estimators $\hat{q}_{ij}^{(k)}$ for $W_{ij}^{(k)}$ in Algorithm \ref{algo1} is asymptotically consistent, and we establish the convergence rate. For the reader's convenience, we rewrite the algorithm below. The following lemma explains the intuition as to why we expect $\hat{q}_{i,j}^{(k)}$ to be an estimator for $\hat{W}_{i,j}^{(k)}.$

\begin{algorithm}[h]
\caption{GNN Architecture and Estimators for the Graphon Moments}
\textbf{Input:} a Graph $G=(V,E)$; $n := |V|.$\\
\textbf{Output:} estimators $\hat{q}_{ij}$ for the edge probability $W_{ij}.$ \\

\textbf{Computing Estimators for $W_{ij}^{(k)}$:}

\For{$i \neq j$}{
    $\hat{q}_{i,j}^{(2)} := \langle \lambda_i^0, \lambda_j^0 \rangle.$
}
\For{$k \in \{3, 4, \dots, L+2\}$}{
    $\hat{q}_{i,j}^{(k)}:= \langle \lambda_i^{k-2}, \lambda_j^0 \rangle -\sum_{r=0}^{k-3} \binom{k-2}{r} \hat{q}_{r+2}$
}

\textbf{Return: $ \big \{ (\hat{q}_{ij}^{(2)}, \hat{q}_{ij}^{(3)}, \dots, \hat{q}_{ij}^{(L+2)})_{i \neq j} \big \}$}
\end{algorithm}

\begin{lemma}
\label{lemma:form-of-q}
As in \cref{algo1}, define (with the weight matrices $M_{k, i} = I_{d_n}$)
$$
\hat{q}_{i,j}^{(2)} := \langle \lambda_i^0, \lambda_j^0 \rangle, \quad \hat{q}_{i,j}^{(k)}:= \langle \lambda_i^{k-2}, \lambda_j^0 \rangle -\sum_{r=0}^{k-3} \binom{k-2}{r} \hat{q}_{r+2}.
$$
Then $$\hat{q}_{i,j}^{(k)} = \left \langle \frac{1}{\sqrt{n-1}} \sum_{\ell \leq n} a_{j, \ell} Z_{\ell}, \frac{1}{\sqrt{n-1}} \sum_{\ell \leq n} \hat{W}_{n, i, \ell}^{(k-1)} Z_\ell  \right \rangle.$$ Under the heuristic that $Z_{\ell_1}^T Z_{\ell_2} \approx \mathbb{I}(\ell_1 = \ell_2)$, then we see that $\hat{q}_{i,j}^{(k)} \approx \hat{W}_{n, i, \ell}^{(k)}.$
\end{lemma}

\begin{proof}[Proof of \cref{lemma:form-of-q}]
We first show that we can write 
\begin{equation}
\label{eq:comb-identity-lastminutelmao}
    \hat{q}_{i,j}^{(k)} = \left \langle \lambda_j^0, \sum_{r=0}^{k-2} \binom{k-2}{k-2-r}(-1)^{k-2-r} \cdot \lambda_i^{r}  \right \rangle.
\end{equation}
Note equivalently this can be written as 
$$    
\hat{q}_{i,j}^{(k)} = \left \langle \lambda_j^0, \sum_{r=0}^{k-2} \binom{k-2}{r}(-1)^{r} \cdot \lambda_i^{k-2-r}  \right \rangle.
$$


We show this using induction. Assume this is true for all $k \leq K$ for some $K$. We can compute $\hat{q}_{i,j}^{(K+1)}$ using the formula in \cref{algo1}. Using the definition of $\hat{q}_{i,j}^{(K+1)}$ we can compute that the coefficient of $\lambda_i^a$ in $\hat{q}_{i,j}^{(K+1)}$ is given by $$-\sum_{r=a}^{K-2} \binom{K-1}{r} \binom{r}{r-a} (-1)^{r-a} = -\sum_{r=a}^{K-2} \binom{K-1}{K-r-1} \binom{r}{a} (-1)^{r-a}.$$ To compute this, we first argue that $$\sum_{r=a}^{K-1} \binom{K-1}{K-r-1} \binom{r}{a} (-1)^{r} = 0.$$ We use generating functions. We note that $\binom{K-1}{r}(-1)^r$ is the coefficient of $x^{K-r-1}$ in the expansion of $(1-x)^{K-1}.$ Then, we note that $\binom{r}{a}$ is the coefficient of $x^{r-a}$ in the expansion of $\frac{1}{(1-x)^{a+1}}.$ Hence, this summation simply represents the coefficient of $x^{K-a-1}$ in the expansion of $(1-x)^{K-a-2}.$ However, since $(1-x)^{K-a-2}$ is a degree $K-a-2$ polynomial, the coefficient is simply 0. Hence, this implies that $$-\sum_{r=a}^{K-2} \binom{K-1}{K-r-1} \binom{r}{a} (-1)^{r-a} = (-1)^{K-1-a} \binom{K-1}{a}.$$ Thus, we have shown that the coefficient of $\lambda_i^a$ in $\hat{q}_{i,j}^{(K+1)}$ is of the desired form, which suffices to prove \cref{eq:comb-identity-lastminutelmao}. Now, continuing with the proof, we recall that \cref{prop:form-of-embedding} states that $$\lambda_i^k = \frac{1}{\sqrt{n-1}} \sum_{\ell \leq n}  \left( \sum_{q=0}^k  \binom{k}{q} \cdot \hat{W}_{n, i, \ell}^{(q+1)}  \right)Z_\ell,$$
so 
$$
\hat{q}_{i,j}^{(k)} = \left \langle \lambda_j^0, \sum_{r=0}^{k-2} \binom{k-2}{r}(-1)^r \cdot \frac{1}{\sqrt{n-1}} \sum_{\ell \leq n} \left( \sum_{q=0}^{k-2-r} \binom{k-2-r}{q} \cdot \hat{W}_{n, i, \ell}^{(q+1)} Z_\ell  \right)  \right \rangle.
$$
We analyze the second term in the dot product more closely. The coefficient of $Z_\ell$ in the second term is equal to (ignoring the factor of $1/\sqrt{n-1}$ for now)
\begin{align*}
&\sum_{r=0}^{k-2} \sum_{q=0}^{k-2-r} (-1)^r \binom{k-2}{r} \binom{k-2-r}{q} \hat{W}_{n, i, \ell}^{(q+1)} \\   
&= \sum_{q=0}^{k-2} \sum_{r=0}^{k-2-r} (-1)^r \binom{k-2}{r} \binom{k-2-r}{q} \hat{W}_{n, i, \ell}^{(q+1)} \\
&= \sum_{q=0}^{k-2} \hat{W}_{n, i, \ell}^{(q+1)} \sum_{r=0}^{k-2-q} (-1)^r \binom{k-2}{r} \binom{k-2-r}{q}.
\end{align*} 

Hence it suffices to argue that $\sum_{r=0}^{k-2-q} (-1)^r \binom{k-2}{r} \binom{k-2-r}{q} = 1$ if $q = k-2$, and 0 otherwise. We argue this in \cref{lemma:binomial-identity-2}. Assuming that this is true, then we see that 
\begin{align*}    
\hat{q}_{i,j}^{(k)} &= \left \langle \lambda_j^0, \sum_{r=0}^{k-2} \binom{k-2}{r}(-1)^r \cdot \frac{1}{\sqrt{n-1}} \sum_{\ell \leq n} \left( \sum_{q=0}^{k-2-r} \binom{k-2-r}{q} \cdot \hat{W}_{n, i, \ell}^{(q+1)} Z_\ell  \right)  \right \rangle \\
&= \left \langle \lambda_j^0, \frac{1}{\sqrt{n-1}}\sum_{\ell \leq n}  \hat{W}_{n, i, \ell}^{(q+1)} Z_\ell   \right \rangle \\
&= \left \langle \frac{1}{\sqrt{n-1}}\sum_{\ell \leq n} a_{j\ell} Z_\ell, \frac{1}{\sqrt{n-1}} \sum_{\ell \leq n}  \hat{W}_{n, i, \ell}^{(q+1)} Z_\ell   \right \rangle, 
\end{align*}
as desired. To conclude the proof, we present and prove \cref{lemma:binomial-identity-2}.
\begin{lemma}
\label{lemma:binomial-identity-2}
Let $k \ge 0$ be an integer. Then $$\sum_{r=0}^{k-q} (-1)^r \binom{k}{r} \binom{k-r}{q} = \begin{cases} 0 & q < k \\
1 & q= k \end{cases}.$$
\end{lemma}

\begin{proof}
Consider the formal series
    $$(1+x)^k = \sum_{s=0}^k \binom{k}{s} x^s, \quad \frac{1}{(x+1)^{q+1}} = \sum_{s=q}^\infty (-1)^{s-q} \binom{s}{q} x^{s-q}.$$ Multiplying these two series, we notice that the desired quantity $\sum_{r=0}^{k-q} (-1)^r \binom{k}{r} \binom{k-r}{q}$ is exactly the coefficient of $x^{k-q}$ in the product of the two series, which is $(1+x)^{k-q-1}.$ However, $x^{k-q}$ is a monomial of degree $k-q$, and hence has coefficient 0 in $(1+x)^{k-q-1}$, which has degree $(1+x)^{k-q-1}$ when $q < k$. The notable exception is when $k=q$, and then the coefficient (of the constant term) in $(1+x)^{-1}$ is exactly equal to 1. This suffices for the proof.
\end{proof}

\end{proof}

We now establish the main concentration result, \cref{prop:estimators-to-empiricalmoments}. Before doing so, we first state the following lemma.

\begin{lemma}
\label{lemma:gaussian-dotproduct}
    Let $\xi = (\xi_1, \dots, \xi_n),$ and let $\xi_1, \dots, \xi_n$ be independent, zero-mean normal random variables with for all $i = 1, 2, \dots n,$ $\E[\xi_i^2] = \sigma_i^2.$ Let $D = \text{Diag}(\sigma_1, \dots, \sigma_n).$ Let $B$ be any $n \times n$ real matrix. Then for all $\epsilon > 0,$
    \begin{equation}
\P\left( \left| \xi^T B \xi - \E[\xi^T B \xi] \right| > \epsilon \right) \leq \exp \left(-\min \left( \frac{\epsilon}{4 \|DBD\|_F}, \frac{\epsilon^2}{16 \|DBD\|_F^2} \right)\right)
\end{equation}
\end{lemma}

\begin{proof}[Proof of \cref{lemma:gaussian-dotproduct}]
\label{proof-of-lemma:gaussian-dotproduct}
We adapt Proposition 1 from \cite{bellec2019concentration}, which states the following.

Let $\xi = (\xi_1, \dots, \xi_n),$ and let $\xi_1, \dots, \xi_n$ be independent, zero-mean normal random variables with for all $i = 1, 2, \dots n,$ $\E[\xi_i^2] = \sigma_i^2.$ Let $D = \text{Diag}(\sigma_1, \dots, \sigma_n).$ Let $B$ be any $n \times n$ real matrix. Then for any $x > 0$, $$\P\left( \left| \xi^T B \xi - \E[\xi^T B \xi] \right| > 2 \|DBD\|_F \sqrt{x} + 2 \|DBD\|_2 x \right) \leq \exp (-x).$$

To adapt this proposition into the form in \cref{lemma:gaussian-dotproduct}, we firstly note that $\|X\|_2 \leq \|X\|_F,$ so 
\begin{align}
2 \|DBD\|_F \sqrt{x} + 2 \|DBD\|_2 x &\leq 2 \|DBD\|_F ( \sqrt{x} + x) \\
&\leq 4\|DBD\|_F \cdot \max(\sqrt{x}, x)
\end{align} Then
\begin{align}
&\P\left( \left| \xi^T B \xi - \E[\xi^T B \xi] \right| > 4 \|DBD\|_F \cdot \max(\sqrt{x}, x) \right) \\
&\leq \P\left( \left| \xi^T B \xi - \E[\xi^T B \xi] \right| > 2 \|DBD\|_F (\sqrt{x}+x) \right) \\
&\leq \P\left( \left| \xi^T B \xi - \E[\xi^T B \xi] \right| > 2 \|DBD\|_F \sqrt{x} + 2 \|DBD\|_2 x \right) \\
&\leq \rm{exp}(-x),
\end{align}
which implies that 
\begin{equation}
\P\left( \left| \xi^T B \xi - \E[\xi^T B \xi] \right| > \epsilon \right) \leq \exp \left(-\min \left( \frac{\epsilon}{4 \|DBD\|_F}, \frac{\epsilon^2}{16 \|DBD\|_F^2} \right)\right),
\end{equation}    
as desired.
\end{proof}

\begin{proposition}[\cref{prop:graph_concentration_sparse}]
\label{prop:estimators-to-empiricalmoments}
    Suppose that $L \le n$ and that (\ref{asp2}) holds. Then, conditional on $A$ and $(\omega_i)_{i=1}^n,$ with probability at least $1 - 5/n - n \cdot \exp \left( -\delta_W \rho_n(n-1)/3 \right)$ we have that for all $2 \leq k\leq L+2,$
\begin{equation}
        \sup_{i \neq j \in [n]} \left| \hat{q}_{i,j}^{(k)} - {W}_{n, i, j}^{(k)}\right| \leq \frac{\rho_n^{k-1}}{\sqrt{n-1}} \log(n)^k \left[ 3a_k\sqrt{\rho_n} + \frac{96 a_{k-1}}{\sqrt{d_n}} \right],
\end{equation}
where $a_k = C  (8(k+2))^k k^{k+1}\sqrt{k!},$ where $C$ is some absolute positive constant.
\end{proposition}

    

We first introduce the following lemma:

\begin{lemma}
\label{lemma:upper-bound-degree}
Suppose that the graphon $W$ satisfies condition \ref{asp2}, and suppose the sparsity factor is $\rho_n.$ Then, 
$$\P\left( \max_{i \in [n]} \frac{1}{n-1} \sum_{\stackrel{j \leq n}{j \neq i}} a_{ij} \ge \rho_n(1+\delta) \Big| (\omega_i)_{i=1}^n \right) \leq n \cdot \exp \left( -\frac{\delta^2}{2+\delta} \sum_{\stackrel{j \leq n}{j \neq i}} \rho_n W(\omega_i, \omega_j) \right).$$ 
Choosing $\delta = 1$ yields that with probability at least $1 - n \cdot \exp \left( -\frac{\delta_W}{3} \rho_n(n-1) \right),$ conditional on $(\omega_i)_{i=1}^n,$ $$\max_{i \in [n]} \frac{1}{n-1} \sum_{\stackrel{j \leq n}{j \neq i}} a_{ij} < 2\rho_n.$$ Summing over all $i$, this implies that 
\end{lemma}
\begin{proof}[Proof of \cref{lemma:upper-bound-degree}]
We use the following lemma about sums of independent Bernoulli random variables:
\begin{lemma}[\cite{chernoff-lecture}, Theorem 4]
    Let $X = \sum_{i=1}^n X_i$, where $X_i \sim \text{Bern}(p_i),$ and all the $X_i$ are independent. Let $\mu = \E[X] = \sum_{i=1}^n p_i.$ Then $$\P(X \ge (1+\delta) \mu) \leq \exp \left( - \frac{\delta^2}{2+\delta} \mu \right)$$ for all $\delta > 0.$
\end{lemma}
Fix $i$. We note that the random variables $(a_{ij})_{j \neq i}$ are independent conditioned on the $(\omega_r)_{r=1}^n.$ Using these variables directly in this lemma above yields $$\P \Bigg( \sum_{\stackrel{j \leq n}{j \neq i}} a_{ij}   \ge (1+\delta) \sum_{\stackrel{j \leq n}{j \neq i}} \rho_n W(\omega_i, \omega_j) \Bigg) \leq \exp \Big( - \frac{\delta^2}{2+\delta} \sum_{\stackrel{j \leq n}{j \neq i}} \rho_n W(\omega_i, \omega_j)\Big).$$ Then, noting that $\delta_W \leq W(\cdot, \cdot) \leq 1$, and substiting $\delta = 1$, we obtain $$\P \Bigg( \frac{1}{n-1} \sum_{\stackrel{j \leq n}{j \neq i}} a_{ij} < 2 \rho_n \Bigg) \ge 1- \exp \Big( - \frac{\delta_W}{3} \rho_n (n-1) \Big).$$ A union bound over all $i \in [n]$ concludes the proof.

\end{proof}


\begin{proof}[Proof of \cref{prop:estimators-to-empiricalmoments}]

In the remainder of this proof, we condition on the event in \cref{lemma:concentration-of-graph}, which is that 
\begin{align*}
    \max_{i \neq j} \big|\hat{W}_{n, i, j}^{(k)} - {W}_{n, i, j}^{(k)} \big| \leq 3a_k \frac{\rho_n^{k-1/2}}{\sqrt{n-1}} \log(n)^k.
\end{align*}    
This contributes the probability of $3/n.$ For simplicity of notation, denote $B_{n,k} := 3a_k \frac{\rho_n^{k-1/2}}{\sqrt{n-1}} \log(n)^k.$ We also condition on the event in \cref{lemma:upper-bound-degree}, which contributes the probability of $n \cdot \rm{exp}(-\delta_W \rho_n(n-1)/3)$.


Fix some $i \neq j$. We prove the claim for this particular choice of $i, j$, and then union bound over all pairs at the end of the proof. Recall that \cref{lemma:form-of-q} states that $$\hat{q}_{i,j}^{(k)} = \left \langle \frac{1}{\sqrt{n-1}} \sum_{\ell \leq n} a_{j, \ell} Z_{\ell}, \frac{1}{\sqrt{n-1}} \sum_{\ell \leq n} \hat{W}_{n, i, \ell}^{(k-1)} Z_\ell  \right \rangle.$$ We first note that because $Z_{\ell} \sim \frac{1}{\sqrt{d_n}} \mathcal{N}(0, I_{d_n}),$ we have that
$$\E_{(Z_\ell)} \left[ \hat{q}_{i,j}^{(k)} \big| A, (\omega_i) \right] = \frac{1}{n-1} \sum_{\ell \leq n} a_{j, \ell} \hat{W}_{n, i, \ell}^{(k-1)} = \hat{W}_{n, i, \ell}^{(k)},$$ where this is the expectation is over the randomness in the Gaussian vectors $(Z_\ell).$ Hence, to show the desired result, it suffices just to show the concentration of a quadratic form of Gaussian vectors. Concretely, writing $Z = (Z_1, Z_2, \dots,  Z_n) ,$ we can write $\langle \lambda_i^k, \lambda_j^0 \rangle = Z^T C Z,$ where 
$$C = \begin{pmatrix} C_{11} & C_{12} & \dots & C_{1n} \\ 
C_{21} & C_{22} & \dots & C_{2n} \\
\vdots & \vdots & \ddots & \vdots \\
C_{n1} & C_{n2} & \dots & C_{nn}
\end{pmatrix}$$ and 
\begin{align*}
C_{m_1 m_2} &= \left( \frac{1}{\sqrt{n-1}}  a_{j, m_2}  \right)\cdot  \left( \frac{1}{\sqrt{n-1}} \hat{W}_{n, i, m_1}^{(k-1)} \right) I_{d_n} \\
&= \frac{a_{j, m_2} \hat{W}_{n, i, m_1}^{(k-1)}}{n-1} \cdot I_{d_n}.
\end{align*}
To show the concentration of the quadratic form $Z^T C Z,$ we employ \cref{lemma:gaussian-dotproduct} to do this. In order to apply \cref{lemma:gaussian-dotproduct}, we first bound the Frobenius norm of $C$. Noting that $$\|C\|_F = \sqrt{ \sum_{m_1, m_2 \leq n} \|C_{m_1 m_2}\|_F^2},$$ we can write
\begin{align*}
    \|C\|_F &= \sqrt{\frac{d_n}{(n-1)^2} \sum_{m_1, m_2 \leq n} a_{j, m_2} \left( \hat{W}_{n, i, m_1}^{(k-1)}\right)^2 } \\
    &= \sqrt{\frac{d_n}{(n-1)^2} \Big( \underbrace{\sum_{m_2 \leq n} a_{j, m_2}}_{\leq 2\rho_n} \Big) \cdot \left( \sum_{m_1 \leq n} \left( \hat{W}_{n, i, m_1}^{(k-1)}\right)^2 \right)  } \\
    &\leq \sqrt{ 3\rho_n d_n B_{n, k-1}^2} \\
    &= B_{n, k-1} \sqrt{3 \rho_n d_n}.
\end{align*}
For ease of notation, we will write $\|C\|_F \leq F B_{n, k-1} \sqrt{\rho_n} \sqrt{d_n}$ for some constant $F \leq \sqrt{3}.$ We now use \cref{lemma:gaussian-dotproduct}. Noting that each element of $Z$ is an independent $N(0, 1/d_n)$ random variable, then \cref{lemma:gaussian-dotproduct} states that when $\frac{\epsilon \sqrt{d_n}}{4 F B_{n, k-1} \sqrt{\rho_n}} > 1,$ we have
$$\P \left( \left| \hat{q}_{i,j}^{(k)} - \hat{W}_{n, i, j}^{(k)} \right| > \epsilon \Big| A, (\omega_i)_{i=1}^n \right) \leq 2 \exp \left( -\frac{\epsilon \sqrt{d_n}}{4F B_{n, k-1} \sqrt{\rho_n}} \right).
$$
Choose $\epsilon = \frac{4F B_{n, k-1} \sqrt{\rho_n}}{\sqrt{d_n}} t.$ Then
$$\P \left( \left| \hat{q}_{i,j}^{(k)} - \hat{W}_{n, i, j}^{(k)} \right| > \frac{4F B_{n, k-1} \sqrt{\rho_n}}{\sqrt{d_n}} t \Big| A, (\omega_i)_{i=1}^n \right) \leq 2 \exp \left( -t\right).
$$
Now, union bounding over all $k \in \{2, 3, \dots, L+2\}$ and $i < j,$ $i, j \in [n],$ we have that with probability at least $1 - 2 \rm{exp}(-t + 3\log(n))$, for all $2 \leq k \leq L+2,$  (assuming $L \leq  n-1$), conditional on $A$, $(\omega_i)_{i=1}^n,$ 
$$
\left| \hat{q}_{i,j}^{(k)} - \hat{W}_{n, i, j}^{(k)} \right| \leq \frac{4F B_{n, k-1} \sqrt{\rho_n}}{\sqrt{d_n}} t.
$$
Taking $t = 4 \log(n)$, we have that with probability $1 - 2/n,$ for all $2 \leq k \leq L+2,$  conditional on $A$, $(\omega_i)_{i=1}^n,$ and using that $F \leq \sqrt{3},$
$$
\left| \hat{q}_{i,j}^{(k)} - \hat{W}_{n, i, j}^{(k)} \right| \leq \frac{32 B_{n, k-1} \sqrt{\rho_n}}{\sqrt{d_n}} \log(n).
$$
Now, recall that $$\max_{i \neq j} \big|\hat{W}_{n, i, j}^{(k)} - {W}_{n, i, j}^{(k)} \big| \leq 3a_k \frac{\rho_n^{k-1/2}}{\sqrt{n-1}} \log(n)^k = B_{n, k}.$$ Hence, the triangle inequality implies that for all $2 \leq k \leq L+2,$ with probability at least $1 -5/n - n \cdot \rm{exp}(-\delta_W \rho_n(n-1)/3)$, 
\begin{align*}
    \left| \hat{q}_{i,j}^{(k)} - {W}_{n, i, j}^{(k)} \right| &\leq B_{n,k} + \frac{32 B_{n, k-1} \sqrt{\rho_n}}{\sqrt{d_n}} \log(n) \\
    &\leq 3a_k \frac{\rho_n^{k-1/2}}{\sqrt{n-1}} \log(n)^k + \frac{32}{\sqrt{d_n}} \cdot 3a_{k-1} \frac{\rho_n^{k-1}}{\sqrt{n-1}}\log(n)^{k} \\
    &= \frac{\rho_n^{k-1}}{\sqrt{n-1}} \log(n)^k \left[ 3 a_k \sqrt{\rho_n} + \frac{96 a_{k-1}}{\sqrt{d_n}} \right],
\end{align*}
as desired.
\end{proof}

\section{Proof of \cref{prop:convergence_of_c}}

In this section, we prove \cref{prop:convergence_of_c}. We first review some notation used below.

Define the vectors 
\begin{gather*}
W_n^{(2, k)}(x,y) := \big(  W^{(2)}_n(x,y),  \dots,  W_n^{(k)}(x,y) \big) \nonumber \\ 
\hat{q}^{(2, k)}_{ij} = \left(\hat{q}^{(2)}_{ij}, \hat{q}^{(3)}_{ij}, \dots, \hat{q}^{(k)}_{ij} \right), \nonumber
\end{gather*}
and recall that $W_{n,i,j}$ denotes $W_n(\omega_i, \omega_j).$ Define
\begin{align*}
    r(n, d_n, m) &:= \max_{2\le k\le m+1}  \rho_n^{-(k-1)} \left( \frac{\rho_n^{k-1}}{\sqrt{n-1}} \log(n)^k \left[ 3 a_k \sqrt{\rho_n} + \frac{96 a_{k-1}}{\sqrt{d_n}} \right]\right) \\
    &= \max_{2\le k\le m+1} \left( \frac{\log(n)^k}{\sqrt{n-1}} \left[ 3 a_k \sqrt{\rho_n} + \frac{96 a_{k-1}}{\sqrt{d_n}} \right]\right)
\end{align*}
We note that when $\rho_n \gg \log(n)^{2(m+1)}/n,$ $r(n, d_n, m) = o(\rho_n).$ The term in the parentheses in the first equation is simply the bound on $\big|\hat{q}_{i,j}^{(k)} - W_{n,i,j}^{(k)}\big|$ presented in \cref{prop:graph_concentration_sparse}. $r(n, d_n, m)$ will be a natural quantity that appears later in this section.

We also define  $$R(\beta) = \E \left[ \left( \left \langle \beta,  W_n^{(2, 1+\rm{len}(\beta))}(x,y)   \right \rangle - W_n(x,y) \right) ^2 \right]$$ where the expectation is over $x,y \sim \text{Unif}(0,1),$ and define the empirical risk $R_n(\beta)$ as $$R_n(\beta) = \frac{2}{n(n-1)} \sum_{i < j}^n \left( \left \langle \beta,  \hat{q}^{(2, 1+\rm{len}(\beta))}_{ij}  \right \rangle - a_{ij} \right)^2.$$  We also define the out-of-sample test error as 
\begin{equation}
    R_T(\beta) = \E \left[ \left( \left \langle \beta,  \hat{q}_{n+1,n+2}^{(2, 1+\rm{len}(\beta))}   \right \rangle - W_n(\omega_{n+1}, \omega_{n+2}) \right)^2 \right].
\end{equation}


The following proposition is the main component of \cref{prop:convergence_of_c}.

\begin{proposition}[\cref{prop:convergence_of_c}]
\label{prop:main-thm-prop}
Let $\mathcal{F}= \prod_{i=1}^k [-a_i, a_i]$ be a subset of $\mathbb{R}^k,$ where $a_i = b_i/\rho_n^i$ for some $b_i > 0.$ Let $k$ be a positive integer and define $$\hat{\beta}^{n, k} := \argmin_{\beta \in \mathcal{F}} R_n(\beta), \quad \beta^{*, k} := \argmin_{\beta \in \mathcal{F}} R(\beta).$$ Define $D = \sum_{i=1}^k |b_i|.$ Then with probability at least $1 - 5/n - n \cdot \rm{exp}(-\delta_W \rho_n(n-1)/3) - \delta,$

$$R_T(\hat{\beta}^{n,k}) \leq R(\beta^{*, k}) + 6D\rho_n \cdot r(n, d_n, k)(T+ 2) + 3D^2 r(n, d_n, k)^2 + \tilde{O}\left( \frac{\rho_n^2 (T+1)^2}{\sqrt{n}} \right)$$

where $T = (1-\delta_W) \sum_{r=1}^k b_r (1-\delta_W)^r$ and the $\tilde{O}$ constant depends on $\sqrt{\log(1/\delta)}.$
\end{proposition}


\begin{proof}[Proof of \cref{prop:convergence_of_c}]
We write 
\begin{equation}
\label{eq:prop46-1}
    R_T(\hat{\beta}^{n,k}) \leq R(\beta^{*, k}) + |R(\hat{\beta}^{n,k}) - R(\beta^{*,k})| + |R_T(\hat{\beta}^{n,k}) - R(\hat{\beta}^{n,k})|. 
\end{equation}
We first bound 
$$R(\hat{\beta}^{n,k}) - R(\beta^{*, k}) = \Big[ R(\hat{\beta}^{n,k}) - R_n(\hat{\beta}^{n,k}) \Big] + \Big[ R_n(\hat{\beta}^{n,k}) - R_n(\hat{\beta}^{*, k}) \Big] + \Big[ R_n(\hat{\beta}^{*, k}) - R(\hat{\beta}^{*, k}) \Big].$$ We note that the LHS is $\ge  0$ by definition of $\beta^{*, k}.$ We note that the second term on the RHS is $\leq 0$ by definition of  $\hat{\beta}^{n, k}.$ Hence, it follows that $$|R(\hat{\beta}^{n,k}) - R(\beta^{*, k})| \leq \left| \Big[ R(\hat{\beta}^{n,k}) - R_n(\hat{\beta}^{n,k}) \Big] + \Big[ R_n(\hat{\beta}^{*, k}) - R(\hat{\beta}^{*, k}) \Big] \right|.$$ \cref{lemma:empirical-minus-risk} states that $$R_n(\beta) - R(\beta) = \frac{2}{n(n-1)} \sum_{i < j} (a_{ij} -(W_{n,ij})^{2} ) + S_2(\beta) + S_3(\beta) + K_n(\beta) - \E[K_n(\beta)],$$ so hence 
\begin{align}
|R(\hat{\beta}^{n,k}) - R(\beta^{*, k})| &\leq | \Big[ R(\hat{\beta}^{n,k}) - R_n(\hat{\beta}^{n,k}) \Big] + \Big[ R_n(\hat{\beta}^{*, k}) - R(\hat{\beta}^{*, k}) \Big]| \nonumber \\
&\leq |S_2(\hat{\beta}^{n,k}) + S_3(\hat{\beta}^{n,k}) + K_n(\hat{\beta}^{n, k}) - \E[K_n(\hat{\beta}^{n, k})]| \nonumber \\
&+ |S_2(\hat{\beta}^{*, k}) + S_3(\hat{\beta}^{*, k}) + K_n(\hat{\beta}^{*, k}) - \E[K_n(\hat{\beta}^{*, k})]| \nonumber \\
&\leq 4D\rho_n \cdot r(n, d_n, k) \cdot (T+2) + 2D^2 r(n, d_n, k)^2 + \tilde{O}\left( \frac{\rho_n^2 (T+1)^2}{\sqrt{n}} \right), \label{eq:prop46-2}
\end{align}
where the last inequality follows from \cref{lemma:empirical-minus-risk}. We now bound $|R_T(\hat{\beta}^{n,k}) - R(\hat{\beta}^{n,k})|.$ For any $\beta \in S$, write
\begin{align*}
    R_T(\beta) &= \E \left[ \left( \left \langle \beta,  \hat{q}_{n+1,n+2}^{(2, k+1)}   \right \rangle -  W_n(\omega_{n+1}, \omega_{n+2}) \right)^2 \right] \\
    &=\E \left[ \left( \left \langle \beta,  W_{n, n+1,n+2}^{(2, k+1)}   \right \rangle + \left \langle \beta, \hat{q}_{n+1,n+2}^{(2, k+1)} - W_{n, n+1,n+2}^{(2, k+1)} \right \rangle-  W_{n, n+1, n+2}  \right)^2 \right] \\
    &=\underbrace{\E \left[ \left( \left \langle \beta,  W_{n, n+1,n+2}^{(2, k+1)}   \right \rangle - W_{n, n+1, n+2} \right)^2  \right]}_{R(\beta)} \\
    &+ \E \left[ 2 \cdot \left( \left \langle \beta,  W_{n, n+1,n+2}^{(2, k+1)}   \right \rangle - W_{n, n+1, n+2} \right) \cdot \left \langle \beta, \hat{q}_{n+1,n+2}^{(2, k+1)} - W_{n, n+1,n+2}^{(2, k+1)} \right \rangle  \right]\\
    &+ \E \left[ \left \langle \beta, \hat{q}_{n+1,n+2}^{(2, k+1)} - W_{n, n+1,n+2}^{(2, k+1)} \right \rangle^2  \right], 
    \end{align*}
which implies that (using similar arguments as in the proof of \cref{lemma:bound-on-S2}),
\begin{equation}
\label{eq:prop46-3}
    |R_T(\beta) - R(\beta)| \leq 2 D\rho_n  (T + 1) r(n, d_n, k) + D^2r(n, d_n, k)^2 
\end{equation}
Substituting \cref{eq:prop46-2} and \cref{eq:prop46-3} into \cref{eq:prop46-1} yields the desired result.

\end{proof}

The following lemma is used directly in the above proof of \cref{prop:convergence_of_c}. We state it and prove it below.

\begin{lemma}
\label{lemma:empirical-minus-risk}
    Let $\mathcal{F}= \prod_{i=1}^k [-a_i, a_i]$ be a subset of $\mathbb{R}^k,$ where $a_i = b_i/\rho_n^i$ for some $b_i > 0.$ Let $\beta \in S$ be arbitrary. Define $D = \sum_{i=1}^k |b_i|.$ Then
$$R_n(\beta) - R(\beta) - \frac{2}{n(n-1)} \sum_{i < j} (a_{ij} -(W_{n,ij})^{2} ) = S_2(\beta) + S_3(\beta) + K_n(\beta) - \E[K_n(\beta)].$$ 
Furthermore, employing \cref{uniform-concentration}, \cref{lemma:bound-on-S2}, and \cref{lemma:bound-on-S3} implies that with probability at least $1 - 5/n - n \cdot \rm{exp}(-\delta_W \rho_n(n-1)/3) - \delta,$
$$
\left|R_n(\beta) - R(\beta) - \frac{2}{n(n-1)} \sum_{i < j} (a_{ij} -(W_{n,ij})^{2} )  \right| \leq 2D\rho_n  \cdot r(n, d_n, k) \cdot (T+2) + D^2 \cdot r(n, d_n, k)^2 + \tilde{O}\left( \frac{\rho_n^2(T+1)^2}{\sqrt{n}} \right), 
$$
where $T = (1-\delta_W) \sum_{r=1}^k b_r (1-\delta_W)^r$ and the $\tilde{O}$ constant depends on $\sqrt{\log(1/\delta)}.$ We note that this is the probability at which this lemma holds, since \cref{uniform-concentration}, \cref{lemma:bound-on-S2}, and \cref{lemma:bound-on-S3} all condition on the same events, so the probabilities in their respective statements do not add.
\end{lemma}

\begin{proof}[Proof of \cref{lemma:empirical-minus-risk}]
Let $\beta \in S$ be arbitrary. Consider
\begin{align*}
    R_n(\beta) &= \frac{2}{n(n-1)} \sum_{i< j} \left( \left \langle \beta, \hat{q}^{(2, k+1)}_{ij} \right \rangle - a_{ij} \right)^2 \\
    &= \frac{2}{n(n-1)} \sum_{i< j} \left( \left \langle \beta, W_{n,ij}^{(2, k+1)} +\hat{q}^{(2, k+1)}_{ij}-W_{n,ij}^{(2, k+1)}\right \rangle - a_{ij} \right)^2 \\
    &= \underbrace{\frac{2}{n(n-1)} \sum_{i < j} \left(\left \langle \beta, W_{n,ij}^{(2, k+1)} \right \rangle - a_{ij} \right)^2}_{S_1(\beta)} \\
    &+ \underbrace{\frac{2}{n(n-1)} \sum_{i < j} \left[ 2 \langle \beta, \hat{q}^{(2, k+1)}_{ij}-W_{n,ij}^{(2, k+1)}\rangle(\langle \beta, W_{n,ij}^{(2, k+1)} \rangle - a_{ij} )  \right]}_{S_2(\beta)} \\
    &+ \underbrace{\frac{2}{n(n-1)} \sum_{i < j} \langle \beta,  \hat{q}^{(2, k+1)}_{ij}-W_{n,ij}^{(2, k+1)}\rangle^2 }_{S_3(\beta)}
\end{align*}

We analyze these three terms successively. We first rewrite $S_1(\beta)$ as 
\begin{align*}
    &\frac{2}{n(n-1)} \sum_{i < j} \left(\left \langle \beta, W_{n,ij}^{(2, k+1)}  \right \rangle^2 - 2a_{ij}\left \langle \beta, W_{n,ij}^{(2, k+1)} \right \rangle + a_{ij} +(W_{n,ij})^{2} -(W_{n,ij})^{2} \right) \\
    &= \frac{2}{n(n-1)} \sum_{i < j} \left[ \left(\left \langle \beta, W_{n,ij}^{(2, k+1)}  \right \rangle^2 - 2a_{ij}\left \langle \beta, W_{n,ij}^{(2, k+1)} \right \rangle +(W_{n,ij})^{2}\right) + \left(a_{ij} -(W_{n,ij})^{2}\right) \right]\\
\end{align*}
We observe that$$\underbrace{\frac{2}{n(n-1)} \sum_{i < j}  \left(\left \langle \beta, W_{n,ij}^{(2, k+1)}  \right \rangle^2 - 2a_{ij}\left \langle \beta, W_{n,ij}^{(2, k+1)} \right \rangle +(W_{n,ij})^{2}\right)}_{K_n(\beta)}$$
has expectation 
\begin{equation}
R(\beta) = \E \left[ \left( \left \langle \beta, W_{n,ij}^{(2, k+1)} \right \rangle - W_{n, ij} \right)^2 \right].
\end{equation}
Hence, we can write
\begin{equation}
    R_n(\beta) = K_n(\beta) + S_2(\beta) + S_3(\beta) + \frac{2}{n(n-1)} \sum_{i < j} (a_{ij} -(W_{n,ij})^{2} ),
\end{equation}
and, we obtain that
\begin{equation} 
R_n(\beta) - R(\beta) - \frac{2}{n(n-1)} \sum_{i < j} (a_{ij} -(W_{n,ij})^{2} ) = S_2(\beta) + S_3(\beta) + K_n(\beta) - \E[K_n(\beta)]
\end{equation}
The result then follows by invoking \cref{uniform-concentration}, \cref{lemma:bound-on-S2}, and \cref{lemma:bound-on-S3}.
\end{proof}


\section{Proofs of \cref{bound:gen-error}, \cref{uniform-concentration}, \cref{lemma:bound-on-S2}, and \cref{lemma:bound-on-S3}}
This section presents \cref{bound:gen-error}, which is used in \cref{prop:convergence_of_c}, and its proof. We also present \cref{uniform-concentration}, \cref{lemma:bound-on-S2}, and \cref{lemma:bound-on-S3}, which are used in the proof of \cref{prop:convergence_of_c}.

\subsection{Proof of \cref{bound:gen-error}}
\begin{lemma}
\label{bound:gen-error}
    Suppose that $W$ has finite distinct rank $m_W,$ and let $\beta^{*, m_W} \in \mathbb{R}^{m_W}$ so that $$W(x,y) = \sum_{r=1}^{m_W} \beta^{*, m_W}_r W^{(r+1)}(x,y).$$ Let $v = (v_1, v_2, \dots, v_{k})$ denote the vector that minimizes $$\left \| W(x,y) - \sum_{r=1}^k v_r W^{(r+1)}(x,y) \right\|_{L^2}.$$ Then $$\left \| W(x,y) - \sum_{r=1}^k v_r W^{(r+1)}(x,y) \right \|_{L^2} \leq \sqrt{ \sum_{s=1}^{m_W} \left[ \sum_{r=k}^{m_W} \beta^{*, m_W}_r \left( \mu_s^{r+1} - \mu_s^{k+1}  \right)  \right]^2}$$
\end{lemma}

\begin{proof}[Proof of \cref{bound:gen-error}]
    By definition of $v$ being a minimizer of $\left \| W(x,y) - \sum_{r=1}^k v_r W^{(r+1)}(x,y) \right\|_{L^2}$, this quantity would be bounded by the error incurred if we replace $v$ with the vector $w = \left( \beta^{*, m_W}_1, \beta^{*, m_W}_2, \dots, \beta^{*, m_W}_{k-1}, \sum_{s=k}^{m_W} \beta^{*, m_W}_s \right)$. This would yield
\begin{align*}
    &\|W(x,y) - \sum_{r=1}^k w_r W^{(r+1)}(x,y) \|_2 \\
    &=\| \sum_{r=1}^{m_W} \beta^{*, m_W}_r W^{(r+1)}(x,y) - \sum_{r=1}^k w_r W^{(r+1)}(x,y) \|_2 \\ 
    &=\| \sum_{r=1}^{m_W} \beta^{*, m_W}_r W^{(r+1)}(x,y) - \sum_{r=1}^{k-1} \beta^{*, m_W}_r W^{(r+1)}(x,y) - \left( \sum_{s=k}^{m_W} \beta^{*, m_W}_s  \right) W^{(k+1)}(x,y) \|_2 \\ 
    &=\| \sum_{r=k}^{m_W} \beta^{*, m_W}_r W^{(r+1)}(x,y) - \left( \sum_{s=k}^{m_W} \beta^{*, m_W}_s  \right) W^{(k+1)}(x,y) \|_2 \\
    &=\| \sum_{r=k}^{m_W} \beta^{*, m_W}_r \left( W^{(r+1)}(x,y) - W^{(k+1)}(x,y)  \right)  \|_2 \\ 
    &=\| \sum_{r=k}^{m_W} \beta^{*, m_W}_r \left( \sum_{s=1}^{m_W} \left( \mu_s^{r+1} - \mu_s^{k+1}  \right) \phi_s(x) \phi_s(y) \right)  \|_2 \\
    &=\| \sum_{s=1}^{m_W} \left( \sum_{r=k}^{m_W} \beta^{*, m_W}_r \left( \mu_s^{r+1} - \mu_s^{k+1}  \right) \right) \phi_s(x) \phi_s(y) \|_2 \\
    &= \sqrt{ \sum_{s=1}^{m_W} \left[ \sum_{r=k}^{m_W} \beta^{*, m_W}_r \left( \mu_s^{r+1} - \mu_s^{k+1}  \right)  \right]^2}
\end{align*}
\end{proof}

\subsection{Proof of \cref{uniform-concentration}}

\begin{lemma}
\label{uniform-concentration}
Let $\mathcal{F}= \prod_{i=1}^k [-a_i, a_i]$ be a subset of $\mathbb{R}^k,$ where $a_i = b_i/\rho_n^i$ for some $b_i > 0.$ Define $D = \sum_{i=1}^k |b_i|,$ and define
$$
K_n(\beta) := \frac{2}{n(n-1)} \sum_{i < j}  \left(\left \langle \beta, W_{n,ij}^{(2, k+1)}  \right \rangle^2 - 2a_{ij}\left \langle \beta, W_{n,ij}^{(2, k+1)} \right \rangle +(W_{n,ij})^{2}\right).
$$
Let $$T = (1-\delta_W) \sum_{r=1}^k b_r (1-\delta_W)^r.$$
Then with probability at least $1-5/n - \exp(-\delta_W \rho_n(n-1)/3 ) - \delta$, we have
  \begin{equation}
        \sup_{\beta \in \mathcal{F}} |K_n(\beta) - \E[K_n(\beta)]| = \tilde{O}\left( \frac{\rho_n^2(T+1)^2}{\sqrt{n}} \right),
    \end{equation}
    where $\tilde{O}$ hides logarithmic factors.
\end{lemma}

\begin{proof}[Proof of \cref{uniform-concentration}]

In the proof of this lemma, we are inherently conditioning on all of the events that the proof of \cref{prop:estimators-to-empiricalmoments} conditions on. Specifically, we are conditioning on the event that $$\sup_{i \neq j \in [n]} \left| \hat{q}_{i,j}^{(k)} - {W}_{n, i, j}^{(k)}\right| \leq \frac{\rho_n^{k-1}}{\sqrt{n-1}} \log(n)^k \left[ 3a_k\sqrt{\rho_n} + \frac{96 a_{k-1}}{\sqrt{d_n}} \right],$$ and that $$\max_{i \in [n]} \frac{1}{n-1} \sum_{\stackrel{j \leq n}{j \neq i}} a_{ij} < 2\rho_n \Rightarrow \frac{2}{n(n-1)} \sum_{i < j} a_{ij} < 2 \rho_n.$$ Firstly, we note that $|\langle \beta, W^{(2,k+1)}(x,y) \rangle| \leq \rho_n T.$ See the proof of \cref{lemma:bound-on-S2} for a more detailed calculation.

We use an $\epsilon$-net argument to obtain the desired uniform concentration result over the entire space. We first bound the cardinality of an $\epsilon$-net needed to cover $S$, where the covering sets are $\epsilon$-balls in the $L^1$ norm in $\mathbb{R}^{m_W}$. We then establish a high-probability bound for the quantity $|K_n(\beta_0) - \E[K_n(\beta_0)]|$ using a concentration inequality for U-statistics, for a fixed $\beta_0$. Then, the continuity of $K_n(\beta)$ will yield a bound for  $|K_n(\beta) - \E[K_n(\beta)]|$ for all $\beta$ in the same $\epsilon$-ball as $\beta_0$. We then take a union bound over all balls in the $\epsilon$-net to arrive at the conclusion.

We note that a hypercube with side length $2\epsilon/k$ centered at some $x$ is contained in the $L^1$ $\epsilon$-ball centered at $x$, so bounding the cardinality of a covering with hypercubes of side length $2 \epsilon/k$ would also bound the cardinality of a covering with $L^1$ $\epsilon$-balls. To determine this cardinality, we can simply consider the construction of tiling $S$ (which is a hyper-rectangle) with hypercubes simply by packing the cubes side-to-side. Hence, we obtain an $\epsilon$-net of size bounded by \begin{equation}
        \prod_{i=1}^{k} 2 \frac{b_i}{\rho_n^i} \frac{k}{2\epsilon} =  \frac{1}{\rho_n^{k(k+1)/2}} \left( \frac{k}{\epsilon} \right)^{k} \prod_{i=1}^{k} b_i. 
    \end{equation}

    Now, we bound $|K_n(\beta) - \E[K_n(\beta)]|$. In this goal we define $$K_n^1(\beta):=\frac{2}{n(n-1)} \sum_{i < j}  \left(\left \langle \beta, W_{n,ij}^{(2, k+1)}  \right \rangle^2-2W_{n,i,j}\left \langle \beta, W_{n,ij}^{(2, k+1)} \right \rangle+(W_{n,ij})^{2}\right)$$ 
    and 
    \begin{align*} K_n^2(\beta):&=-\frac{4}{n(n-1)} \sum_{i < j} \left(   a_{ij}\left \langle \beta, W_{n,ij}^{(2, k+1)} \right \rangle-W_{n,i,j}\left \langle \beta, W_{n,ij}^{(2, k+1)} \right \rangle \right)
    \end{align*} 
    
    

    
    We remark that $K_n(\beta)=K_n^1(\beta)+K_n^2(\beta)$. Using the triangle inequality, we notice that it is enough to show concentration of $K_n^1(\beta)$ and $K_n^2(\beta)$ around their respective expectations. 
    
    We first remark that $\mathbb{E}(K_n^2(\beta))=0$ and show concentration $K_n^2(\beta)$ around its expectation. In this goal, notice that conditional on $(\omega_i)$, the random variables $\big(a_{i,j}\big)$ are i.i.d Bernoulli random variables. Moreover we notice that conditionally on the features $(\omega_i)$ we have that $P( (a_{i,j})) = - \frac{4}{n(n-1)} \sum_{i<j}a_{i,j}\left \langle \beta, W_{n,ij}^{(2, k+1)} \right \rangle$ is a polynomial of degree one of the Bernoulli random variables $(a_{i,j})$. Hence, we use \cref{luna2}. We note that $E[P((a_{i,j}))] \leq 2 T \rho_n^2,$ and the first derivative with respect to $a_{1,2}$ is $\frac{\partial}{\partial a_{1,2}} P((a_{i,j})) = -\frac{4}{n(n-1)} \left \langle \beta, W_{n, 1, 2}^{(2,k+1)} \right \rangle \leq \frac{4\rho_n T}{n(n-1)}.$  Then,  for all $\lambda>0$ we have 

Think it should be \begin{align}
    P\left(\frac{4}{n(n-1)} \left|\sum_{i<j}^n a_{i,j}\left \langle \beta, W_{n,ij}^{(2, k+1)} \right \rangle-\mathbb{E}\left[ a_{i,j}\left \langle \beta, W_{n,ij}^{(2, k+1)} \right \rangle|(\omega_i)\right]\right|\ge   \frac{2 \sqrt{2}a_1 \lambda}{\sqrt{n(n-1)}}  T\rho_n^{3/2} \right)\le 2G \cdot \rm{exp}\big(-\lambda\big),
\end{align} where $G$ is some constant from \cref{luna2}. Moreover we notice that $\mathbb{E}\left[ a_{i,j}\left \langle \beta, W_{n,ij}^{(2, k+1)} \right \rangle|(\omega_i) \right]=W_{n,i,j}\left \langle \beta, W_{n,ij}^{(2, k+1)} \right \rangle$. Therefore, we obtain that 

\begin{align}
    P\left(\big|K_n^2(\beta)\big|\ge  \frac{2 \sqrt{2}a_1 \lambda}{\sqrt{n(n-1)}} \rho_n^{3/2} T \right)\le 2G \cdot \rm{exp}\big(-\lambda\big)
\end{align}
Then, we derive a concentration bound for $K_n^1(\beta)$, for a fixed vector $\beta \in S.$ The randomness in $K^1_n(\beta)$ term comes from the latent features $\omega_i$ and we observe that it is a U-Statistic with two variables. To bound the desired quantity, we use the following

    \begin{lemma}[Equation (5.7) from \cite{hoeffding_ustatistics}]
        Let $X_1, X_2, \dots, X_N$ be independent random variables. For $r \leq n$, consider a random variable of the form $$U = \frac{1}{n(n-1)\dots(n-r+1)} \sum_{i_1 \neq i_2 \neq \dots \neq i_r} g(X_{i_1}, \dots, X_{i_r}).$$ Then if $a \leq g(x_1, x_2, \dots, x_r) \leq b$, it follows that $$\P(|U - \E[U]| \ge t) \leq e^{-2 \lfloor n/r \rfloor t^2/(b-a)^2}$$
    \end{lemma}
    To use this quantity, we first bound $K_n(\beta).$ Using \cref{bound:T-bounding-cproduct}, we have
\begin{align}
    &\left| \left \langle \beta, W_{n,ij}^{(2, k+1)}  \right \rangle^2 -2W_{n,i,j} \left \langle \beta, W_{n,ij}^{(2, k+1)}  \right \rangle+ (W_{n,ij})^{2}   \right|\\
    &\leq \rho_n^2 (T^2 +2T +1) \\
    &= \rho_n^2 (T+1)^2
\end{align}

Hence, for a fixed $\beta_0$, we have that 
\begin{equation}
\P \left( |K^1_n(\beta_0) -  \E[K^1_n(\beta_0)]| \ge t \right) \leq 2 \exp \left( \frac{-\lfloor \frac{n}{2} \rfloor t^2}{2\rho_n^4 (T+1)^4} \right).
\end{equation} 
We now use continuity to argue that $|K_n(\beta) - \E[K_n(\beta)]|$ is bounded for all $\beta$ in the $\epsilon$-ball containing $\beta_0$. we derive a bound on $|(K_n(\beta_1) - \E[K_n(\beta_1)]) - (K_n(\beta_2) - \E[K_n(\beta_2)]|$ for when $\|\beta_1 - \beta_2\|_1 \leq \epsilon.$ We can use the triangle inequality and bound $|K_n(\beta_1) - K_n(\beta_2)|$ and $|\E[K_n(\beta_1)] - \E[K_n(\beta_2)]|$ separately. 

We can write 
\begin{align}
    |\E[K_n(\beta_1)] - \E[K_n(\beta_2)] |&= \left| \E \left[ \left( \left \langle \beta_1, W_{n,ij}^{(2, k+1)} \right \rangle - \rho_n W_{ij} \right)^2 \right] - \E \left[ \left( \left \langle \beta_2, W_{n,ij}^{(2, k+1)} \right \rangle - \rho_n W_{ij} \right)^2 \right] \right| \\
    &= \left |\E \left[ \left \langle \beta_1 - \beta_2, W_{n,ij}^{(2, k+1)} \right \rangle   \left( \left \langle \beta_1 + \beta_2, W_{n,ij}^{(2, k+1)} \right \rangle -  2\rho_n W_{ij} \right) \right] \right| \\
    &\leq \rho_n^3 \epsilon \cdot 2(T+1)
\end{align}
where the $\rho_n^2 \epsilon$ term is from the first term: the $L^1$ norm of $\beta_1 - \beta_2$ is bounded by $\epsilon,$ and each entry in the vector $W_{n,ij}^{(2, k+1)}$ is bounded by $\rho_n^2.$ The factor of $\rho_n(T+1)$ is using \cref{bound:T-bounding-cproduct}.


In a similar way, we can bound $|K_n(\beta_1) - K_n(\beta_2)|$. We first bound the quantity
\begin{align}
&\left| \frac{2}{n(n-1)} \sum_{i < j} \left( \left \langle \beta_1, W_{n,ij}^{(2, k+1)}  \right \rangle^2 -  \left \langle \beta_2, W_{n,ij}^{(2, k+1)}  \right \rangle^2 \right) \right| \\
&= \left| \frac{2}{n(n-1)} \sum_{i < j}  \left \langle \beta_1 - \beta_2, W_{n,ij}^{(2, k+1)}  \right \rangle \cdot \left \langle \beta_1+ \beta_2, W_{n,ij}^{(2, k+1)} \right \rangle \right| \\
&\leq \rho_n^3 \epsilon \cdot 2T
\end{align}
Then we can bound 
\begin{align}
&\left| \frac{2}{n(n-1)} \sum_{i < j} \left( 2a_{ij}\left \langle \beta_1, W_{n,ij}^{(2, k+1)} \right \rangle - 2a_{ij}\left \langle \beta_2, W_{n,ij}^{(2, k+1)} \right \rangle   \right)\right| \\
&\leq \frac{2}{n(n-1)} \sum_{i < j} \left| 2a_{ij}\left \langle \beta_1 -  \beta_2, W_{n,ij}^{(2, k+1)} \right \rangle   \right| \\
&\leq  \max_{i < j} \left| \left \langle \beta_1 -  \beta_2, W_{n,ij}^{(2, k+1)} \right \rangle \right| \cdot \frac{2}{n(n-1)} \sum_{i < j} \left| 2a_{ij}   \right| \\
&\leq 4 \epsilon \rho_n^3,
\end{align}
where the last inequality comes from conditioning on the event mentioned at the beginning of the proof.
From here, we can see that 
\begin{align}
    | (K_n(\beta_1) - \E[K_n(\beta_1)]) - ( K_n(\beta_2) - \E[K_n(\beta_2)])|  &\leq  \rho_n^3 \epsilon (4T + 6)
\end{align}

This implies that
\begin{align}&
    \P \left( \sup_{\beta \in S_\epsilon} |K_n(\beta) -  \E[K_n(\beta)]| \ge t + \frac{2 \sqrt{2}a_1 \lambda}{\sqrt{n(n-1)}}  T \rho_n^{3/2}+ \rho_n^3 \epsilon (4T+6)  \right) \\
    &\leq 2 \cdot \text{card}(S_\epsilon) \exp \left( \frac{-\lfloor \frac{n}{2} \rfloor t^2}{2\rho_n^4 (T+1)^4} \right) - 2G \exp(-\lambda)
\end{align}
Choosing 
$$t = \rho_n^2 (T+1)^2 \sqrt{\frac{1}{4 \lfloor \frac{n}{2} \rfloor} \log\left( \frac{4 \cdot \rm{card}(S_\epsilon)}{\delta} \right)}, \quad \text{and} \quad   \lambda=\log\left( \frac{4 G}{\delta} \right),$$ we have that with probability $1-\delta,$
$$\sup_{\beta \in \mathcal{F}} |K_n(\beta) - \E[K_n(\beta)]| \leq \rho_n^2 (T+1)^2 \sqrt{\frac{1}{4 \lfloor \frac{n}{2} \rfloor} \log\left( \frac{4 \cdot \rm{card}(S_\epsilon)}{\delta} \right)} + 
\frac{2 \sqrt{2}a_1 \lambda}{\sqrt{n(n-1)}}  T\rho_n^{3/2} \log(4G/\delta) + \rho_n^3 \epsilon(4T+6).$$
Choose $\epsilon = \frac{1}{\sqrt{n}}.$ Recall that \begin{gather*}
\text{card}(S_\epsilon) \leq  \frac{1}{\rho_n^{k(k+1)/2}} \left( \frac{k}{\epsilon} \right)^{k} \prod_{i=1}^{k} b_i \\
\Rightarrow \log(\text{card}(S_\epsilon)) \leq k \log(k) + k \log(1/\epsilon) + \frac{k(k+1)}{2} \log(1/\rho_n) + \log\left( \prod_{i=1}^{k} b_i \right).
\end{gather*}
Then, we conclude that with probability at least $1-\delta$, 
$$\sup_{\beta \in \mathcal{F}} |K_n(\beta) - \E[K_n(\beta)]| \leq \tilde{O}\left( \frac{\rho_n^2 (T+1)^2}{\sqrt{n}} \right),$$ where the Big-O constant depends on $\sqrt{\log(1/\delta)}.$

\end{proof}

\begin{lemma}
\label{lemma:bound-on-S2}
Let $\mathcal{F}= \prod_{i=1}^k [-a_i, a_i]$ be a subset of $\mathbb{R}^k,$ where $a_i = b_i/\rho_n^i$ for some $b_i > 0.$ Define $D = \sum_{i=1}^k |b_i|,$ and for $\beta \in \mathcal{F},$ define
$$ S_2(\beta) := \frac{2}{n(n-1)} \sum_{i < j} \left[ 2 \langle \beta, \hat{q}^{(2, k+1)}_{ij}-W_{n,ij}^{(2, k+1)}\rangle(\langle \beta, W_{n,ij}^{(2, k+1)} \rangle - a_{ij} )  \right].$$ Then with probability at least $1 - 5/n - n \cdot \rm{exp}(-\delta_W \rho_n(n-1)/3),$
\begin{equation}
    S_2(c) \leq 2D\rho_n\cdot r(n, d_n, k)\cdot(T+2 ),
\end{equation}
where $T = (1-\delta_W) \sum_{r=1}^k b_r (1-\delta_W)^r.$
\end{lemma}

\begin{proof}[Proof of \cref{lemma:bound-on-S2}]
In the proof of this lemma, we are inherently conditioning on all of the events that the proof of \cref{prop:estimators-to-empiricalmoments} conditions on. Specifically, we are conditioning on the event that $$\sup_{i \neq j \in [n]} \left| \hat{q}_{i,j}^{(k)} - {W}_{n, i, j}^{(k)}\right| \leq \frac{\rho_n^{k-1}}{\sqrt{n-1}} \log(n)^k \left[ 3a_k\sqrt{\rho_n} + \frac{96 a_{k-1}}{\sqrt{d_n}} \right],$$ and that $$\max_{i \in [n]} \frac{1}{n-1} \sum_{\stackrel{j \leq n}{j \neq i}} a_{ij} < 2\rho_n \Rightarrow \frac{2}{n(n-1)} \sum_{i < j} a_{ij} < 2 \rho_n.$$ 

We first consider $S_2(\beta)$ for some arbitrary $\beta \in S$. 
\begin{align*}
S_2(\beta) &\leq \frac{2}{n(n-1)}\left|  \sum_{i < j}^n  2 \langle \beta, \hat{q}^{(2, k+1)}_{ij}-W_{n,ij}^{(2, k+1)}\rangle \Big(\langle \beta,W_{n,ij}^{(2, k+1)} \rangle - a_{ij}\Big)  \right|  \\
& \leq 2\cdot \frac{2}{n(n-1)}  \sum_{i < j}^n  \left|  \langle \beta, \hat{q}^{(2, k+1)}_{ij}-W_{n,ij}^{(2, k+1)}\rangle \Big(\langle \beta,W_{n,ij}^{(2, k+1)} \rangle - a_{ij}\Big)  \right| \\
&= 2\cdot \frac{2}{n(n-1)}  \sum_{i < j}^n  \left|  \langle \beta, \hat{q}^{(2, k+1)}_{ij}-W_{n,ij}^{(2, k+1)}\rangle \right| \left| \Big(\langle \beta,W_{n,ij}^{(2, k+1)} \rangle - a_{ij}\Big)  \right|  \\
&\leq 2D \cdot r(n, d_n, k)  \cdot \frac{2}{n(n-1)}  \sum_{i < j}^n  \left| \langle \beta,W_{n,ij}^{(2, k+1)} \rangle - a_{ij} \right| \\
&\leq 2D \cdot r(n, d_n, k)  \cdot \frac{2}{n(n-1)}  \sum_{i < j}^n \Big(  \left|\langle \beta,W_{n,ij}^{(2, k+1)} \rangle \right| + a_{ij}\Big)\\
&\leq 2D \cdot r(n, d_n, k) \cdot \frac{2}{n(n-1)} \left( \sum_{i < j}^n \left|\langle \beta,W_{n,ij}^{(2, k+1)} \rangle \right| + \sum_{i < j}^n a_{ij} \right) 
\end{align*}

We write 
\begin{align}
    |\langle \beta,W_{n,ij}^{(2, k+1)} \rangle| &\leq \left| \sum_{r=1}^{m_W} \beta_r \cdot \rho_n^{r+1} W_{i,j}^{(r+1)} \right| \nonumber \\
    &\leq \left| \sum_{r=1}^{m_W} \beta_r \cdot \rho_n^{r+1} (1-\delta_W)^{r+1} \right| \nonumber \\
    &\leq \sum_{r=1}^{m_W} \frac{b_r}{\rho_n^r}  \cdot \rho_n^{r+1} (1-\delta_W)^{r+1} \nonumber \\
    &\leq \rho_n \underbrace{(1-\delta_W)\sum_{r=1}^{m_W} b_r (1-\delta_W)^r}_{T} \label{bound:T-bounding-cproduct} 
\end{align}
This yields the bound 
\begin{align*}
&S_2(\beta) \leq 2D\rho_n \cdot r(n, d_n, k) \cdot (T+ 2),
\end{align*}
where $T = (1-\delta_W)\sum_{r=1}^{m_W} b_r (1-\delta_W)^r.$
\end{proof}

\begin{lemma}
\label{lemma:bound-on-S3}
Let $\mathcal{F}= \prod_{i=1}^k [-a_i, a_i]$ be a subset of $\mathbb{R}^k,$ where $a_i = b_i/\rho_n^i$ for some $b_i > 0.$ Define $D = \sum_{i=1}^k |b_i|,$ and for $\beta \in \mathcal{F},$ define
$$
S_3(\beta) := \frac{2}{n(n-1)} \sum_{i < j} \langle \beta,  \hat{q}^{(2, k+1)}_{ij}-W_{n,ij}^{(2, k+1)}\rangle^2. 
$$
Then with probability at least $1 - 5/n - n \cdot \rm{exp}(-\delta_W \rho_n(n-1)/3),$
$$
S_3(\beta) \leq D^2 \cdot r(n, d_n, k)^2.
$$
\end{lemma}

\begin{proof}[Proof of \cref{lemma:bound-on-S3}]
We bound this term as in the proof of \cref{lemma:bound-on-S2}. It directly follows that
    $$\frac{2}{n(n-1)} \sum_{i < j}  \langle \beta,  \hat{q}^{(2, k+1)}_{ij}-W_{n,ij}^{(2, k+1)}\rangle^2 \leq D^2 \cdot r(n, d_n, k)^2$$
\end{proof}

\section{Proof of \cref{prop:preserve-rank}}

In this section, we state a formal version of \cref{prop:preserve-rank} and provide the proof. 

\begin{proposition}[\cref{prop:preserve-rank}, Formal]
\label{prop:preserve-rank-formal}
Consider a $k$-community symmetric stochastic block model (see \cref{sec:sbm} for the definition) with parameters $p > q$ and sparsity factor $\rho_n,$ which has eigenvalues $\mu_1 = \frac{p+(k-1)q}{k} > \frac{p-q}{k} = \mu_2.$ Fix some $L \ge 1$ and define $\mathcal{F}= \{ \beta \in \mathbb{R}^{L+1} | \text{ } ||\beta||_{L^1} \leq (\mu_1 \rho_n)^{-1} \}$. 

Produce probability estimators $\hat{p}_{i,j}$ for the probability of an edge between vertices $i$ and $j$ using \cref{algo1} and \cref{alg:compute_regression_coefs}. Let $r(n, d_n, L+1)$ be defined as in \cref{prop:convergence_of_c}. Suppose that $n, d_n$ satisfy
\begin{gather*}
    \frac{4 \mu_2}{k}\frac{r(n, d_n, L+1)}{\rho_n} + \frac{4}{k^2 \mu_1} \left( \frac{r(n, d_n, L+1)}{\rho_n}\right)^2 + \frac{4 }{(k-1)}\frac{r(n, d_n, L+1)}{\rho_n} (T+2) \\
    + \frac{2}{(k-1) \mu_1}\left( \frac{r(n, d_n, L+1)}{\rho_n}\right)^2 + \frac{A \mu_1}{\sqrt{n}} \frac{\log(n)}{k-1} \leq \mu_2^3,
\end{gather*}
where $A$ is a constant that depends on $\sqrt{\log(1/\delta)}$, for some positive constant $\delta > 0.$ 


Let $S_{in} = \{(i, j) | \text{$i,j$ belong to the same community} \}$ and $S_{out} = \{ (k, ell) | k, l \text{$k, \ell$ belong to different communities} \}.$
Then, with probability at least $1 - 5/n - \rm{exp}(-\delta_W \rho_n \cdot (n-1)) - \delta,$ the following event occurs:
$$ \left \{ \min_{(i,j) \in S_{in}} \hat{p}_{i,j} > \max_{(k,\ell) \in S_{out}} \hat{p}_{k, \ell} \right\}$$
\end{proposition}

\begin{proof}[Proof of \cref{prop:preserve-rank}]

Consider a $k$-community symmetric SBM with connection matrix $P$, where $P_{i,i} = p$ for all $i \in k$ and $P_{i, j} = q$ for all $i \neq j$. We first write the eigen-decomposition of this matrix; there are two eigenvalues and we write an orthogonal basis for their eigenspaces. $$\lambda_1 = p + q(k-1): \left \{\begin{pmatrix} 1 \\ 1 \\1 \\ \vdots \\ 1 \end{pmatrix} \right \}, \lambda_2 = p-q: \left \{\begin{pmatrix} 1 \\ -1 \\ 0 \\ 0  \\ \vdots \\ 0\end{pmatrix}, \begin{pmatrix} 1 \\ 1 \\-2  \\ 0\\ \vdots \\ 0\end{pmatrix}, \begin{pmatrix} 1 \\ 1 \\1  \\ -3 \\ \vdots \\ 0\end{pmatrix}, \dots, \begin{pmatrix} 1 \\ 1 \\1  \\ 1 \\ \vdots \\ -(k-1)\end{pmatrix} \right \}.$$

According to \cref{lemma:sbm}, the eigenvalues of the corresponding graphon $W$ representation are $\frac{p-q}{k}$ and $\frac{p+q(k-1)}{k}.$ Call the eigenvalues $\mu_1$ and $\mu_2$, and let $\mu_1 > \mu_2$ without loss of generality. The eigenfunctions $\phi_i$ of $W$ are also given by \cref{lemma:sbm}, and are essentially scaled versions of the above eigenvectors. For the remainder of this proof, we assume that the graph was generated from $\rho_n W$ for some sparsity factor $\rho_n$.

Define $p_n = \rho_n p$, $q_n = \rho_n q.$ Using that $W_n(x,y) = \sum_r (\rho_n \mu_r) \phi_r(x) \phi_r(y),$ we see that $$p_n = \rho_n\mu_1 + \rho_n\mu_2 \underbrace{\left[ \sum_{r=1}^{k-1} \frac{k}{r(r+1)} \right]}_{C_1}, \quad q_n = \rho_n\mu_1 + \rho_n\mu_2 \underbrace{\left[ -\frac{k}{2} + \sum_{r=2}^{k-1} \frac{k}{r(r+1)} \right]}_{C_2}.$$

Now suppose that a graph $G = (V, E) = ([n], E)$ is generated from $W_n.$ We demonstrate that under the conditions mentioned in \cref{prop:preserve-rank}, \cref{algo1} and \cref{alg:compute_regression_coefs} results in probability predictions $\hat{p}_{i,j}$ so that $\hat{p}_{i,j} > \hat{p}_{k, \ell}$ when $c_i = c_j$ and $c_k \leq c_\ell$. In other words, the probability predictions for all of the intra (within) community edges are higher than the probability predictions for all of the inter (across) community edges.

As in the proposition statement, define $\mathcal{F}= \{ \beta \in \mathbb{R}^{L+1} | \text{ } ||\beta||_{L^1} \leq (\mu_1 \rho_n)^{-1} \}$, and suppose that $L \ge 1$ is the number of layers that are computed. In other words, LG-GNN computes the set of embeddings $\{ \lambda_i^0, \lambda_i^1, \dots, \lambda_i^L\}$ for all $i$, and for each pair of vertices $i,j$, it computes moment estimators $\left \{ \hat{q}_{i,j}^{(2)}, \hat{q}_{i,j}^{(3)}, \dots, \hat{q}_{i,j}^{(L+2)} \right\} .$

Then, in \cref{alg:compute_regression_coefs}, we solve the optimization problem $$\hat{\beta}^{n,L+1} = \argmin_{\beta \in \mathcal{F}} \sum_{i < j} \left( a_{ij} - \left\langle \beta, \hat{q}_{i,j}^{(2, 3, \dots, L+2)} \right\rangle \right)^2.$$ For $i = 1,2$ and any $\beta \in \mathbb{R}^{L+1}$, define $\hat{\mu}_{n,i}(\beta) = \sum_{r=1}^m \beta_r (\rho_n \mu_i)^{r+1},$ where the subscript $n$ makes implicit that there is dependence on $\rho_n.$ Defining $$R(\beta) = \E \left[ \left( \left \langle \beta,  W_n^{(2,L+2)}(x,y)   \right \rangle - W_n(x,y) \right) ^2 \right],$$ we note that for any fixed $\beta \in \mathcal{F}$, we have that 
\begin{align*}
    R(\beta) = (\rho_n \mu_1 - \hat{\mu}_{n, 1}(\beta))^2 + (k-1) \cdot (\rho_n \mu_2 - \hat{\mu}_{n, 2}(\beta))^2.
\end{align*}
Let $\omega_i, \omega_j$ be the latent features of two vertices that both correspond to being in community 1, and let $\omega_k, \omega_\ell$ be the latent features of two vertices that correspond to being in communities 1 and 2, respectively. Now, suppose that for the edges $(i,j)$ and $(k, \ell)$, LG-GNN assigns them predicted probabilities $\hat{p}_{i,j} = \left \langle \hat{\beta}^{n,L+1}, \hat{q}_{i,j}^{(2, \dots,L+2)}    \right \rangle,$ $\hat{p}_{k, \ell} = \left \langle \hat{\beta}^{n,L+1}, \hat{q}_{k, \ell}^{(2, \dots,L+2)}    \right \rangle,$ respectively, and suppose that $\hat{p}_{k, \ell} > \hat{p}_{i, j}.$ Consider
\begin{align*}
\hat{p}_{i,j} = \left \langle \hat{\beta}^{n,L+1}, \hat{q}_{i,j}^{(2, \dots,L+2)}  \right \rangle &= \left \langle \hat{\beta}^{n,L+1}, W_{n,i, j}^{(2, \dots, L+2)} \right \rangle + \left \langle \hat{\beta}^{n,L+1}, \hat{q}_{i,j}^{(2, \dots,L+2)} - W_{n,i, j}^{(2, \dots, L+2)}  \right \rangle \\
&= \hat{\mu}_{n,1}(\hat{\beta}^{n,L+1}) + C_1 \hat{\mu}_{n, 2}(\hat{\beta}^{n,L+1}) + \left \langle \hat{\beta}^{n,L+1}, \hat{q}_{i,j}^{(2, \dots,L+2)} - W_{n,i, j}^{(2, \dots, L+2)}  \right \rangle,
\end{align*}
where we simplified the first term this way in the last equality by using the form of the eigenvectors $\phi_r$, and noting that $\omega_i, \omega_j$ both correspond to vertices in community 1. In a similar way, we have that 
\begin{align*}
\hat{p}_{k,\ell} = \left \langle \hat{\beta}^{n,L+1}, \hat{q}_{k,\ell}^{(2, \dots,L+2)}  \right \rangle &= \left \langle \hat{\beta}^{n,L+1}, W_{n,k, \ell}^{(2, \dots, L+2)} \right \rangle + \left \langle \hat{\beta}^{n,L+1}, \hat{q}_{k,\ell}^{(2, \dots,L+2)} - W_{n,k, \ell}^{(2, \dots, L+2)}  \right \rangle \\
&= \hat{\mu}_{n, 1}(\hat{\beta}^{n,L+1}) + C_2 \hat{\mu}_{n,2}(\hat{\beta}^{n,L+1}) + \left \langle \hat{\beta}^{n,L+1}, \hat{q}_{k,\ell}^{(2, \dots,L+2)} - W_{n,k, \ell}^{(2, \dots, L+2)}  \right \rangle,
\end{align*}
Hence, if $\hat{p}_{k, \ell} > \hat{p}_{i, j}$, then noting that $C_1 - C_2 = k,$ we have
\begin{align*}
    k \hat{\mu}_2(\hat{\beta}^{n,L+1}) < \left \langle \hat{\beta}^{n,L+1}, \hat{q}_{k,\ell}^{(2, \dots,L+2)} - W_{n,k, \ell}^{(2, \dots, L+2)}  \right \rangle - \left \langle \hat{\beta}^{n,L+1}, \hat{q}_{i,j}^{(2, \dots,L+2)} - W_{n,i, j}^{(2, \dots, L+2)}  \right \rangle 
\end{align*}
We note that \cref{prop:graph_concentration_sparse} states that with probability at least $1 - 5/n - n \cdot \rm{exp}(-\delta_W \rho_n(n-1)/3) - \delta,$ we have that for all $2 \leq k \leq L+2$,
$$\left| \hat{q}_{i,j}^{(k)} - {W}_{n, i, j}^{(k)} \right| 
    \leq \frac{\rho_n^{k-1}}{\sqrt{n-1}} \log(n)^k \left[ 3 a_k \sqrt{\rho_n} + \frac{96 a_{k-1}}{\sqrt{d_n}} \right],$$ for some constants $a_k$, and also that the conclusion from \cref{uniform-concentration} holds. We will be conditioning on these events for the remainder of the proof. We also note that
\begin{align*}
    r(n, d_n, L+2) &:= \max_{2 \leq k\leq  L+2} \rho_n^{-(k-1)} \left( \frac{\rho_n^{k-1}}{\sqrt{n-1}} \log(n)^k \left[ 3 a_k \sqrt{\rho_n} + \frac{96 a_{k-1}}{\sqrt{d_n}} \right] \right) \\
     &= \max_{2 \leq k \leq L+2} \frac{\log(n)^k}{\sqrt{n-1}}  \left[ 3 a_k \sqrt{\rho_n} + \frac{96 a_{k-1}}{\sqrt{d_n}} \right] \\
     &= o(\rho_n) \quad \text{if $\rho_n \gg \rho_n^{2(L+2)}{n}$}
\end{align*}


Using these definitions, noting that $\left\|\hat{\beta}^{n,L+1} \right\|_{L^1} \leq \frac{1}{\mu_1 \rho_n},$ we have that $$\hat{\mu}_{n, 2}(\hat{\beta}^{n,L+1}) < \frac{2}{k \mu_1} r(n, d_n, L+1). $$
Define $\beta_0 = (1/(\mu_1 \rho_n), 0, 0, \dots, 0) \in \mathbb{R}^m$ to have $1/\mu_1$ as the first component, and 0 everywhere else. Now, we note that \cref{lemma:empirical-minus-risk} states, with the same probability above, that $$R(\beta_0) - R(\hat{\beta}^{n, L+1}) = R_n(\beta_0) - R_n(\hat{\beta}^{n,L+1}) + P,$$ where $$|P| \leq \frac{4}{\mu_1} \rho_n \cdot r(n, d_n, L+1) (T+2) + \frac{2}{\mu_1^2} \cdot r(n, d_n, L+1)^2 + A \frac{\rho_n^2}{\sqrt{n}} \log(n),$$ where $T = \frac{p^2}{\mu_1}$ and $A$ is some constant that depends on $\sqrt{\log(1/\delta)}$. We also note that $R_n(\beta_0) - R_n(\hat{\beta}^{n,L+1}) \ge 0$ because $\hat{\beta}^{n, L+1}$ is the minimizer of the empirical risk. This implies that $R(\beta_0) \ge R(\hat{\beta}^{n,L+1}) + P.$ So, noting that $R(\beta_0) = \rho_n^2 (k-1) \left( \mu_2 - \frac{\mu_2^2}{\mu_1} \right)^2$, and noting that $\frac{\mu_2}{\mu_1} < 1,$
\begin{gather*}
    \rho_n^2 (k-1) \left(  \mu_2 - \frac{\mu_2^2}{\mu_1} \right)^2 \ge ( \rho_n\mu_1 - \hat{\mu}_{n, 1}(\hat{\beta}^{n,L+1}))^2 + (k-1) ( \rho_n\mu_2 - \hat{\mu}_{n, 2}(\hat{\beta}^{n,L+1}))^2 + P \\
    \Rightarrow \rho_n^2 (k-1) \left( \mu_2 - \frac{\mu_2^2}{\mu_1} \right)^2 \ge  (k-1) (\rho_n\mu_2 - \hat{\mu}_{n, 2}(\hat{\beta}^{n,L+1}))^2 + P \\
    \Rightarrow 0 <  \rho_n^2\frac{\mu_2^3}{\mu_1}\left[ 2 - \frac{\mu_2}{\mu_1} \right] \leq  (2 \rho_n \mu_2 \hat{\mu}_{n, 2}(\hat{\beta}^{n,L+1}) + \hat{\mu}_{n, 2}(\hat{\beta}^{n,L+1})^2)- \frac{P}{k-1}.
\end{gather*}
However, this is a contradiction when $|(2 \rho_n \mu_2 \hat{\mu}_{n, 2}(\hat{\beta}^{n,L+1}) - \hat{\mu}_{n, 2}(\hat{\beta}^{n,L+1})^2)- \frac{P}{k-1}| < \rho_n^2\frac{\mu_2^3}{\mu_1}\left[ 2 - \frac{\mu_2}{\mu_1} \right].$ Consider the bounds 
$$
|P| \leq \frac{4}{\mu_1} \rho_n \cdot r(n, d_n, L+1) (T+2) + \frac{2}{\mu_1^2} \cdot r(n, d_n, L+1)^2 + A \frac{\rho_n^2}{\sqrt{n}} \log(n)
$$
Hence, when 
\begin{gather*}
    \frac{4 \rho_n \mu_2}{k \mu_1}r(n, d_n, L+1) + \frac{4}{k^2 \mu_1^2} r(n, d_n, L+1)^2 + \frac{4\rho_n }{(k-1)\mu_1}r(n, d_n, L+1) (T+2) \\
    + \frac{2}{(k-1) \mu_1^2}r(n, d_n, L+1)^2 + A \frac{\rho_n^2}{\sqrt{n}} \frac{\log(n)}{k-1} \leq \rho_n^2 \frac{\mu_2^3}{\mu_1},
\end{gather*}
the result follows. Dividing both sides by $\rho_n^2$, and multiplying by $\mu_1,$ the above is equivalent to
\begin{gather*}
    \frac{4 \mu_2}{k}\frac{r(n, d_n, L+1)}{\rho_n} + \frac{4}{k^2 \mu_1} \left( \frac{r(n, d_n, L+1)}{\rho_n}\right)^2 + \frac{4 }{(k-1)}\frac{r(n, d_n, L+1)}{\rho_n} (T+2) \\
    + \frac{2}{(k-1) \mu_1}\left( \frac{r(n, d_n, L+1)}{\rho_n}\right)^2 + \frac{A \mu_1}{\sqrt{n}} \frac{\log(n)}{k-1} \leq \mu_2^3.
\end{gather*}
The result follows.

\end{proof}

\section{Proof of \cref{prop:idenfiability}}

Note that in this proof, we assume that the sparisity factor $\rho_n = 1$. Consider a 2-community stochastic block model (see \cref{sec:sbm} for more details) parameterized by the matrix $\begin{pmatrix} p & r \\ r & q \end{pmatrix}.$ The eigenvalues and eigenvectors are given by
$$
\lambda_1 = \frac{1}{2} \left( p+q+A \right), \quad v_1 = \begin{pmatrix}
\frac{p-q+A}{2r} \\1  
\end{pmatrix}, \quad \lambda_2 = \frac{1}{2} \left( p+q-A \right), \quad v_2 = \begin{pmatrix}
\frac{p-q-A}{2r} \\1  
\end{pmatrix},
$$
where $A = \sqrt{(p-q)^2 + 4r^2}.$ Then, recall that \cref{lemma:sbm} states that the eigenvalues of the graphon representation $W$ of this SBM has eigenvalues $\mu_i = \frac{1}{2} \lambda_i.$ We also use the eigenfunctions $\phi_i$ for $W$ as written in \cref{lemma:sbm}. We recall that \cref{lemma:expectation-dotproducts} states that for all $L \ge 0$, we have
$$
    \E\left[ \langle \lambda_i^{L}, \lambda_j^{L} \rangle | A, (\omega_i)_{i=1}^n \right] = \sum_{q_1=0}^{L} \sum_{q_2=0}^{L} \binom{L}{q_1} \binom{L}{q_2} \hat{W}_{n, i, j}^{(q_1+q_2+2)} 
$$
Then, the proof of \cref{prop:estimators-to-empiricalmoments} implies that for all $i,j$,
$$
    \Rightarrow \langle \lambda_i^{L}, \lambda_j^{L} \rangle \stackrel{p}{\to} \sum_{q_1=0}^{L} \sum_{q_2=0}^{L} \binom{L}{q_1} \binom{L}{q_2} W_{n, i, j}^{(q_1+q_2+2)}
$$
In this proof, we let $c_i$ denote the community of vertex $i$ and let $S_j$ denote all of the vertices in community $j$. In the graphon reprentation of this 2-community SBM, if $\omega_i$ is the latent feature for vertex $i$, then $\omega_i \in [0, 1/2)$ if and only if vertex $i$ belongs to community 1, and $\omega_i \in [1/2, 1]$ if and only if vertex $i$ belongs to community 2. To reflect this and simplify notation, we let $W_{n, S_i, S_j} := W_n(\omega_a, \omega_b)$, where $\omega_a$ and $\omega_b$ are any $\omega \in [0,1]$ so that correspond to the appropriate communities. For example, $W_{n, S_1, S_1} = W_{n}(1/4, 1/4) = p$, which is the probability that two vertices in community 1 are connected. We write the above as

$$\sup_{k \in S_i, \ell \in S_j} \left|   \langle \lambda_k^{L}, \lambda_\ell^{L} \rangle- \sum_{q_1=0}^{L} \sum_{q_2=0}^{L} \binom{L}{q_1} \binom{L}{q_2} W_{n, S_i, S_j}^{(q_1+q_2+2)} \right| \stackrel{p}{\to} 0 $$





In the remainder of the proof, we choose parameters $p, q, r \in [0,1]$ so that $\sum_{q_1=0}^{k_1} \sum_{q_2=0}^{k_2} \binom{k_1}{q_1} \binom{k_2}{q_2} W_{n, S_2,S_2}^{(q_1+q_2+2)} = \sum_{q_1=0}^{k_1} \sum_{q_2=0}^{k_2} \binom{k_1}{q_1} \binom{k_2}{q_2} W_{n, S_1,S_2}^{(q_1+q_2+2)}$, but $W_{n, S_2, S_2} \neq W_{n, S_1, S_2}$ (this last equality indicates that the connection probability between two vertices in community 2 is different than the connection probability between a vertex in community 1 and a vertex in community 2). This would suffice for the proof, since the continuous mapping theorem would imply that for any continuous function $f$, 

$$\sup_{k \in S_1, \ell \in S_2} \left|   f \left( \langle \lambda_k^{L}, \lambda_\ell^{L} \rangle \right) - f \left(\sum_{q_1=0}^{L} \sum_{q_2=0}^{L} \binom{L}{q_1} \binom{L}{q_2} W_{n, S_1, S_2}^{(q_1+q_2+2)} \right) \right| \stackrel{p}{\to} 0 $$

$$\sup_{k \in S_2, \ell \in S_2} \left|   f \left( \langle \lambda_k^{L}, \lambda_\ell^{L} \rangle \right) - f \left(\sum_{q_1=0}^{L} \sum_{q_2=0}^{L} \binom{L}{q_1} \binom{L}{q_2} W_{n, S_2, S_2}^{(q_1+q_2+2)} \right) \right| \stackrel{p}{\to} 0,$$
and we note that $W_{n, S_2, S_2} \neq W_{n, S_1, S_2}.$ With this in mind, consider
\begin{align*}
    \sum_{q_1=0}^{L} \sum_{q_2=0}^{L} \binom{L}{q_1} \binom{L}{q_2} W_{n, S_i,S_j}^{(q_1+q_2+2)} &= \sum_{q_1=0}^{L} \sum_{q_2=0}^{L} \binom{L}{q_1} \binom{L}{q_2} \left(\sum_{r=1}^2  \mu_r^{q_1+q_2+2} \phi_r(S_i) \phi_r(S_j) \right) \\
    &= \sum_{k=0}^{2L} \sum_{r=1}^2  \binom{2L}{k} \mu_r^{k+2} \phi_r(S_i) \phi_r(S_j)  \\
    &= \sum_{r=1}^2 \mu_r^2 (\mu_r+1)^{2L} \phi_r(S_i) \phi_r(S_j) 
\end{align*}
Substituting in the forms of the eigenvectors $\phi_1, \phi_2$, it suffices to show that there exist values of $p, q, r$, with $q \neq r$, so that 
\begin{gather*}
\mu_1^2 (\mu_1 + 1)^{2L} + \mu_2^2 (\mu_2+1)^{2L} = \frac{p-q+A}{2r} \mu_1^2 (\mu_1+1)^{2L} + \frac{p-q-A}{2r} \mu_2^2 (\mu_2+1)^{2L} \\
\Leftrightarrow \mu_2^2(\mu_2+1)^{2L} (2r - (p-q) + A) = \mu_1^2(\mu_1+1)^{2L} (-2r + (p-q) + A),
\end{gather*}
which would suffice for the proof. Recall that $A = \sqrt{(p-q)^2 + 4r^2},$ $\mu_1 = \frac{1}{4} (p+q+A),$ and $\mu_2 = \frac{1}{4}(p+q-A).$ Choosing $q=0$ and substituting these values in, we obtain that the above is equivalent to 
\begin{gather*}
\underbrace{\left( p-\sqrt{p^2+4r^2} \right)^2 \left( \frac{1}{4}(p-\sqrt{p^2+4r^2} )+1  \right)^{2L} (2r-p+\sqrt{p^2 + 4r^2})}_{(1)}\\ -  \underbrace{\left( p+\sqrt{p^2+4r^2} \right)^2 \left( \frac{1}{4}(p+\sqrt{p^2+4r^2} )+1  \right)^{2L} (p-2r+\sqrt{p^2 + 4r^2})}_{(2)} = 0
\end{gather*}
We show that there exist values $p$ and $r$, $r \neq 0$, so that there exists a root for some $p, r \in (0,1).$ We will fix $p = \epsilon \ll 1$ to be some small number to be decided later. $\epsilon$ might depend on $L$. Since this function is continuous in all variables, we use the intermediate value theorem (by varying $r$) to deduce that there exists a root for some sufficiently small $p.$ Firstly, we observe that $\lim_{r \downarrow 0} (1) = 0$. On the other hand, $\lim_{r \downarrow 0} (2) = (2p)^3(1+p/2)^{2L} > 0.$ This implies that $(1) - (2) < 0$ for sufficiently small $r$. Then, it suffices to argue that $(1)-(2)> 0$ for $r=1$, as then the intermediate value theorem implies the desired result.

Let $p = \epsilon.$ Taylor's theorem implies that $\sqrt{\epsilon^2 + 4} = 2 + \frac{\epsilon^2}{4} + O(\epsilon^4).$ Hence, we can write $(1)$ as 
$$(1) = \left( 2 - \epsilon + O(\epsilon^2) \right)^2 \left( \frac{1}{2} + \frac{1}{4} \epsilon + O(\epsilon^2)  \right)^{2L} (4-\epsilon + O(\epsilon^2)).$$ We can also write $$(2) = (2 + \epsilon + O(\epsilon^2)) \left( \frac{3}{2} + \frac{1}{4}\epsilon + O(\epsilon^2) \right)^{2L} (\epsilon + O(\epsilon^2)).$$
We note that the first two terms in both $(1)$ and $(2)$ are of constant order (the constants are difference, but both of constant order). However, the third term in $(1)$ is of constant order, but the third term in $(2)$ is of order $\epsilon.$ This implies that for sufficiently small $\epsilon$, $(1) - (2) < 0$. This suffices to imply that for this small enough $\epsilon$, there is some $r > 0$ so that $(1) - (2)$ has a root. This suffices for the proof.



\section{Experiments}
\label{sec:experiments}

We present three different sets of results. The first is for real-data (the Cora dataset), and we use the in-sample train/test splitting scheme for this. The second are in-sample experiments for a variety of random graph models, and the third set are out-sample experiments for these random graph models. For clarity, we explain and define the metrics we are using in the experiments. For the real-data experiments, we consider only the in-sample setting, while for the random graph experiments, we consider both the in-sample and out-of-sample settings.

One point of clarification is that our link prediction procedure is not simply to guess a particular edge to be a positive edge if its predicted probability is over 0.5. If the graphon $W (\cdot, \cdot) < 0.5$, then this doens't make sense because the LG-GNN estimates the underlying probability of edges. 

Instead, we concern ourselves more with the \textit{ranking} of the edges and choose evaluation metrics to reflect that. Concretely, we evaluate our algorithms on whether they are able to assign higher probabilities to edges with higher underlying probabilities (and in the real-data case, whether they are able to assign higher probabilities to the positive edges than to the negative edges in the testing set). 

To reflect this, the principle metrics we use are the AUC-ROC, Hits@k (as is standard for link prediction tasks in the Stanford Open Graph Benchmark) for the real data experiments. For the random graphs, introduce a new metric called the \textit{Probability Ratio@k}, defined below, which is inspired by the Hits@K metric.

\subsection{Train/Test Splits}
We first describe our train/test split procedures. 


\subsubsection{In-Sample (random graph)}
We generate a graph $G = (V,E)$ with $V = [n]$ ($n$ vertices). Let $N$ be the set of non-edges. Concretely, $N = \{ (i,j), i \neq j \in [n] | (i,j) \not \in E_n\}.$ We then split the edges into a train, validation, and testing set as follows.

For each edge $e \in E$, we remove it from the graph (independently from all the other edges) with probability $p=0.2$. The edges that are not removed are labelled $E_{train}.$ The set $E_{train}$ will be the set of positive training edges. Among the edges that were removed, half of those will be the set of positive validation edges and the remaining will be the set of positive test edges. Call these $E_{val}$ and $E_{test},$ respectively. During training, message passing only occurs along the edges in $E_{train}.$ 

Now, we select the negative edges among the set of edges $N \cup E_{test}.$ Specifically, $1-p$ fraction of these edges will be the negative training edges, $p/2$ fraction will be the negative validation edges, and the remaining $p/2$ fraction will be the negative testing edges. 

It is important to pick the negative training edges from the set $N \cup E_{test}$, as opposed to simply from $N$. If the negative training edges were sampled only from $N$, then this would give implicit information about where the edges are in the graph. The model should not have access to which edges are in $E_{test}$ vs in $N$ a priori; if the negative training edges that were given to it are only from $N$, then it would implicitly know that the edges in $E_{test}$ are less likely to be negative edges. Indeed, when we trained the GCN on a 2-community SBM with parameters 80 and 20 in the setting in which the negative training edges were sampled only from $N$, it was able to estimate the underlying parameters $80$ and $20$ almost perfectly, which should be impossible if it only had access to a graph with edges removed.

\subsubsection{In-Sample (real data)}

For real data, we use the same train/test split procedure as described above. However, during training and testing, we do not use the entire set of negative edges. This is because the graph is very sparse, and hence there are many more negative edges than positive edges. This makes link prediction difficult and causes the training procedure to be erratic.

\subsubsection{Out-Sample}

For each random graph model, we generate a graph $G = (V,E)$ with $V = [n]$ ($n$ vertices). We partition $V = V_1 \cup V_2,$ where $V_1$ contains a random $1-p=0.8$-fraction of the original set of vertices, and $V_2$ contains the remaining $p$ fraction. Let $G_1$ be the subgraph induced by $V_1$ (i.e., the set of all positive and negative edges with both endpoints in $V_1$). Let $E_1$ be the set of edges that have both endpoints in $V_1.$ Let $E_2$ be the set of edges that have at least one vertex in $V_2$, and let $N_2$ be the set of negative edges with at least one vertex in $V_2.$

We pick a random $1-p$ fraction of the positive and negative edges from $G_1$ to be the training positive and negative edges, and the remaining $p$ fraction to be the validation edges. Then, we pick $p$ fraction of the positive edges in $E_2$ and $p$ fraction of the negative edges in $N_2$ to be the testing edges. The remaining edges in $E_2$ and $N_2$ we will refer to as message-passing edges. 

We first train the models on the positive + validation edges. Then once the model is trained, we compute the embedding vectors by running message passing on the set of train  + message passing edges. Finally, we do edge prediction on the testing edges.

\subsection{Definition of Probability Ratio}

    Let $P(e)$ be the underlying probability of an edge $e=(i,j)$. In the graphon model, $P((i,j)) = W_n(\omega_i, \omega_j),$ and note that we have access to these values. Given a set of edges $E = \{ e_i = (v_{i,1}, v_{i,2}) \}_{i=1}^{|E|}$, we say that a link prediction algorithm ranks the edges as  $e_{i_1} > e_{i_2} > \dots > e_{i_{|E|}}$ if  $\hat{p}_{e_{i_1}} > \hat{p}_{e_{i_2}}>\dots>\hat{p}_{e_{i_{|E|}}},$ where $\hat{p}_{e}$ is the probability that the algorithm predicts for the edge $e$. Given some edge ranking as above, define the total predicted probability as $$P_{pred, k} := \sum_{r=1}^k P(e_{i_r})$$ and the maximum probability as $$P_{max, k} := \max_{e_1 \neq e_2 \neq \dots \neq e_k \in E} \sum_{r=1}^k P(e_r).$$ In other words, the $P_{max, k}$ is the sum of the probabilities of the top $k$ most likely edges in $E$. Then, the probability ratio is defined as $P_{pred, k}/P_{max, k}.$

    In essence, the Probability Ratio@k captures what fraction of the top $k$ probabilities a link prediction algorithm can capture. For example, suppose that there are three testing edges $e_1, e_2, e_3$ with underlying probabilities $0.8, 0.5, 0.2$, respectively. Suppose that some edge prediction algorithm ranks the edges as $e_1 > e_3 > e_2.$ Then the Probability Ratio@2 is equal to $\frac{0.8+0.2}{0.8+0.5} \approx 0.77.$



\subsection{Real-Data: Cora}

For the dataset, we perform a train/test split using the StellarGraph edge splitter, which randomly removes positive edges while ensuring that the resulting graph remains connected. For the negative training edges, we sample an equal number of negative edges as positive edges to train on. For the negative testing edges, we sample an equal number of negative edges as positive testing edges.

\subsubsection{Results without Node Features}


\begin{center}
    
\begin{table}[h]
\caption{GCN does not have access to node features}
\centering
\captionsetup{justification=centering,margin=2cm}
\label{tab:simulation_results}
\begin{tabular}{lllll}
\toprule
Parameter Set & Model & Cross Entropy & Hits@50 & Hits@100 \\
\midrule
\multirow{3}{*}{layers=2} 
 & GCN & \textbf{0.645} $\pm$ 0.043 & 0.496 $\pm$ 0.025 & 0.633 $\pm$ 0.023 \\
 & LG-GNN & 2.953 $\pm$ 0.013 & 0.565 $\pm$ 0.012 & 0.637 $\pm$ 0.006 \\
 & PLSG-GNN & 0.679 $\pm$ 0.012 & \textbf{0.591} $\pm$ 0.014 & \textbf{0.646} $\pm$ 0.013 \\
\cline{1-5}
\multirow{3}{*}{layers=4}
 & GCN & 0.689 $\pm$ 0.002 & 0.539 $\pm$ 0.008 & 0.665 $\pm$ 0.007 \\
 & LG-GNN & 2.682 $\pm$ 0.010 & 0.564 $\pm$ 0.005 & 0.620 $\pm$ 0.008 \\
 & PLSG-GNN & \textbf{0.660} $\pm$ 0.030 & \textbf{0.578} $\pm$ 0.014 & 0.637 $\pm$ 0.013 \\
\cline{1-5}
\bottomrule
\end{tabular}
\end{table}
\end{center}

\FloatBarrier

\subsubsection{Results with Node Features (GCN has access to node features)}

\begin{table}[h]
\caption{GCN has access to node features}
\centering
\captionsetup{justification=centering,margin=2cm}
\label{tab:simulation_results}
\begin{tabular}{lllll}
\toprule
Parameter Set & Model & Cross Entropy & Hits@50 & Hits@100 \\
\midrule
\multirow{3}{*}{layers=2} 
 & GCN & \textbf{0.487} $\pm$ 0.003 & \textbf{0.753} $\pm$ 0.019 & \textbf{0.898} $\pm$ 0.021 \\
 & LG-GNN & 3.034 $\pm$ 0.285 & 0.555 $\pm$ 0.027 & 0.603 $\pm$ 0.034 \\
 & PLSG-GNN & 0.679 $\pm$ 0.027 & 0.577 $\pm$ 0.033 & 0.626 $\pm$ 0.042 \\
\cline{1-5}
\multirow{3}{*}{layers=4} 
 & GCN & \textbf{0.661} $\pm$ 0.041 & \textbf{0.609} $\pm$ 0.072 & \textbf{0.776} $\pm$ 0.069 \\
 & LG-GNN & 2.711 $\pm$ 0.213 & 0.560 $\pm$ 0.013 & 0.601 $\pm$ 0.012 \\
 & PLSG-GNN & 0.677 $\pm$ 0.019 & 0.574 $\pm$ 0.025 & 0.625 $\pm$ 0.024 \\
\cline{1-5}
\bottomrule
\end{tabular}
\end{table}
\FloatBarrier

\cref{fig:images} shows histograms of the predicted probabilities by each of the algorithms (with the two cases of the GCN having access or not having access to the node features). This is to give a visual demonstrate as to what PLSG-GNN is doing. There is a "low probability" hump around 0.25, but then smaller peaks of high-probability predictions. The humps clearly separate the edges in regimes of how connected they are and show clearly the properties of the graph topology. 

\begin{figure}[h]
\centering
\begin{subfigure}[b]{0.49\columnwidth}
    \includegraphics[width=\linewidth]{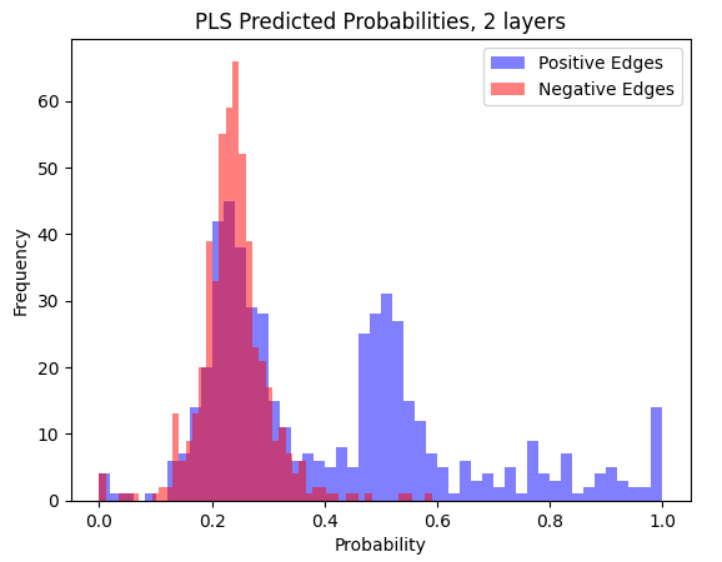}
    \label{fig:image1}
\end{subfigure}
\begin{subfigure}[b]{0.49\columnwidth}
    \includegraphics[width=\linewidth]{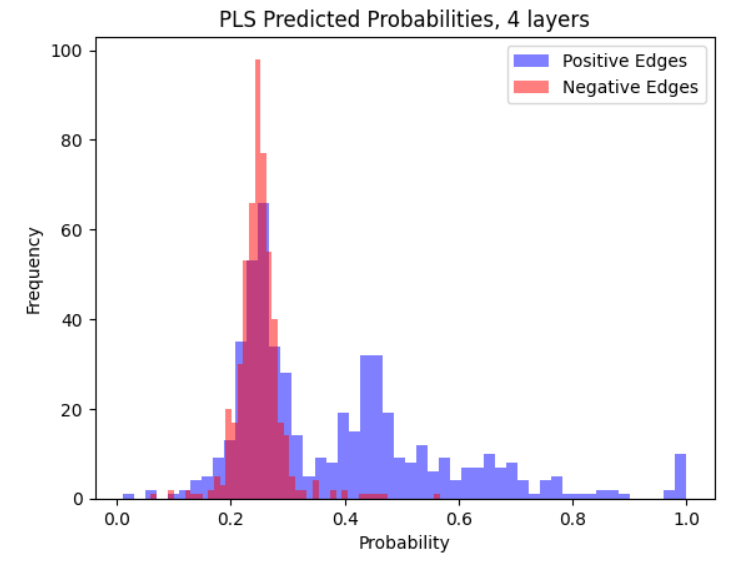}
    \label{fig:image2}
\end{subfigure}


\begin{subfigure}[b]{0.49\columnwidth}
    \includegraphics[width=\linewidth]{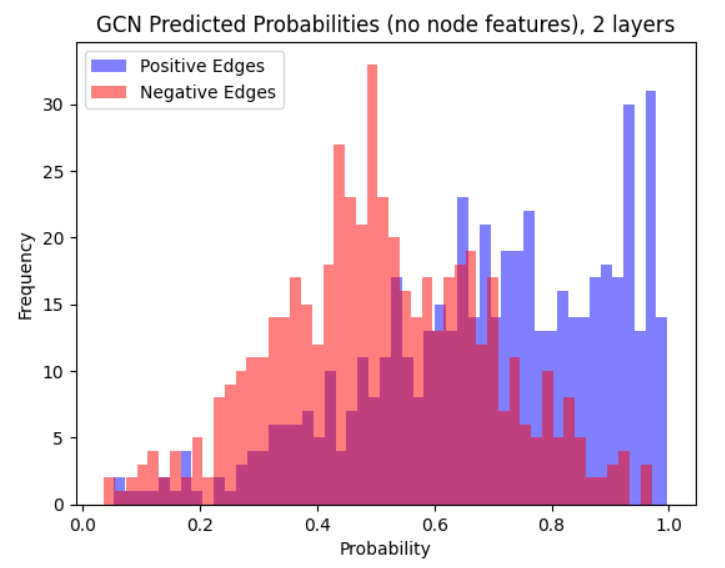}
    \label{fig:image3}
\end{subfigure}
\begin{subfigure}[b]{0.49\columnwidth}
    \includegraphics[width=\linewidth]{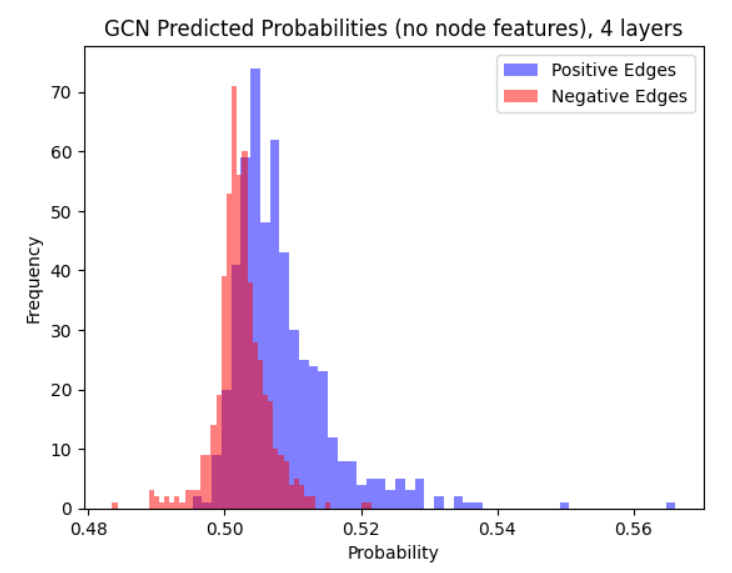}
    \label{fig:image4}
\end{subfigure}


\begin{subfigure}[b]{0.49\columnwidth}
    \includegraphics[width=\linewidth]{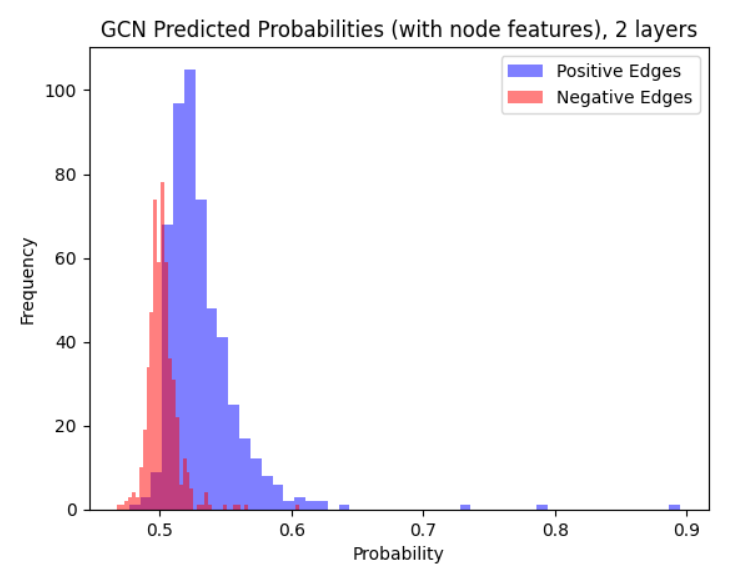}
    \label{fig:image5}
\end{subfigure}
\begin{subfigure}[b]{0.49\columnwidth}
    \includegraphics[width=\linewidth]{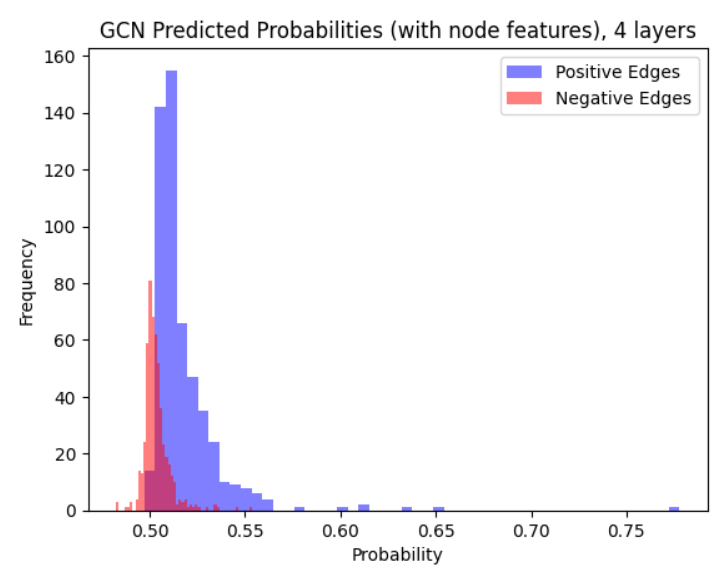}
    \label{fig:image6}
\end{subfigure}

\caption{Plot of the predicted probabilities by the PLS Regression (row 1), GCN without node features (row 2), and GCN with node features (row 3). The left column shows for 2 layers, the right column shows for 4 layers.}
\label{fig:images}
\end{figure}
\FloatBarrier

\section{Experiments (In-Sample)}

\subsection{6SSBM (80-20)}

6-community symmetric stochastic block model with connection probabilities 0.8 and 0.2.
\begin{center}

\begin{table}[h]
\caption{Symmetric Stochastic Block Model with Connection Probabilities 0.8, 0.2}
\centering
\captionsetup{justification=centering,margin=2cm}
\label{tab:simulation_results}
\begin{tabular}{llllll}
\toprule
 &  & Cross Entropy & Prob Ratio @ 100 & Prob Ratio @ 500 & AUC ROC \\
Parameters & Model &  &  &  &  \\
\midrule
\multirow[t]{3}{*}{rho=1, layers=2} & GCN & 0.587 $\pm$ 0.010 & \textbf{1.000} $\pm$ 0.000 & \textbf{1.000} $\pm$ 0.000 & \textbf{0.697} $\pm$ 0.002 \\
 & LG-GNN & \textbf{0.563} $\pm$ 0.001 & \textbf{1.000} $\pm$ 0.000 & \textbf{1.000} $\pm$ 0.000 & 0.677 $\pm$ 0.003 \\
 & PLSG-GNN & 0.583 $\pm$ 0.004 & \textbf{1.000} $\pm$ 0.000 & \textbf{1.000} $\pm$ 0.000 & 0.673 $\pm$ 0.004 \\
\cline{1-6}
\multirow[t]{3}{*}{rho=1, layers=4} & GCN & 0.693 $\pm$ 0.000 & 0.973 $\pm$ 0.009 & 0.978 $\pm$ 0.006 & 0.640 $\pm$ 0.029 \\
 & LG-GNN & \textbf{0.532} $\pm$ 0.000 & \textbf{1.000} $\pm$ 0.000 & \textbf{1.000} $\pm$ 0.000 & \textbf{0.697} $\pm$ 0.001 \\
 & PLSG-GNN & 0.579 $\pm$ 0.002 & \textbf{1.000} $\pm$ 0.000 & 0.999 $\pm$ 0.001 & 0.680 $\pm$ 0.002 \\
\cline{1-6}
\multirow[t]{3}{*}{rho=1/sqrt(n), layers=2} & GCN & 0.693 $\pm$ 0.000 & 0.388 $\pm$ 0.058 & 0.408 $\pm$ 0.006 & \textbf{0.503} $\pm$ 0.007 \\
 & LG-GNN & \textbf{0.090} $\pm$ 0.001 & \textbf{0.458} $\pm$ 0.035 & \textbf{0.450} $\pm$ 0.006 & 0.503 $\pm$ 0.007 \\
 & PLSG-GNN & 0.091 $\pm$ 0.002 & 0.398 $\pm$ 0.014 & 0.393 $\pm$ 0.010 & 0.491 $\pm$ 0.007 \\
\cline{1-6}
\multirow[t]{3}{*}{rho=1/sqrt(n), layers=4} & GCN & 0.693 $\pm$ 0.000 & 0.383 $\pm$ 0.009 & 0.376 $\pm$ 0.004 & 0.499 $\pm$ 0.007 \\
 & LG-GNN & \textbf{0.088} $\pm$ 0.000 & \textbf{0.430} $\pm$ 0.016 & \textbf{0.439} $\pm$ 0.013 & 0.505 $\pm$ 0.004 \\
 & PLSG-GNN & 0.089 $\pm$ 0.001 & 0.388 $\pm$ 0.009 & 0.388 $\pm$ 0.012 & \textbf{0.507} $\pm$ 0.001 \\
\cline{1-6}
\multirow[t]{3}{*}{rho=log(n)/n, layers=2} & GCN & 0.693 $\pm$ 0.000 & \textbf{0.382} $\pm$ 0.009 & 0.376 $\pm$ 0.012 & 0.458 $\pm$ 0.014 \\
 & LG-GNN & 0.021 $\pm$ 0.001 & 0.367 $\pm$ 0.004 & \textbf{0.389} $\pm$ 0.002 & 0.517 $\pm$ 0.006 \\
 & PLSG-GNN & \textbf{0.019} $\pm$ 0.000 & 0.370 $\pm$ 0.027 & 0.374 $\pm$ 0.012 & \textbf{0.521} $\pm$ 0.004 \\
\cline{1-6}
\multirow[t]{3}{*}{rho=log(n)/n, layers=4} & GCN & 0.693 $\pm$ 0.000 & 0.405 $\pm$ 0.007 & 0.380 $\pm$ 0.005 & 0.496 $\pm$ 0.010 \\
 & LG-GNN & 0.021 $\pm$ 0.000 & \textbf{0.470} $\pm$ 0.007 & \textbf{0.410} $\pm$ 0.003 & 0.505 $\pm$ 0.017 \\
 & PLSG-GNN & \textbf{0.019} $\pm$ 0.000 & 0.380 $\pm$ 0.004 & 0.375 $\pm$ 0.007 & \textbf{0.513} $\pm$ 0.010 \\
\cline{1-6}
\bottomrule
\end{tabular}
\end{table}
\end{center}
\FloatBarrier

\newpage
\subsection{6SSBM (55-45)}

6-community symmetric stochastic block model with edge connection probabilities 0.55 and 0.45.

\begin{center}
\begin{table}[h]
\caption{Symmetric Stochastic Block Model with Connection Probabilities 0.55, 0.45}
\centering
\captionsetup{justification=centering,margin=1cm}
\label{tab:simulation_results}
\begin{tabular}{llllll}
\toprule
 &  & Cross Entropy & Prob Ratio @ 100 & Prob Ratio @ 500 & AUC ROC \\
Parameters & Model &  &  &  &  \\
\midrule
\multirow[t]{3}{*}{rho=1, layers=2} & GCN & 0.693 $\pm$ 0.000 & \textbf{0.859} $\pm$ 0.004 & \textbf{0.852} $\pm$ 0.001 & 0.500 $\pm$ 0.001 \\
 & LG-GNN & 0.695 $\pm$ 0.000 & 0.849 $\pm$ 0.006 & 0.849 $\pm$ 0.003 & \textbf{0.500} $\pm$ 0.002 \\
 & PLSG-GNN & \textbf{0.693} $\pm$ 0.000 & 0.849 $\pm$ 0.006 & 0.849 $\pm$ 0.003 & 0.500 $\pm$ 0.002 \\
\cline{1-6}
\multirow[t]{3}{*}{rho=1, layers=4} & GCN & 0.693 $\pm$ 0.000 & 0.847 $\pm$ 0.000 & 0.846 $\pm$ 0.003 & 0.500 $\pm$ 0.000 \\
 & LG-GNN & 0.695 $\pm$ 0.000 & 0.848 $\pm$ 0.005 & 0.850 $\pm$ 0.002 & 0.500 $\pm$ 0.001 \\
 & PLSG-GNN & \textbf{0.693} $\pm$ 0.000 & \textbf{0.853} $\pm$ 0.005 & \textbf{0.851} $\pm$ 0.002 & \textbf{0.501} $\pm$ 0.001 \\
\cline{1-6}
\multirow[t]{3}{*}{rho=1/sqrt(n), layers=2} & GCN & 0.693 $\pm$ 0.000 & 0.847 $\pm$ 0.005 & 0.848 $\pm$ 0.001 & 0.502 $\pm$ 0.002 \\
 & LG-GNN & \textbf{0.130} $\pm$ 0.000 & \textbf{0.852} $\pm$ 0.001 & \textbf{0.849} $\pm$ 0.000 & 0.503 $\pm$ 0.006 \\
 & PLSG-GNN & 0.130 $\pm$ 0.001 & 0.850 $\pm$ 0.011 & 0.848 $\pm$ 0.004 & \textbf{0.506} $\pm$ 0.003 \\
\cline{1-6}
\multirow[t]{3}{*}{rho=1/sqrt(n), layers=4} & GCN & 0.693 $\pm$ 0.000 & 0.848 $\pm$ 0.002 & 0.847 $\pm$ 0.002 & 0.502 $\pm$ 0.006 \\
 & LG-GNN & \textbf{0.121} $\pm$ 0.001 & \textbf{0.849} $\pm$ 0.007 & 0.849 $\pm$ 0.001 & \textbf{0.502} $\pm$ 0.010 \\
 & PLSG-GNN & 0.131 $\pm$ 0.001 & 0.848 $\pm$ 0.005 & \textbf{0.852} $\pm$ 0.002 & 0.495 $\pm$ 0.004 \\
\cline{1-6}
\multirow[t]{3}{*}{rho=log(n)/n, layers=2} & GCN & 0.693 $\pm$ 0.000 & \textbf{0.853} $\pm$ 0.002 & 0.850 $\pm$ 0.002 & 0.489 $\pm$ 0.000 \\
 & LG-GNN & 0.031 $\pm$ 0.000 & 0.845 $\pm$ 0.007 & \textbf{0.850} $\pm$ 0.003 & \textbf{0.508} $\pm$ 0.008 \\
 & PLSG-GNN & \textbf{0.030} $\pm$ 0.001 & 0.852 $\pm$ 0.005 & 0.850 $\pm$ 0.002 & 0.496 $\pm$ 0.006 \\
\cline{1-6}
\multirow[t]{3}{*}{rho=log(n)/n, layers=4} & GCN & 0.693 $\pm$ 0.000 & 0.851 $\pm$ 0.004 & 0.848 $\pm$ 0.002 & 0.487 $\pm$ 0.004 \\
 & LG-GNN & 0.031 $\pm$ 0.001 & 0.847 $\pm$ 0.001 & 0.848 $\pm$ 0.003 & 0.492 $\pm$ 0.005 \\
 & PLSG-GNN & \textbf{0.030} $\pm$ 0.001 & \textbf{0.856} $\pm$ 0.007 & \textbf{0.852} $\pm$ 0.003 & \textbf{0.508} $\pm$ 0.017 \\
\cline{1-6}
\bottomrule
\end{tabular}
\end{table}
\end{center}
\FloatBarrier
\newpage

\subsection{10 SBM}

10-community stochastic block model with parameter matrix $P$ that has randomly generated entries. The diagonal entries $P_{i,i}$ are generated as $\text{Unif}(0.5, 1)$, and $P_{i,j}$ is generated as $\text{Unif}(0, \min(P_{i,i}, P_{j,j}))$. The connection matrix is 

\begin{center}
   $$ \begin{pmatrix}
0.9949 & 0.3084 & 0.4553 & 0.3747 & 0.6187 & 0.0052 & 0.2626 & 0.5787 & 0.4540 & 0.6768 \\
0.3084 & 0.8309 & 0.6851 & 0.0571 & 0.5225 & 0.3345 & 0.1279 & 0.0197 & 0.7063 & 0.7795 \\
0.4553 & 0.6851 & 0.7854 & 0.1000 & 0.7726 & 0.1882 & 0.1736 & 0.6723 & 0.3278 & 0.6033 \\
0.3747 & 0.0571 & 0.1000 & 0.6160 & 0.1168 & 0.0965 & 0.0021 & 0.1856 & 0.3248 & 0.4507 \\
0.6187 & 0.5225 & 0.7726 & 0.1168 & 0.8614 & 0.5492 & 0.1098 & 0.4278 & 0.6386 & 0.1171 \\
0.0052 & 0.3345 & 0.1882 & 0.0965 & 0.5492 & 0.6623 & 0.4277 & 0.0070 & 0.1145 & 0.2878 \\
0.2626 & 0.1279 & 0.1736 & 0.0021 & 0.1098 & 0.4277 & 0.5528 & 0.2016 & 0.5466 & 0.0410 \\
0.5787 & 0.0197 & 0.6723 & 0.1856 & 0.4278 & 0.0070 & 0.2016 & 0.8805 & 0.5233 & 0.0777 \\
0.4540 & 0.7063 & 0.3278 & 0.3248 & 0.6386 & 0.1145 & 0.5466 & 0.5233 & 0.9510 & 0.4890 \\
0.6768 & 0.7795 & 0.6033 & 0.4507 & 0.1171 & 0.2878 & 0.0410 & 0.0777 & 0.4890 & 0.8526 \\
\end{pmatrix}$$

\end{center}
\begin{center}
\begin{table}[h]
\caption{10-community SBM with randomly generated parameters}
\centering
\captionsetup{justification=centering,margin=1cm}
\label{tab:simulation_results}
\begin{tabular}{llllll}
\toprule
 &  & Cross Entropy & Prob Ratio @ 100 & Prob Ratio @ 500 & AUC ROC \\
Parameters & Model &  &  &  &  \\
\midrule
\multirow[t]{3}{*}{rho=1, layers=2} & GCN & 0.599 $\pm$ 0.001 & 0.878 $\pm$ 0.007 & \textbf{0.872} $\pm$ 0.007 & \textbf{0.764} $\pm$ 0.001 \\
 & LG-GNN & 0.588 $\pm$ 0.001 & 0.908 $\pm$ 0.008 & 0.867 $\pm$ 0.009 & 0.726 $\pm$ 0.002 \\
 & PLSG-GNN & \textbf{0.588} $\pm$ 0.001 & \textbf{0.909} $\pm$ 0.007 & 0.867 $\pm$ 0.006 & 0.727 $\pm$ 0.001 \\
\cline{1-6}
\multirow[t]{3}{*}{rho=1, layers=4} & GCN & 0.677 $\pm$ 0.003 & 0.737 $\pm$ 0.094 & 0.758 $\pm$ 0.106 & 0.672 $\pm$ 0.011 \\
 & LG-GNN & \textbf{0.562} $\pm$ 0.008 & 0.868 $\pm$ 0.042 & \textbf{0.868} $\pm$ 0.037 & \textbf{0.780} $\pm$ 0.002 \\
 & PLSG-GNN & 0.588 $\pm$ 0.001 & \textbf{0.896} $\pm$ 0.038 & 0.858 $\pm$ 0.023 & 0.728 $\pm$ 0.001 \\
\cline{1-6}
\multirow[t]{3}{*}{rho=1/sqrt(n), layers=2} & GCN & 0.693 $\pm$ 0.000 & 0.288 $\pm$ 0.022 & 0.315 $\pm$ 0.008 & 0.505 $\pm$ 0.003 \\
 & LG-GNN & 0.111 $\pm$ 0.002 & 0.561 $\pm$ 0.015 & 0.546 $\pm$ 0.023 & 0.515 $\pm$ 0.003 \\
 & PLSG-GNN & \textbf{0.110} $\pm$ 0.001 & \textbf{0.610} $\pm$ 0.008 & \textbf{0.577} $\pm$ 0.005 & \textbf{0.520} $\pm$ 0.005 \\
\cline{1-6}
\multirow[t]{3}{*}{rho=1/sqrt(n), layers=4} & GCN & 0.693 $\pm$ 0.000 & 0.298 $\pm$ 0.048 & 0.312 $\pm$ 0.043 & 0.512 $\pm$ 0.014 \\
 & LG-GNN & \textbf{0.105} $\pm$ 0.003 & 0.584 $\pm$ 0.036 & 0.564 $\pm$ 0.011 & 0.516 $\pm$ 0.008 \\
 & PLSG-GNN & 0.110 $\pm$ 0.002 & \textbf{0.589} $\pm$ 0.022 & \textbf{0.564} $\pm$ 0.003 & \textbf{0.517} $\pm$ 0.011 \\
\cline{1-6}
\multirow[t]{3}{*}{rho=log(n)/n, layers=2} & GCN & 0.693 $\pm$ 0.000 & 0.300 $\pm$ 0.017 & 0.307 $\pm$ 0.014 & 0.478 $\pm$ 0.019 \\
 & LG-GNN & 0.026 $\pm$ 0.000 & 0.486 $\pm$ 0.010 & \textbf{0.494} $\pm$ 0.005 & \textbf{0.525} $\pm$ 0.009 \\
 & PLSG-GNN & \textbf{0.024} $\pm$ 0.000 & \textbf{0.493} $\pm$ 0.012 & 0.490 $\pm$ 0.004 & 0.514 $\pm$ 0.018 \\
\cline{1-6}
\multirow[t]{3}{*}{rho=log(n)/n, layers=4} & GCN & 0.693 $\pm$ 0.000 & 0.312 $\pm$ 0.013 & 0.303 $\pm$ 0.010 & \textbf{0.517} $\pm$ 0.019 \\
 & LG-GNN & 0.026 $\pm$ 0.002 & \textbf{0.498} $\pm$ 0.004 & 0.494 $\pm$ 0.007 & 0.517 $\pm$ 0.014 \\
 & PLSG-GNN & \textbf{0.025} $\pm$ 0.001 & 0.496 $\pm$ 0.008 & \textbf{0.501} $\pm$ 0.006 & 0.514 $\pm$ 0.013 \\
\cline{1-6}
\bottomrule
\end{tabular}
\end{table}
\end{center}
\FloatBarrier

\newpage

\subsection{Geometric Graph}

Each vertex $i$ has a latent feature $X_i$ generated uniformly at random on $\mathbb{S}^{d-1}$, $d=11.$ Two vertices $i$ and $j$ are connected if $\langle X_i, X_j \rangle \ge t = 0.2,$ corresponding to a connection probability $\approx 0.26.$ Higher sparsity is achieved by adjusting $t$. 

\begin{center}
\begin{table}[h]
\caption{Geometric Graph with threshold 0.2 (corresponding to a connection probability of $\approx 0.26$)}
\centering
\captionsetup{justification=centering,margin=1cm}
\label{tab:simulation_results}
\begin{tabular}{llllll}
\toprule
 &  & Cross Entropy & Prob Ratio @ 100 & Prob Ratio @ 500 & AUC ROC \\
Parameters & Model &  &  &  &  \\
\midrule
\multirow[t]{3}{*}{rho=1, layers=2} & GCN & 0.537 $\pm$ 0.012 & \textbf{1.000} $\pm$ 0.000 & \textbf{1.000} $\pm$ 0.000 & 0.886 $\pm$ 0.015 \\
 & LG-GNN & 0.354 $\pm$ 0.005 & \textbf{1.000} $\pm$ 0.000 & \textbf{1.000} $\pm$ 0.000 & 0.916 $\pm$ 0.004 \\
 & PLSG-GNN & \textbf{0.343} $\pm$ 0.006 & \textbf{1.000} $\pm$ 0.000 & 0.996 $\pm$ 0.003 & \textbf{0.918} $\pm$ 0.005 \\
\cline{1-6}
\multirow[t]{3}{*}{rho=1, layers=4} & GCN & 0.693 $\pm$ 0.000 & 0.900 $\pm$ 0.127 & 0.767 $\pm$ 0.075 & 0.759 $\pm$ 0.039 \\
 & LG-GNN & 0.305 $\pm$ 0.002 & \textbf{1.000} $\pm$ 0.000 & \textbf{1.000} $\pm$ 0.000 & 0.950 $\pm$ 0.002 \\
 & PLSG-GNN & \textbf{0.301} $\pm$ 0.002 & \textbf{1.000} $\pm$ 0.000 & 0.999 $\pm$ 0.001 & \textbf{0.956} $\pm$ 0.002 \\
\cline{1-6}
\multirow[t]{3}{*}{rho=1/sqrt(n), layers=2} & GCN & 0.693 $\pm$ 0.000 & 0.333 $\pm$ 0.062 & 0.232 $\pm$ 0.030 & \textbf{0.848} $\pm$ 0.007 \\
 & LG-GNN & 0.046 $\pm$ 0.002 & \textbf{0.637} $\pm$ 0.059 & \textbf{0.379} $\pm$ 0.020 & 0.822 $\pm$ 0.012 \\
 & PLSG-GNN & \textbf{0.046} $\pm$ 0.002 & 0.453 $\pm$ 0.076 & 0.293 $\pm$ 0.035 & 0.844 $\pm$ 0.012 \\
\cline{1-6}
\multirow[t]{3}{*}{rho=1/sqrt(n), layers=4} & GCN & 0.693 $\pm$ 0.000 & 0.410 $\pm$ 0.016 & 0.275 $\pm$ 0.021 & \textbf{0.883} $\pm$ 0.003 \\
 & LG-GNN & \textbf{0.045} $\pm$ 0.001 & \textbf{0.637} $\pm$ 0.054 & \textbf{0.377} $\pm$ 0.024 & 0.827 $\pm$ 0.003 \\
 & PLSG-GNN & 0.045 $\pm$ 0.001 & 0.530 $\pm$ 0.079 & 0.345 $\pm$ 0.012 & 0.848 $\pm$ 0.003 \\
\cline{1-6}
\multirow[t]{3}{*}{rho=log(n)/n, layers=2} & GCN & 0.693 $\pm$ 0.000 & 0.003 $\pm$ 0.005 & 0.019 $\pm$ 0.005 & \textbf{0.624} $\pm$ 0.018 \\
 & LG-GNN & 0.019 $\pm$ 0.001 & \textbf{0.163} $\pm$ 0.021 & \textbf{0.247} $\pm$ 0.005 & 0.607 $\pm$ 0.019 \\
 & PLSG-GNN & \textbf{0.019} $\pm$ 0.000 & 0.100 $\pm$ 0.024 & 0.097 $\pm$ 0.014 & 0.611 $\pm$ 0.009 \\
\cline{1-6}
\multirow[t]{3}{*}{rho=log(n)/n, layers=4} & GCN & 0.693 $\pm$ 0.000 & 0.003 $\pm$ 0.005 & 0.011 $\pm$ 0.009 & 0.608 $\pm$ 0.018 \\
 & LG-GNN & 0.019 $\pm$ 0.001 & \textbf{0.170} $\pm$ 0.022 & \textbf{0.237} $\pm$ 0.029 & 0.609 $\pm$ 0.032 \\
 & PLSG-GNN & \textbf{0.018} $\pm$ 0.001 & 0.133 $\pm$ 0.019 & 0.154 $\pm$ 0.023 & \textbf{0.634} $\pm$ 0.022 \\
\cline{1-6}
\bottomrule
\end{tabular}
\end{table}
\end{center}
\FloatBarrier

\newpage

\newpage

\section{Out-Sample Experiments}

\subsection{6SSBM (80-20)}

6-community symmetric stochastic block model with connection probabilities 0.8 and 0.2.
\begin{center}
\begin{table}[h]
\caption{ Symmetric Stochastic Block Model with Connection Probabilities 0.8, 0.2}
\centering
\captionsetup{justification=centering,margin=2cm}
\label{tab:simulation_results}
\begin{tabular}{llllll}
\toprule
 &  & Cross Entropy & Prob Ratio @ 100 & Prob Ratio @ 500 & AUC ROC \\
Parameters & Model &  &  &  &  \\
\midrule
\multirow[t]{3}{*}{rho=1, layers=2} & GCN & 0.610 $\pm$ 0.022 & \textbf{1.000} $\pm$ 0.000 & \textbf{1.000} $\pm$ 0.000 & \textbf{0.699} $\pm$ 0.001 \\
 & LG-GNN & \textbf{0.569} $\pm$ 0.004 & 0.998 $\pm$ 0.004 & 0.998 $\pm$ 0.001 & 0.682 $\pm$ 0.002 \\
 & PLSG-GNN & 0.730 $\pm$ 0.184 & 0.998 $\pm$ 0.004 & 0.999 $\pm$ 0.001 & 0.677 $\pm$ 0.002 \\
\cline{1-6}
\multirow[t]{3}{*}{rho=1, layers=4} & GCN & 0.693 $\pm$ 0.000 & 0.623 $\pm$ 0.106 & 0.581 $\pm$ 0.090 & 0.520 $\pm$ 0.004 \\
 & LG-GNN & \textbf{0.545} $\pm$ 0.004 & \textbf{1.000} $\pm$ 0.000 & \textbf{1.000} $\pm$ 0.000 & \textbf{0.698} $\pm$ 0.001 \\
 & PLSG-GNN & 0.585 $\pm$ 0.014 & \textbf{1.000} $\pm$ 0.000 & 0.997 $\pm$ 0.003 & 0.680 $\pm$ 0.001 \\
\cline{1-6}
\multirow[t]{3}{*}{rho=1/sqrt(n), layers=2} & GCN & 0.693 $\pm$ 0.000 & 0.448 $\pm$ 0.028 & 0.423 $\pm$ 0.005 & 0.502 $\pm$ 0.008 \\
 & LG-GNN & 0.092 $\pm$ 0.002 & \textbf{0.450} $\pm$ 0.031 & \textbf{0.436} $\pm$ 0.013 & \textbf{0.508} $\pm$ 0.004 \\
 & PLSG-GNN & \textbf{0.091} $\pm$ 0.002 & 0.405 $\pm$ 0.004 & 0.387 $\pm$ 0.003 & 0.503 $\pm$ 0.003 \\
\cline{1-6}
\multirow[t]{3}{*}{rho=1/sqrt(n), layers=4} & GCN & 0.693 $\pm$ 0.000 & 0.415 $\pm$ 0.006 & 0.387 $\pm$ 0.004 & \textbf{0.506} $\pm$ 0.006 \\
 & LG-GNN & \textbf{0.088} $\pm$ 0.002 & \textbf{0.460} $\pm$ 0.011 & \textbf{0.436} $\pm$ 0.008 & 0.501 $\pm$ 0.014 \\
 & PLSG-GNN & 0.089 $\pm$ 0.001 & 0.390 $\pm$ 0.019 & 0.382 $\pm$ 0.011 & 0.496 $\pm$ 0.007 \\
\cline{1-6}
\multirow[t]{3}{*}{rho=log(n)/n, layers=2} & GCN & 0.693 $\pm$ 0.000 & 0.382 $\pm$ 0.018 & 0.371 $\pm$ 0.008 & 0.488 $\pm$ 0.010 \\
 & LG-GNN & 0.021 $\pm$ 0.001 & 0.355 $\pm$ 0.024 & 0.371 $\pm$ 0.012 & \textbf{0.510} $\pm$ 0.019 \\
 & PLSG-GNN & \textbf{0.019} $\pm$ 0.001 & \textbf{0.387} $\pm$ 0.049 & \textbf{0.377} $\pm$ 0.015 & 0.491 $\pm$ 0.030 \\
\cline{1-6}
\multirow[t]{3}{*}{rho=log(n)/n, layers=4} & GCN & 0.694 $\pm$ 0.000 & 0.377 $\pm$ 0.022 & 0.379 $\pm$ 0.014 & 0.506 $\pm$ 0.015 \\
 & LG-GNN & 0.021 $\pm$ 0.000 & \textbf{0.460} $\pm$ 0.054 & \textbf{0.384} $\pm$ 0.011 & 0.499 $\pm$ 0.015 \\
 & PLSG-GNN & \textbf{0.018} $\pm$ 0.000 & 0.395 $\pm$ 0.013 & 0.376 $\pm$ 0.021 & \textbf{0.511} $\pm$ 0.011 \\
\cline{1-6}
\bottomrule
\end{tabular}
\end{table}
\end{center}
\FloatBarrier
\newpage

\subsection{6SSBM (55-45)}

6-community symmetric stochastic block model with edge connection probabilities 0.55 and 0.45.

\begin{center}
\begin{table}[h]
\caption{Symmetric Stochastic Block Model with Connection Probabilities 0.55, 0.45}
\centering
\captionsetup{justification=centering,margin=1cm}
\label{tab:simulation_results}
\begin{tabular}{llllll}
\toprule
 &  & Cross Entropy & Prob Ratio @ 100 & Prob Ratio @ 500 & AUC ROC \\
Parameters & Model &  &  &  &  \\
\midrule
\multirow[t]{3}{*}{rho=1, layers=2} & GCN & \textbf{0.693} $\pm$ 0.000 & \textbf{0.848} $\pm$ 0.006 & \textbf{0.850} $\pm$ 0.002 & \textbf{0.500} $\pm$ 0.002 \\
 & LG-GNN & 0.698 $\pm$ 0.002 & 0.845 $\pm$ 0.005 & 0.846 $\pm$ 0.003 & 0.500 $\pm$ 0.001 \\
 & PLSG-GNN & 0.694 $\pm$ 0.001 & 0.844 $\pm$ 0.003 & 0.846 $\pm$ 0.003 & 0.500 $\pm$ 0.001 \\
\cline{1-6}
\multirow[t]{3}{*}{rho=1, layers=4} & GCN & \textbf{0.693} $\pm$ 0.000 & \textbf{0.848} $\pm$ 0.004 & \textbf{0.850} $\pm$ 0.001 & 0.498 $\pm$ 0.001 \\
 & LG-GNN & 0.702 $\pm$ 0.001 & 0.847 $\pm$ 0.012 & 0.848 $\pm$ 0.004 & 0.499 $\pm$ 0.001 \\
 & PLSG-GNN & 0.695 $\pm$ 0.000 & 0.844 $\pm$ 0.011 & 0.850 $\pm$ 0.004 & \textbf{0.499} $\pm$ 0.001 \\
\cline{1-6}
\multirow[t]{3}{*}{rho=1/sqrt(n), layers=2} & GCN & 0.693 $\pm$ 0.000 & 0.851 $\pm$ 0.004 & 0.850 $\pm$ 0.002 & 0.496 $\pm$ 0.004 \\
 & LG-GNN & 0.131 $\pm$ 0.003 & \textbf{0.859} $\pm$ 0.010 & \textbf{0.851} $\pm$ 0.005 & 0.505 $\pm$ 0.011 \\
 & PLSG-GNN & \textbf{0.131} $\pm$ 0.002 & 0.844 $\pm$ 0.007 & 0.850 $\pm$ 0.001 & \textbf{0.505} $\pm$ 0.003 \\
\cline{1-6}
\multirow[t]{3}{*}{rho=1/sqrt(n), layers=4} & GCN & 0.693 $\pm$ 0.000 & 0.842 $\pm$ 0.008 & 0.847 $\pm$ 0.001 & \textbf{0.500} $\pm$ 0.009 \\
 & LG-GNN & \textbf{0.123} $\pm$ 0.001 & 0.850 $\pm$ 0.003 & \textbf{0.849} $\pm$ 0.001 & 0.497 $\pm$ 0.016 \\
 & PLSG-GNN & 0.130 $\pm$ 0.002 & \textbf{0.852} $\pm$ 0.008 & 0.848 $\pm$ 0.001 & 0.498 $\pm$ 0.012 \\
\cline{1-6}
\multirow[t]{3}{*}{rho=log(n)/n, layers=2} & GCN & 0.693 $\pm$ 0.000 & 0.844 $\pm$ 0.007 & \textbf{0.850} $\pm$ 0.001 & 0.488 $\pm$ 0.031 \\
 & LG-GNN & 0.030 $\pm$ 0.000 & 0.842 $\pm$ 0.002 & 0.846 $\pm$ 0.005 & 0.484 $\pm$ 0.011 \\
 & PLSG-GNN & \textbf{0.029} $\pm$ 0.001 & \textbf{0.851} $\pm$ 0.001 & 0.849 $\pm$ 0.002 & \textbf{0.505} $\pm$ 0.008 \\
\cline{1-6}
\multirow[t]{3}{*}{rho=log(n)/n, layers=4} & GCN & 0.693 $\pm$ 0.000 & \textbf{0.851} $\pm$ 0.004 & \textbf{0.847} $\pm$ 0.002 & 0.493 $\pm$ 0.024 \\
 & LG-GNN & 0.030 $\pm$ 0.002 & 0.844 $\pm$ 0.005 & 0.845 $\pm$ 0.003 & 0.488 $\pm$ 0.015 \\
 & PLSG-GNN & \textbf{0.028} $\pm$ 0.000 & 0.845 $\pm$ 0.009 & 0.846 $\pm$ 0.004 & \textbf{0.504} $\pm$ 0.002 \\
\cline{1-6}
\bottomrule
\end{tabular}
\end{table}
\end{center}
\FloatBarrier

\newpage

\subsection{10 SBM}
10-community stochastic block model with parameter matrix $P$ that has randomly generated entries. The diagonal entries $P_{i,i}$ are generated as $\text{Unif}(0.5, 1)$, and $P_{i,j}$ is generated as $\text{Unif}(0, \min(P_{i,i}, P_{j,j}))$. The connection matrix is 

\begin{center}
    $$\begin{pmatrix}
0.9949 & 0.3084 & 0.4553 & 0.3747 & 0.6187 & 0.0052 & 0.2626 & 0.5787 & 0.4540 & 0.6768 \\
0.3084 & 0.8309 & 0.6851 & 0.0571 & 0.5225 & 0.3345 & 0.1279 & 0.0197 & 0.7063 & 0.7795 \\
0.4553 & 0.6851 & 0.7854 & 0.1000 & 0.7726 & 0.1882 & 0.1736 & 0.6723 & 0.3278 & 0.6033 \\
0.3747 & 0.0571 & 0.1000 & 0.6160 & 0.1168 & 0.0965 & 0.0021 & 0.1856 & 0.3248 & 0.4507 \\
0.6187 & 0.5225 & 0.7726 & 0.1168 & 0.8614 & 0.5492 & 0.1098 & 0.4278 & 0.6386 & 0.1171 \\
0.0052 & 0.3345 & 0.1882 & 0.0965 & 0.5492 & 0.6623 & 0.4277 & 0.0070 & 0.1145 & 0.2878 \\
0.2626 & 0.1279 & 0.1736 & 0.0021 & 0.1098 & 0.4277 & 0.5528 & 0.2016 & 0.5466 & 0.0410 \\
0.5787 & 0.0197 & 0.6723 & 0.1856 & 0.4278 & 0.0070 & 0.2016 & 0.8805 & 0.5233 & 0.0777 \\
0.4540 & 0.7063 & 0.3278 & 0.3248 & 0.6386 & 0.1145 & 0.5466 & 0.5233 & 0.9510 & 0.4890 \\
0.6768 & 0.7795 & 0.6033 & 0.4507 & 0.1171 & 0.2878 & 0.0410 & 0.0777 & 0.4890 & 0.8526 \\
\end{pmatrix}$$
\end{center}

\begin{center}
    
\begin{table}[h]
\caption{10-community SBM with randomly generated parameters}
\centering
\captionsetup{justification=centering,margin=1cm}
\label{tab:simulation_results}
\begin{tabular}{llllll}
\toprule
 &  & Cross Entropy & Prob Ratio @ 100 & Prob Ratio @ 500 & AUC ROC \\
Parameters & Model &  &  &  &  \\
\midrule
\multirow[t]{3}{*}{rho=1, layers=2} & GCN & 0.635 $\pm$ 0.014 & 0.709 $\pm$ 0.125 & 0.726 $\pm$ 0.108 & 0.716 $\pm$ 0.019 \\
 & LG-GNN & 0.586 $\pm$ 0.004 & 0.883 $\pm$ 0.016 & 0.843 $\pm$ 0.014 & 0.734 $\pm$ 0.005 \\
 & PLSG-GNN & \textbf{0.586} $\pm$ 0.004 & \textbf{0.886} $\pm$ 0.016 & \textbf{0.844} $\pm$ 0.013 & \textbf{0.735} $\pm$ 0.005 \\
\cline{1-6}
\multirow[t]{3}{*}{rho=1, layers=4} & GCN & 0.801 $\pm$ 0.193 & 0.645 $\pm$ 0.025 & 0.633 $\pm$ 0.027 & 0.578 $\pm$ 0.109 \\
 & LG-GNN & \textbf{0.564} $\pm$ 0.011 & 0.879 $\pm$ 0.011 & \textbf{0.886} $\pm$ 0.004 & \textbf{0.786} $\pm$ 0.002 \\
 & PLSG-GNN & 0.592 $\pm$ 0.004 & \textbf{0.883} $\pm$ 0.013 & 0.836 $\pm$ 0.015 & 0.732 $\pm$ 0.001 \\
\cline{1-6}
\multirow[t]{3}{*}{rho=1/sqrt(n), layers=2} & GCN & 0.693 $\pm$ 0.000 & 0.344 $\pm$ 0.021 & 0.318 $\pm$ 0.013 & 0.493 $\pm$ 0.004 \\
 & LG-GNN & 0.115 $\pm$ 0.002 & 0.580 $\pm$ 0.020 & 0.557 $\pm$ 0.007 & 0.497 $\pm$ 0.009 \\
 & PLSG-GNN & \textbf{0.112} $\pm$ 0.004 & \textbf{0.586} $\pm$ 0.035 & \textbf{0.561} $\pm$ 0.001 & \textbf{0.521} $\pm$ 0.008 \\
\cline{1-6}
\multirow[t]{3}{*}{rho=1/sqrt(n), layers=4} & GCN & 0.693 $\pm$ 0.000 & 0.285 $\pm$ 0.016 & 0.275 $\pm$ 0.006 & 0.486 $\pm$ 0.006 \\
 & LG-GNN & \textbf{0.105} $\pm$ 0.000 & \textbf{0.589} $\pm$ 0.016 & \textbf{0.563} $\pm$ 0.003 & \textbf{0.532} $\pm$ 0.003 \\
 & PLSG-GNN & 0.111 $\pm$ 0.002 & 0.578 $\pm$ 0.013 & 0.544 $\pm$ 0.009 & 0.508 $\pm$ 0.011 \\
\cline{1-6}
\multirow[t]{3}{*}{rho=log(n)/n, layers=2} & GCN & 0.693 $\pm$ 0.000 & 0.312 $\pm$ 0.011 & 0.316 $\pm$ 0.006 & 0.503 $\pm$ 0.017 \\
 & LG-GNN & 0.026 $\pm$ 0.000 & \textbf{0.528} $\pm$ 0.029 & \textbf{0.504} $\pm$ 0.006 & 0.506 $\pm$ 0.015 \\
 & PLSG-GNN & \textbf{0.023} $\pm$ 0.002 & 0.511 $\pm$ 0.017 & 0.501 $\pm$ 0.013 & \textbf{0.519} $\pm$ 0.002 \\
\cline{1-6}
\multirow[t]{3}{*}{rho=log(n)/n, layers=4} & GCN & 0.693 $\pm$ 0.000 & 0.304 $\pm$ 0.027 & 0.304 $\pm$ 0.015 & \textbf{0.518} $\pm$ 0.005 \\
 & LG-GNN & 0.026 $\pm$ 0.000 & 0.498 $\pm$ 0.017 & 0.486 $\pm$ 0.015 & 0.500 $\pm$ 0.013 \\
 & PLSG-GNN & \textbf{0.024} $\pm$ 0.000 & \textbf{0.546} $\pm$ 0.018 & \textbf{0.505} $\pm$ 0.018 & 0.498 $\pm$ 0.016 \\
\cline{1-6}
\bottomrule
\end{tabular}
\end{table}

\end{center}

\FloatBarrier

\newpage

\subsection{Geometric Graph}

We generate points uniformly on $\mathbb{S}^d$ and connect two points if $\langle X_i, X_j \rangle \ge t.$ For the following experiment, we chose $d=11$ and $t = 0.3.$ This corresponds to a probability of about 0.15.

\begin{table}[h]
\caption{Geometric Graph with threshold 0.2 (corresponding to a connection probability of $\approx 0.26$)}
\centering
\captionsetup{justification=centering,margin=1cm}
\label{tab:simulation_results}
\begin{tabular}{llllll}
\toprule
 &  & Cross Entropy & Prob Ratio @ 100 & Prob Ratio @ 500 & AUC ROC \\
Parameters & Model &  &  &  &  \\
\midrule
\multirow[t]{3}{*}{rho=1, layers=2} & GCN & 0.573 $\pm$ 0.015 & \textbf{1.000} $\pm$ 0.000 & 0.996 $\pm$ 0.002 & 0.873 $\pm$ 0.020 \\
 & LG-GNN & 0.358 $\pm$ 0.009 & \textbf{1.000} $\pm$ 0.000 & \textbf{0.999} $\pm$ 0.001 & 0.915 $\pm$ 0.007 \\
 & PLSG-GNN & \textbf{0.350} $\pm$ 0.013 & 0.997 $\pm$ 0.005 & 0.999 $\pm$ 0.002 & \textbf{0.917} $\pm$ 0.010 \\
\cline{1-6}
\multirow[t]{3}{*}{rho=1, layers=4} & GCN & 0.693 $\pm$ 0.000 & 0.813 $\pm$ 0.021 & 0.733 $\pm$ 0.079 & 0.591 $\pm$ 0.016 \\
 & LG-GNN & 0.303 $\pm$ 0.004 & \textbf{1.000} $\pm$ 0.000 & \textbf{1.000} $\pm$ 0.000 & 0.956 $\pm$ 0.001 \\
 & PLSG-GNN & \textbf{0.298} $\pm$ 0.004 & \textbf{1.000} $\pm$ 0.000 & \textbf{1.000} $\pm$ 0.000 & \textbf{0.958} $\pm$ 0.001 \\
\cline{1-6}
\multirow[t]{3}{*}{rho=1/sqrt(n), layers=2} & GCN & 0.693 $\pm$ 0.000 & 0.333 $\pm$ 0.017 & 0.216 $\pm$ 0.017 & 0.840 $\pm$ 0.008 \\
 & LG-GNN & 0.046 $\pm$ 0.003 & \textbf{0.523} $\pm$ 0.037 & \textbf{0.311} $\pm$ 0.021 & 0.818 $\pm$ 0.022 \\
 & PLSG-GNN & \textbf{0.045} $\pm$ 0.002 & 0.423 $\pm$ 0.054 & 0.244 $\pm$ 0.020 & \textbf{0.842} $\pm$ 0.017 \\
\cline{1-6}
\multirow[t]{3}{*}{rho=1/sqrt(n), layers=4} & GCN & 0.693 $\pm$ 0.000 & 0.313 $\pm$ 0.021 & 0.207 $\pm$ 0.013 & \textbf{0.848} $\pm$ 0.021 \\
 & LG-GNN & \textbf{0.045} $\pm$ 0.001 & \textbf{0.570} $\pm$ 0.016 & \textbf{0.311} $\pm$ 0.010 & 0.823 $\pm$ 0.010 \\
 & PLSG-GNN & 0.045 $\pm$ 0.001 & 0.510 $\pm$ 0.014 & 0.289 $\pm$ 0.003 & 0.843 $\pm$ 0.013 \\
\cline{1-6}
\multirow[t]{3}{*}{rho=log(n)/n, layers=2} & GCN & 0.693 $\pm$ 0.000 & 0.003 $\pm$ 0.005 & 0.012 $\pm$ 0.004 & 0.610 $\pm$ 0.026 \\
 & LG-GNN & 0.018 $\pm$ 0.000 & \textbf{0.210} $\pm$ 0.029 & \textbf{0.276} $\pm$ 0.005 & 0.616 $\pm$ 0.018 \\
 & PLSG-GNN & \textbf{0.018} $\pm$ 0.001 & 0.063 $\pm$ 0.037 & 0.123 $\pm$ 0.033 & \textbf{0.631} $\pm$ 0.027 \\
\cline{1-6}
\multirow[t]{3}{*}{rho=log(n)/n, layers=4} & GCN & 0.693 $\pm$ 0.000 & 0.007 $\pm$ 0.009 & 0.035 $\pm$ 0.007 & \textbf{0.607} $\pm$ 0.002 \\
 & LG-GNN & 0.019 $\pm$ 0.000 & \textbf{0.147} $\pm$ 0.012 & \textbf{0.191} $\pm$ 0.017 & 0.569 $\pm$ 0.015 \\
 & PLSG-GNN & \textbf{0.018} $\pm$ 0.000 & 0.107 $\pm$ 0.012 & 0.143 $\pm$ 0.018 & 0.607 $\pm$ 0.010 \\
\cline{1-6}
\bottomrule
\end{tabular}
\end{table}
\FloatBarrier

\newpage

\end{document}